\documentclass{article}

\usepackage{microtype}
\usepackage{graphicx}
\usepackage{booktabs} 
\usepackage{amsmath,amsthm,amssymb,bm} 
\usepackage{xspace}
\usepackage{nicefrac}
\usepackage{amsfonts}
\usepackage{array}
\usepackage{enumerate}
\usepackage{stmaryrd}
\usepackage{caption}
\usepackage{subcaption}
\usepackage{multirow}
\usepackage{times}
\usepackage{algorithm}
\usepackage{sidecap}
\usepackage{wrapfig}

\newcommand{\R}{\mathbb{R}}
\newcommand{\rbr}[1]{\left(#1\right)}
\newcommand{\diag}{\mathop{\mathrm{diag}}}

\usepackage{hyperref}

\usepackage{Definitions}

\newcommand{\bB}{\bar B}
\newcommand{\bn}{\bar n}

\theoremstyle{definition}
\theoremstyle{remark}

\usepackage[accepted]{icml2021}

\icmltitlerunning{Optimization of  Graph Neural Networks: Implicit Acceleration by Skip Connections and More Depth}

\begin{document}

\twocolumn[
\icmltitle{Optimization of  Graph Neural Networks: \\ Implicit Acceleration by Skip Connections and More Depth}



\icmlsetsymbol{equal}{*}

\begin{icmlauthorlist}
\icmlauthor{Keyulu Xu}{equal,mit}
\icmlauthor{Mozhi Zhang}{m}
\icmlauthor{Stefanie Jegelka}{mit}
\icmlauthor{Kenji Kawaguchi}{equal,harvard}
\end{icmlauthorlist}

\icmlaffiliation{mit}{Massachusetts Institute of Technology (MIT)}
\icmlaffiliation{m}{The University of Maryland}
\icmlaffiliation{harvard}{Harvard University}

\icmlcorrespondingauthor{Keyulu Xu}{keyulu@mit.edu}
\icmlcorrespondingauthor{Kenji Kawaguchi}{kkawaguchi@fas.harvard.edu}

\icmlkeywords{deep learning, graph neural networks}

\vskip 0.3in
]

\printAffiliationsAndNotice{\icmlEqualContribution} 

\begin{abstract}
Graph Neural Networks (GNNs) have been studied through the lens of expressive power and generalization. However, their optimization properties are less well understood. We take the first step towards analyzing GNN training by studying the gradient dynamics of GNNs. First, we analyze linearized GNNs and prove that despite the non-convexity of training, convergence to a global minimum at a linear rate is guaranteed under mild assumptions that we validate on real-world graphs. Second, we study what may affect the GNNs' training speed. Our results show that the training  of GNNs is implicitly accelerated by skip connections, more depth, and/or a good label distribution. Empirical results confirm that our theoretical results for linearized GNNs align with the training behavior of nonlinear GNNs. Our results provide the first theoretical support for the success of GNNs with skip connections in terms of optimization, and suggest that deep GNNs with skip connections would be promising in practice.
\end{abstract}

\section{Introduction}
\label{sec:intro}

Graph Neural Networks (GNNs)~\citep{gori2005new, scarselli2009graph} are an effective framework for learning with graphs. GNNs learn node representations  on a graph by extracting high-level features not only from a node itself but also from a node's surrounding subgraph. Specifically, the node representations are recursively aggregated and updated using neighbor representations  ~\citep{merkwirth05, duvenaud2015convolutional, defferrard2016convolutional,  kearnes2016molecular,  gilmer2017neural,  hamilton2017inductive, velivckovic2017graph, liao2020graph}.   


Recently, there has been a surge of interest in studying the theoretical aspects of GNNs to understand their success and limitations. Existing works have studied GNNs' expressive power~\citep{keriven2019universal, maron2019provably, chen2019equivalence, xu2018how, sato2019approximation, Loukas2020}, generalization capability~\citep{scarselli2018vapnik, du2019graph, Xu2020What, garg2020generalization}, and extrapolation properties~\citep{xu2020neural}. However, the understanding of the optimization properties of GNNs has remained limited.  For example, researchers working on the fundamental problem of designing more expressive GNNs hope and often empirically observe that more powerful GNNs better fit the training set~\citep{xu2018how, sato2020random,  vignac2020building}. Theoretically, given the non-convexity of GNN training, it is still an open question whether better representational power always translates into smaller training loss. This motivates the more general questions: 
\begin{center}
    \textit{Can gradient descent find a global minimum for GNNs? What affects the speed of convergence in training?}
\end{center}
In this work, we take an initial step towards answering the questions above by analyzing the trajectory of gradient descent, i.e., \textit{gradient dynamics} or \textit{optimization dynamics}. A complete understanding of the dynamics of GNNs, and deep learning in general, is challenging. Following prior works on gradient dynamics~\citep{saxe2013exact, arora2019convergence, bartlett2019gradient}, we consider the  linearized regime, i.e., GNNs with \textit{linear} activation. Despite the linearity, key properties of nonlinear GNNs are present: The objective function is \textit{non-convex} and the dynamics are \textit{nonlinear}~\citep{saxe2013exact, kawaguchi2016deep}. Moreover, we observe the learning curves of linear GNNs and ReLU GNNs are surprisingly similar, both converging to nearly zero training loss at the same linear rate (Figure~\ref{fig:intro}).  Similarly, prior works report comparable performance in node classification benchmarks even if we remove the non-linearities~\citep{thekumparampil2018attention, wu2019simplifying}. Hence, understanding the dynamics of linearized GNNs is a valuable step towards understanding the general GNNs.

 Our analysis leads to an affirmative answer to the first question. We establish that gradient descent training of a linearized GNN with squared loss converges to a global minimum at a linear rate. Experiments confirm that the assumptions of our theoretical results for global convergence hold on real-world datasets. The most significant contribution of our convergence analysis is on multiscale GNNs, i.e., GNN architectures that use \textit{skip connections} to combine graph features at various scales~\citep{xu2018representation,li2019deepgcns,  abu2020n, chen2020simple, li2020deepergcn}. The skip connections introduce complex interactions among layers, and thus the resulting dynamics are more intricate. To our knowledge, our results are the first convergence results for GNNs with \textit{more than one} hidden layer, with or without skip connections.

\begin{figure}[!t]
    \centering
        \includegraphics[width=0.2382\textwidth]{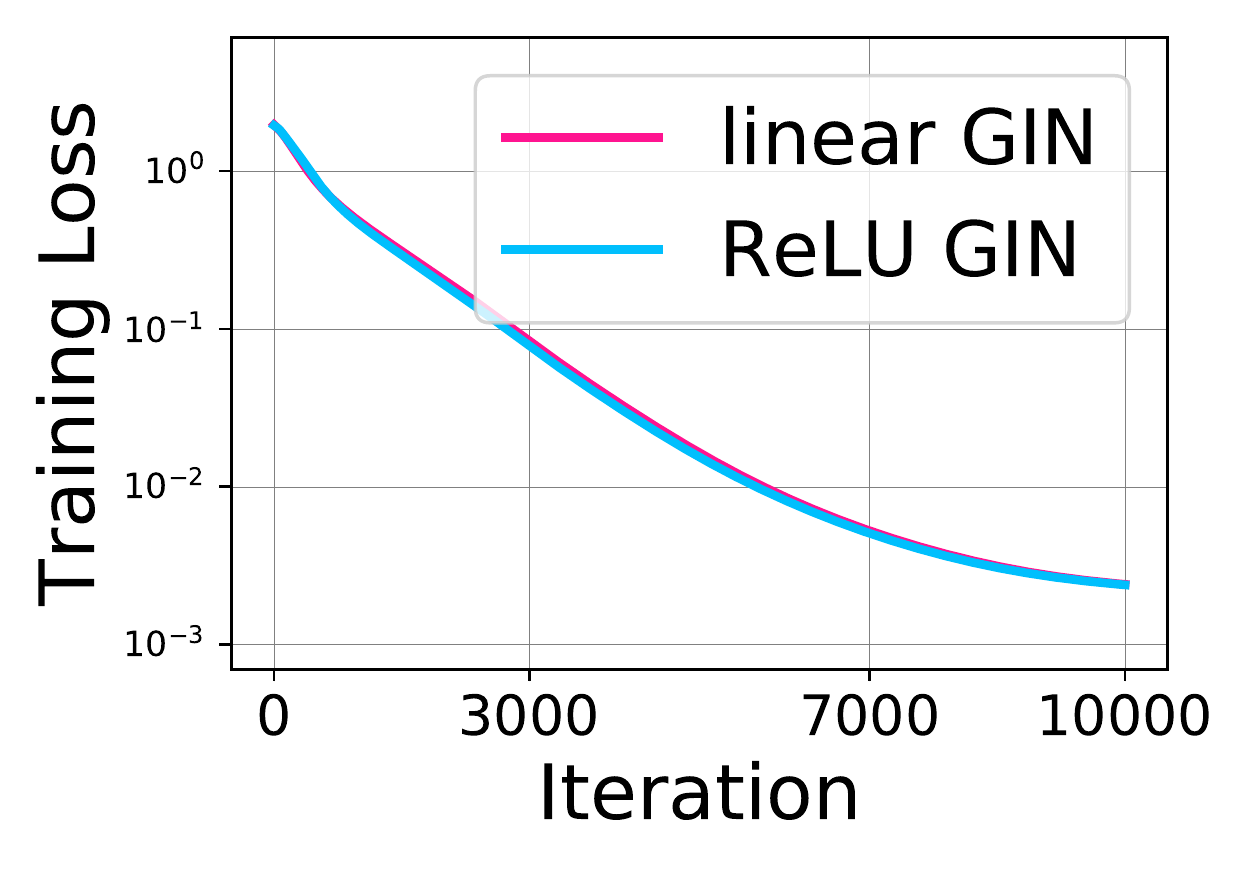}
         \includegraphics[width=0.2382\textwidth]{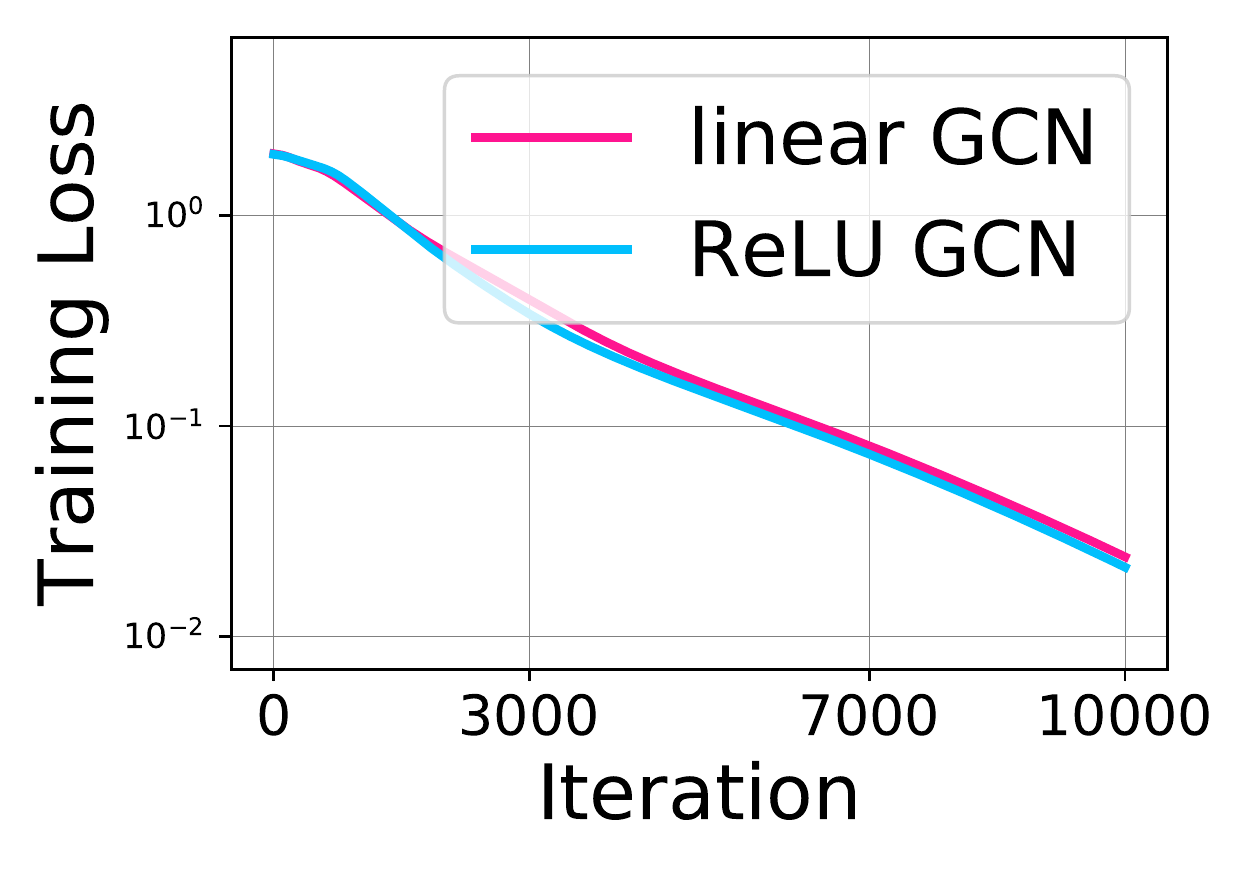}
   \caption{\textbf{Training curves of linearized GNNs vs. ReLU GNNs} on the Cora node classification dataset.  }
\label{fig:intro}
\end{figure}

We then study what may affect the training speed of GNNs.  First, for any fixed depth,  GNNs with skip connections train faster. Second, increasing the depth further accelerates the training of GNNs. Third, faster training is obtained when the labels are more correlated with the graph features, i.e., labels contain ``signal'' instead of ``noise''.  Overall, experiments for nonlinear GNNs agree with the prediction of our theory for linearized GNNs. 

Our results provide the first theoretical justification for the empirical success of multiscale GNNs in terms of optimization, and suggest that deeper GNNs with skip connections may be promising in practice. In the GNN literature, skip connections are initially motivated by the ``over-smoothing'' problem~\citep{xu2018representation}: via the recursive neighbor aggregation, node representations of a \textit{deep} GNN on expander-like subgraphs would be mixing features from almost the entire graph, and may thereby ``wash out'' relevant local information. In this case, shallow GNNs may perform better. Multiscale GNNs with \textit{skip connections} can combine and adapt to the graph features at various scales, i.e., the output of intermediate GNN layers, and such architectures are shown to help with this over-smoothing problem~\citep{xu2018representation, li2019deepgcns, li2020deepergcn, abu2020n, chen2020simple}. However, the properties of multiscale GNNs have mostly been understood at a conceptual level.   \citet{xu2018representation} relate the learned representations to random walk distributions and \citet{oono2020optimization} take a boosting view, but they do not consider the optimization dynamics. We give an  explanation from the lens of optimization. The training losses of deeper GNNs may be worse due to over-smoothing. In contrast, multiscale GNNs can express any shallower GNNs and fully exploit the  power by converging to a global minimum. Hence, our results suggest that deeper GNNs with skip connections are guaranteed to train faster with smaller training losses.

We present our results on global convergence in Section~\ref{sec:convergence}, after introducing  relevant background (Section~\ref{sec:preliminary}). In Section~\ref{sec:accelerate}, we compare the training speed of GNNs as a function of skip connections, depth, and the label distribution. All proofs are deferred to the Appendix.

\section{Preliminaries}
\label{sec:preliminary}

\begin{figure*}[!t]
    \centering
    \begin{subfigure}[b]{0.32\textwidth}
        \includegraphics[width=0.9\textwidth]{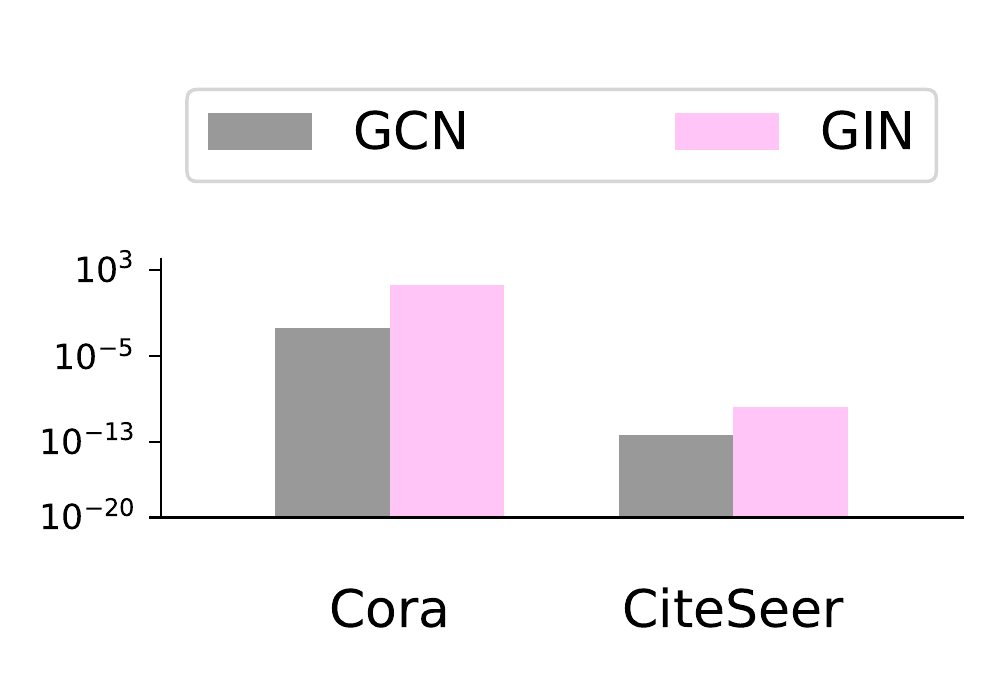} 
        \vspace{0.02in}
        \caption{Graph  $\sigma^2_{\min}(X(S^{H})_{*\Ical})$}
        \label{fig:c11}
    \end{subfigure}   \hspace{-0.03\textwidth}
      \centering
    \begin{subfigure}[b]{0.32\textwidth}
        \includegraphics[width=0.9\textwidth]{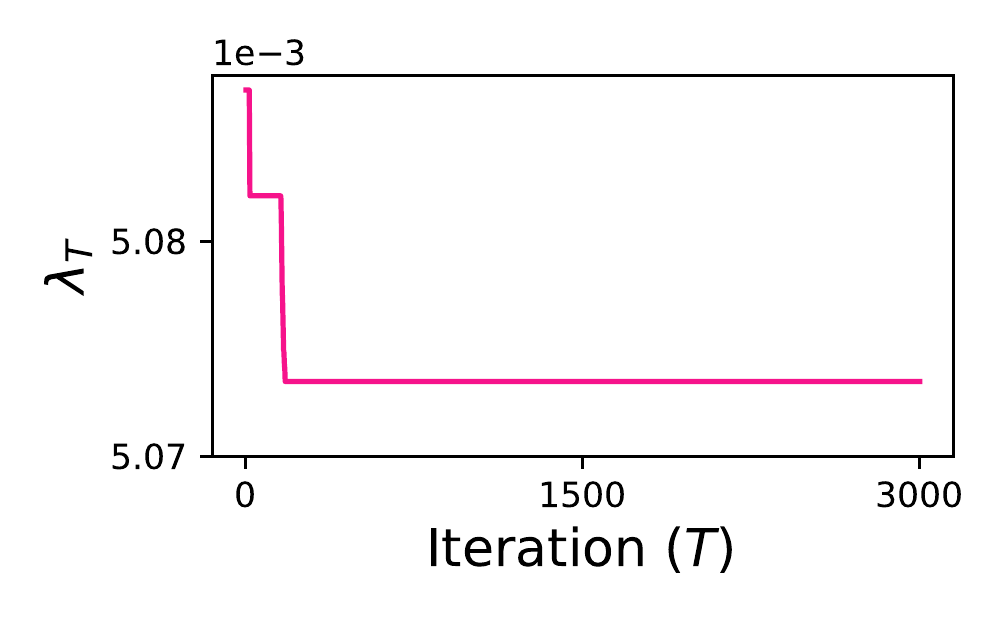} 
        \caption{Time-dependent  $\lambda_T^{(H)}$}
        \label{fig:c12}
    \end{subfigure}  
      \centering
      \begin{subfigure}[b]{0.32\textwidth}
        \includegraphics[width=0.9\textwidth]{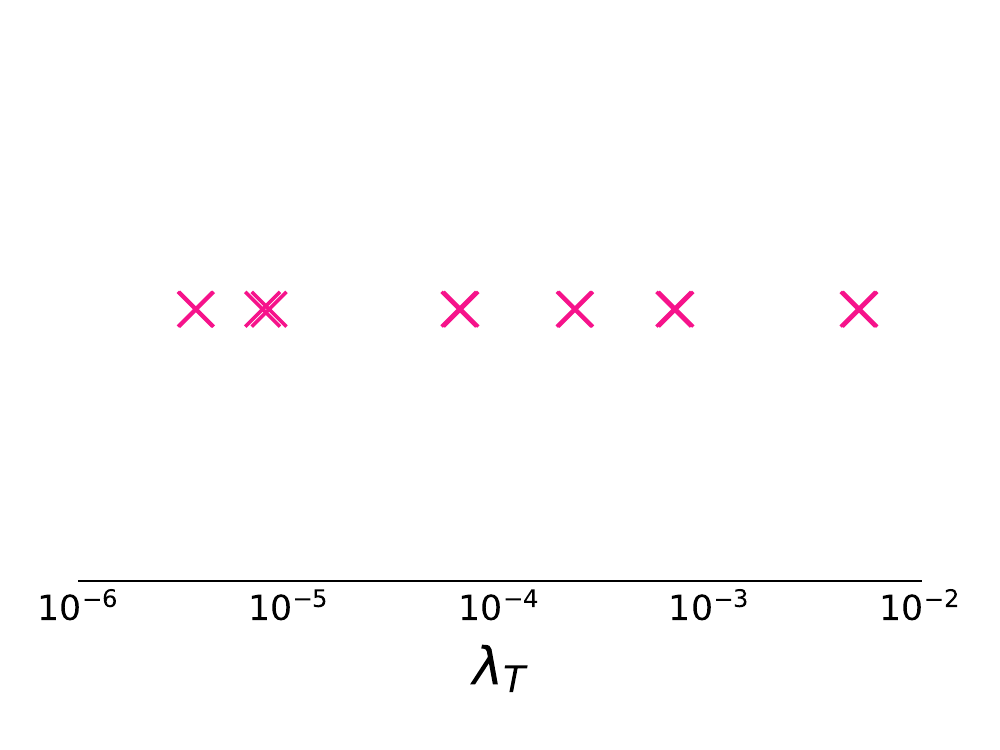} 
        \caption{$\lim_{T \rightarrow \infty}\lambda_T^{(H)}$ across training settings}
        \label{fig:c13}
    \end{subfigure}   
   \caption{\textbf{Empirical validation of assumptions for global convergence of  linear GNNs.} Left panel confirms the graph condition $\sigma^2_{\min}(X(S^{H})_{*\Ical})>0$ for  datasets Cora and Citeseer, and for models GCN  and GIN. Middle panel shows the time-dependent  $\lambda_T^{(H)}$ for one training setting (linear GCN on Cora). Each point in right panel is  $\lambda_T^{(H)} >0$ at the last iteration for different training settings. }
\label{fig:condition1}
\end{figure*}

\subsection{Notation and Background}
We begin by introducing our notation. Let $G = \left(V, E \right)$ be a graph with $n$ vertices $V = \{v_1, v_2, \cdots, v_n \}$. Its adjacency matrix  $A\in\R^{n\times n}$ has entries $A_{ij}=1$ if $(v_i,v_j)\in E$ and $0$ otherwise. The degree matrix associated with $A$ is  $D=\diag \rbr{d_1, d_2, \dots, d_n}$ with $d_i=\sum_{j=1}^{n} A_{ij}$. For any matrix $M \in \RR^{m \times m'}$, we denote its $j$-th column vector  by $M_{*j} \in \RR^m$, its $i$-th row vector by $M_{i*} \in \RR^{m'}$, and its largest and smallest (i.e., $\min(m, m')$-th largest) singular values by $\sigma_{\max}(M)$ and $\sigma_{\min}(M)$, respectively.
The data matrix 
$X \in \R^{m_x \times n}$ has columns $X_{\ast j}$ corresponding to the feature vector of node $v_j$, with input dimension $m_x$. 

The task of interest is node  classification or regression.  Each node $v_i \in V$ has an associated label $y_i \in \R^{m_y}$. In  the transductive (semi-supervised) setting, we have access to training labels for only a subset $\mathcal{I}  \subseteq [n]$ of nodes on $G$,  and the goal is to predict the labels for the other nodes in $[n] \setminus \mathcal{I}$. Our problem formulation easily extends to the inductive setting by letting $\Ical=[n]$, and we can use the trained model for prediction on unseen graphs. Hence, we have access to $\bn = |\Ical| \le n$ training labels $Y = [y_i]_{i \in \Ical}\in \R^{m_y \times \bn}$, and we train the GNN using $X, Y, G$. Additionally, for any  $M \in \RR^{m \times m'}$, $\Ical$ may index sub-matrices  $M_{*\Ical}=[M_{*i}]_{i\in \Ical} \in \RR^{m \times \bn}$ (when $m' \ge n$) and $M_{\Ical*}=[M_{i*}]_{i\in \Ical} \in \RR^{\bn \times m}$ (when $m \ge n$). 

Graph Neural Networks (GNNs) use the graph structure and node features to learn representations of nodes~\citep{scarselli2009graph}. GNNs maintain hidden representations $h^v_{(l)} \in \R^{m_l}$ for each node $v$, where $m_l$ is the hidden dimension on the $l$-th layer. We let $X_{(l)} = \big[h^1_{(l)}, h^2_{(l)}, \cdots, h^n_{(l)}\big]\in\R^{m_l \times n}$, and set $X_{(0)}$ as the  input features $X$. The node hidden representations $X_{(l)}$ are updated by  aggregating and transforming the neighbor representations:
\begin{align}
    X_{(l)} = \sigma \big( B_{(l)} X_{(l-1)} S \big) \in \R^{m_l \times n},
\end{align}
where $\sigma$ is a nonlinearity such as ReLU, $B_{(l)} \in \R^{m_l \times m_{l-1}}$ is the weight matrix, and $S \in \R^{n \times n}$ is the GNN aggregation matrix, whose  formula depends on the exact variant of GNN.  In Graph Isomorphism Networks (GIN)~\citep{xu2018how}, $S  = A  +I_n$ is  the adjacency matrix of $G$ with self-loop, where $I_n \in \R^{n \times n}$ is an identity matrix. In Graph Convolutional Networks (GCN)~\citep{kipf2016semi}, $S=\hat{D}^{-\frac{1}{2}} (A + I_n) \hat{D}^{-\frac{1}{2}}$ is the normalized adjacency matrix, where $\hat{D}$ is the degree matrix of $A+I_n$.

\subsection{Problem Setup}
We first formally define linearized GNNs.
\begin{definition} \label{def:1}
(Linear GNN). Given data  matrix $X \in \R^{m_x \times n}$, aggregation matrix $S \in \R^{n \times n}$, weight matrices $W \in \R^{m_y \times m_H}$,  $B_{(l)} \in \R^{m_l \times m_{l-1}}$, and their collection $B=(B_{(1)},\dots, B_{(H)})$, a linear GNN with $H$ layers $f(X, W, B)  \in \R^{m_y \times n}$  is defined as
\begin{align}
    f(X, W, B)  = W X_{(H)}, \quad
    X_{(l)} = B_{(l)} X_{(l-1)} S.
\end{align}
\end{definition}

Throughout this paper, we refer multiscale GNNs to the commonly used \textit{Jumping Knowledge Network (JK-Net)}~\citep{xu2018representation}, which connects the output of all intermediate GNN layers to the final layer with skip connections: 
\begin{definition} \label{def:2}
(Multiscale linear GNN). Given data  $X \in \R^{m_x \times n}$, aggregation matrix $S \in \R^{n \times n}$, weight matrices $W_{(l)} \in \R^{m_y \times m_l}$,  $B_{(l)} \in \R^{m_l \times m_{l-1}}$ with $W=(W_{(0)},W_{(1)},\dots, W_{(H)})$, a multiscale linear GNN with $H$ layers  $f(X, W, B)  \in \R^{m_y \times n}$ is defined as
\begin{align}
    f(X, W, B)  &= \sum_{l=0}^H W_{(l)} X_{(l)}, \\
    X_{(l)} &= B_{(l)} X_{(l-1)} S.
\end{align}
\end{definition}

Given a GNN $f(\cdot)$ and a loss function $\ell(\cdot, Y)$, we can train the GNN by minimizing the training loss $L(W, B)$:
\begin{align}
    L(W, B) = \ell \big( f(X, W, B)_{*\Ical}, Y \big),
\end{align}
where  $f(X, W, B)_{*\Ical}$ corresponds to the GNN's predictions on nodes that have training labels and thus incur training losses. The pair $(W,B)$ represents the trainable weights: 
$$
L(W, B) = L(W_{(1)},\dots,W_{(H)},B_{(1)},\dots,B_{(H)})
$$

For completeness, we define the global minimum of GNNs.
\begin{definition} \label{def:3}
(Global minimum). For any $H \in \NN_{0}$, $L^*_{H}$ is the global minimum value of the $H$-layer linear GNN $f$: 
\begin{align}
L^*_{H}=\inf_{W,B}  \ell \big( f_{}(X, W, B)_{*\Ical}, Y \big).
\end{align}
Similarly, we define $L^*_{1:H}$ as the global minimum value of the multiscale linear GNN $f$ with $H$ layers.
\end{definition}

We are ready to present our main results on global convergence for linear GNNs and multiscale linear GNNs.

\section{Convergence Analysis}
\label{sec:convergence}

In this section, we show that gradient descent training a linear GNN with squared loss, with or without skip connections, converges linearly to a global minimum. Our conditions for global convergence hold on real-world datasets and provably hold under assumptions, e.g., initialization.

In linearized GNNs, the loss $L(W, B)$ is non-convex (and non-invex) despite the linearity. The graph aggregation  $S$ creates  interaction among the data and poses additional challenges in the analysis. We show a fine-grained analysis of the GNN's gradient dynamics can overcome these challenges. Following previous works on gradient dynamics \citep{saxe2013exact,huang2020dynamics,ji2020directional, kawaguchi2021}, we analyze the GNN learning process via the  \textit{gradient flow}, i.e., gradient descent with infinitesimal steps: $\forall t \ge 0,$ the network weights evolve as
\begin{align}
\frac{d}{dt}W_{t} = - \frac{\partial L}{\partial W}(W_{t},B_{t}), \quad  \frac{d}{dt}B_{t} = - \frac{\partial L}{\partial B}(W_{t},B_{t}),
\end{align}
where $(W_{t}, B_t)$ represents the trainable parameters at time $t$ with  initialization $(W_{0}, B_0)$.

\subsection{Linearized GNNs} \label{sec:convergence:1}

Theorem~\ref{thm:6} states our result on global convergence for linearized GNNs without skip connections. 
\begin{theorem} \label{thm:6}
Let $f$ be an $H$-layer linear GNN and $\ell(q,Y)=\|q-Y\|^{2}_F$ where $q,Y \in \RR^{m_y \times \bn}$. Then, for any $T > 0$, 
\begin{align} \label{eq:thm:6:1} 
&L(W_T, B_T) -L^*_{H}
\\ \nonumber &\le(L(W_0,B_0) -L^*_{H})
 e^{- 4\lambda_T^{(H)} \sigma^2_{\min}(X(S^{H})_{*\Ical})T}, 
\end{align}
where $\lambda_T^{(H)}$ is the smallest eigenvalue $\lambda_T^{(H)}:= \inf_{t \in[0,T]}  \lambda_{\min}((\bB^{(1:H)}_t)\T \bB^{(1:H)}_t)$ and $\bB^{(1:l)}:=B_{(l)}B_{(l-1)} \cdots B_{(1)}$ for any $l \in \{0,\dots, H\}$ with $\bB^{(1:0)}:=I$. 
\end{theorem}
\begin{proof} (Sketch) We decompose the gradient dynamics into three components: the graph interaction,  non-convex factors, and convex factors. We then bound the effects of the graph interaction and  non-convex factors  through  $\sigma^2_{\min}(X(S^{H})_{*\Ical})$ and $\lambda_{\min}((\bB^{(1:H)}_t)\T \bB^{(1:H)}_t)$ respectively. The complete proof is in Appendix~\ref{sec:proof_thm:6}.
\end{proof}
Theorem \ref{thm:6} implies that  convergence to a global minimum at a linear rate is guaranteed if $\sigma^2_{\min}(X(S^{H})_{*\Ical})>0$ and $\lambda_T > 0$. The first condition on  
the product of $X$ and $S^{H}$ indexed by $\Ical$ only depends on the node features $X$ and the GNN aggregation matrix $S$.  It is satisfied if $\rank(X(S^{H})_{*\Ical})=\min(m_{x}, \bn)$, because $\sigma_{\min}(X(S^{H})_{*\Ical})$ is  the  $\min(m_{x}, \bn)$-th largest singular value of $X(S^{H})_{*\Ical} \in \RR^{m_x \times \bn}$.  The second condition $\lambda_T^{(H)}>0$ is  time-dependent and requires a more careful treatment.
Linear convergence is implied as long as $\lambda_{\min}((\bB^{(1:H)}_t)\T \bB^{(1:H)}_t) \geq \epsilon > 0$ for all times $t$  before stopping.

\textbf{Empirical validation of conditions.} We verify both the  graph condition $\sigma^2_{\min}(X(S^{H})_{*\Ical})>0$ and the time-dependent condition $\lambda_T^{(H)} > 0$ for (discretized) $T>0$. First, on the popular graph datasets, Cora and Citeseer~\cite{sen2008collective}, and the GNN models, GCN~\cite{kipf2016semi} and GIN~\cite{xu2018how}, we have $\sigma^2_{\min}(X(S^{H})_{*\Ical})>0$ (Figure~\ref{fig:c11}). Second, we train linear GCN and GIN on Cora and Citeseer to plot an example of how the $\lambda_T^{(H)} = \inf_{t \in[0,T]}  \lambda_{\min}((\bB^{(1:H)}_t)\T \bB^{(1:H)}_t)$ changes with respect to time $T$ (Figure~\ref{fig:c12}). We further confirm that $\lambda_T^{(H)}>0$  until convergence, $\lim_{T \rightarrow \infty} \lambda_T^{(H)} > 0$ across different settings, e.g., datasets, depths, models (Figure~\ref{fig:c13}). Our experiments use the squared loss, random initialization, learning rate 1e-4, and set the hidden dimension to the input dimension (note that Theorem~\ref{thm:6} assumes the hidden dimension is at least the input dimension). Further experimental details are in Appendix~\ref{sec:experiments}. Along with Theorem~\ref{thm:6}, we conclude that linear GNNs converge linearly to a global minimum. Empirically, we indeed see both linear and ReLU  GNNs converging  at the same linear rate to nearly zero training loss in node classification tasks (Figure~\ref{fig:intro}). 

\textbf{Guarantee via initialization.}  Besides  the empirical verification, we theoretically show that a \textit{good initialization} guarantees the time-dependent condition $\lambda_T>0$ for any $T > 0$.  Indeed, like other neural networks, GNNs do not converge to a global optimum with certain initializations: e.g., initializing all weights to zero leads to zero gradients and $\lambda_T^{(H)}=0$ for all $T$, and hence no learning. We introduce a notion of \textit{singular margin} and say an initialization is good if it has a positive singular margin. Intuitively, a good initialization starts with an already small loss.

\begin{definition} \label{def:4}
(Singular margin). The initialization $(W_0,B_0)$ is said to have 
singular margin $\gamma>0$ with respect to a layer $l \in \{1,\dots, H\}$ if $\sigma_{\min}(B_{(l)}B_{(l-1)} \cdots B_{(1)}) \ge \gamma$ for all $(W,B)$ such that $L(W,B)\le  L(W_0, B_0)$.
\end{definition}

Proposition~\ref{prop:2} then states that an initialization with positive singular margin $\gamma$ guarantees  $\lambda_T^{(H)} \ge \gamma^2 > 0$ for all $T$:
\begin{proposition} \label{prop:2}
Let $f$ be a linear GNN with $H$ layers and $\ell(q,Y)=\|q-Y\|^{2}_F$. If the initialization $(W_0,B_0)$ has 
singular margin $\gamma>0$ with respect to the layer $H$ and $m_{H}\ge m_{x}$, then  $\lambda_T^{(H)} \ge \gamma^{2}$ for all $T \in [0, \infty)$.
\end{proposition}
Proposition~\ref{prop:2} follows since $L(W_t,B_t)$ is  non-increasing with respect to time $t$ (proof in Appendix~\ref{sec:proof_prop:2}).

Relating to previous works,  our singular margin is a generalized variant of the deficiency margin of linear feedforward networks~\citep[Definition 2 and Theorem 1]{arora2019convergence}: 
\begin{proposition} \label{prop:1}
(Informal) If initialization $(W_0,B_0)$ has 
deficiency margin $c>0$, then it has  singular margin $\gamma>0$.  \end{proposition}
The formal version of Proposition~\ref{prop:1} is in Appendix~\ref{sec:complete}.

To summarize,  Theorem \ref{thm:6} along with Proposition \ref{prop:2} implies that we have a prior guarantee of  linear  convergence to a global minimum for any graph with $\rank(X(S^{H})_{*\Ical})=\min(m_{x}, \bn)$ and initialization   $(W_0,B_0)$ with singular margin $\gamma>0$: i.e., for any desired  $\epsilon>0$, we have that $L(W_T,B_T) - L^*_{H}\le \epsilon$ for any $T$ such that
\begin{align}
 \quad \text{  } T \ge  \frac{1}{4\gamma^{2}\sigma_{\min}^2(X(S^{H})_{*\Ical})} \log \frac{L(A_{0},B_{0}) -L^*_{H}}{\epsilon}.
\end{align}
While the margin condition theoretically guarantees linear convergence, empirically, we have already seen that the convergence conditions of across different training settings for widely used random initialization. 

Theorem~\ref{thm:6} suggests that the convergence rate depends on a combination of data features $X$, the GNN architecture and graph structure via $S$ and $H$,  the label distribution and initialization via $\lambda_T$. For example, GIN has better such constants than GCN on the Cora dataset with everything else held equal (Figure~\ref{fig:c11}). Indeed, in practice, GIN  converges faster than GCN on Cora (Figure~\ref{fig:intro}). In general, the computation and comparison of the rates given by Theorem~\ref{thm:6} requires computation such as those in Figure~\ref{fig:condition1}. In Section~\ref{sec:accelerate}, we will study an alternative way of comparing the  speed of training by directly comparing the gradient dynamics.

\begin{figure*}[!t]
    \centering
    \begin{subfigure}[b]{0.32\textwidth}
        \includegraphics[width=0.9\textwidth]{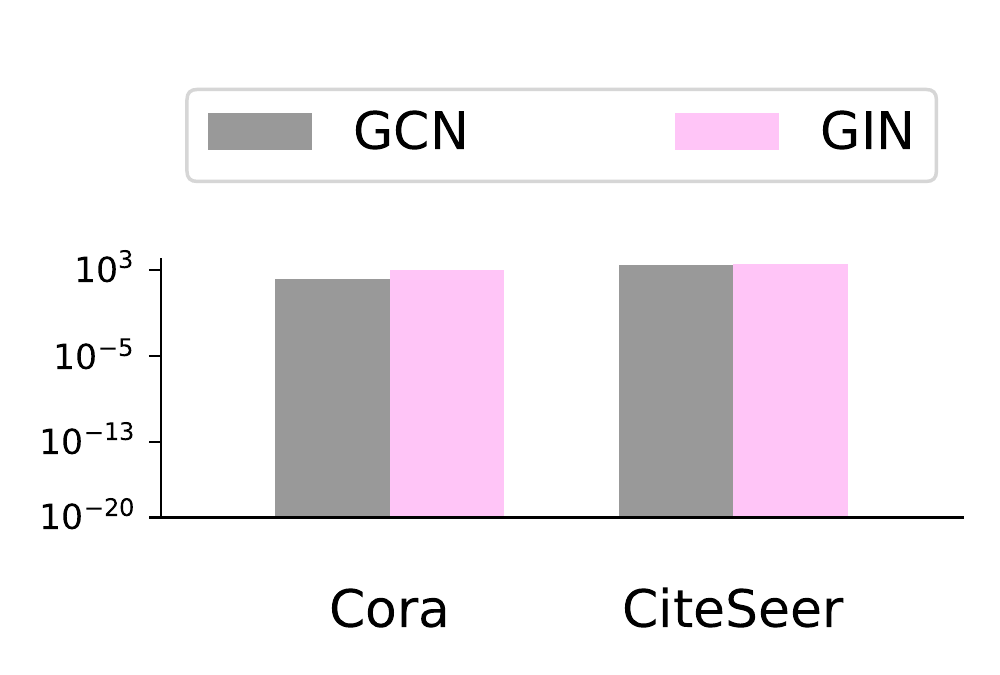} 
        \vspace{0.02in}
        \caption{Graph  $\sigma^2_{\min}((G_{H})_{*\Ical})$}
        \label{fig:c21}
    \end{subfigure}   \hspace{-0.03\textwidth}
      \centering
    \begin{subfigure}[b]{0.32\textwidth}
        \includegraphics[width=0.9\textwidth]{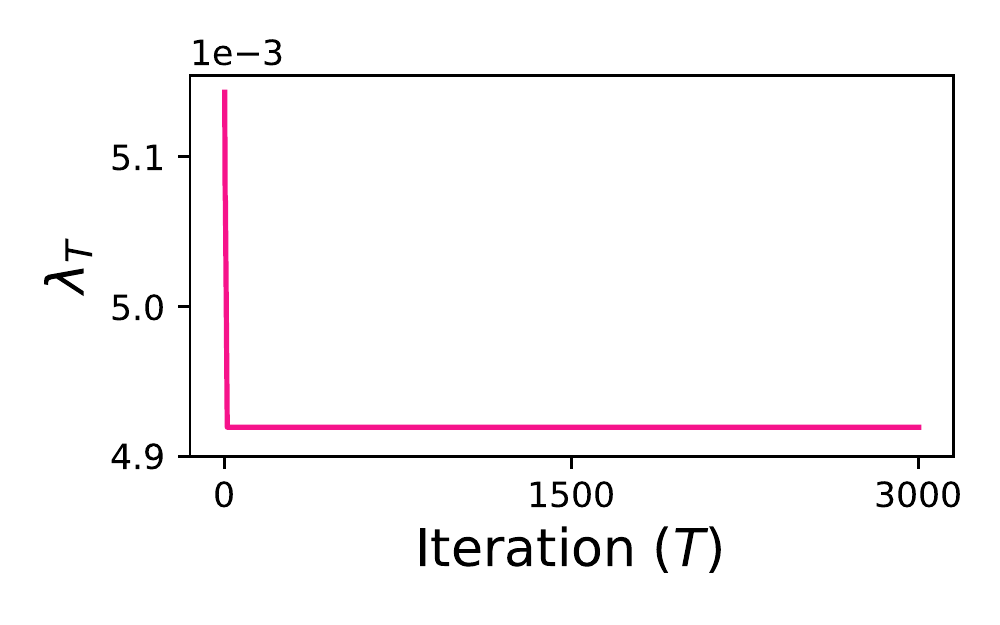} 
        \caption{Time-dependent $\lambda_T^{(1:H)}$}
        \label{fig:c22}
    \end{subfigure}  
      \centering
      \begin{subfigure}[b]{0.32\textwidth}
        \includegraphics[width=0.9\textwidth]{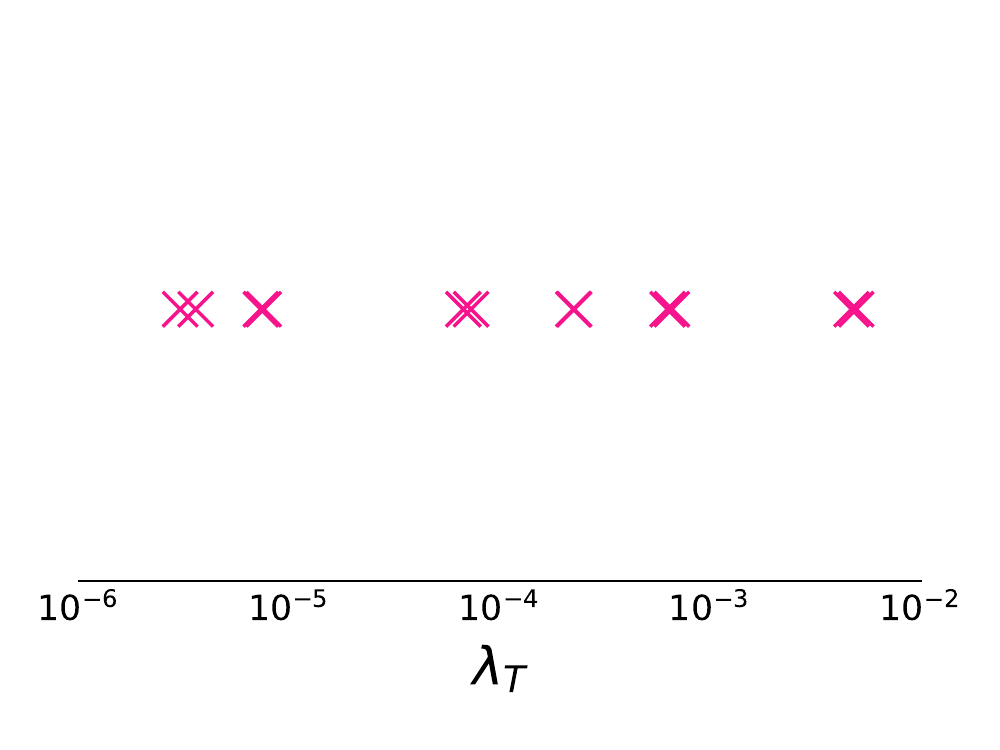} 
        \caption{$\lim_{T \rightarrow \infty}\lambda_T^{(1:H)}$ across training settings}
        \label{fig:c23}
    \end{subfigure}   
   \caption{\textbf{Empirical validation of assumptions for global convergence of multiscale linear GNNs.} Left panel confirms the graph condition $\sigma^2_{\min}((G_{H})_{*\Ical})>0$ for   Cora and Citeseer, and for  GCN  and GIN. Middle panel shows the time-dependent $\lambda_T^{(1:H)}$ for one training setting (multiscale linear GCN on Cora). Each point in right panel is $\lambda_T^{(1:H)} >0$ at the last  iteration for different training settings. }
\label{fig:condition2}
\end{figure*}

\subsection{Multiscale Linear GNNs}

Without skip connections, the GNNs under linearization still behave like linear feedforward networks with augmented graph features. With skip connections, the dynamics and analysis become much more intricate. The expressive power of multiscale linear GNNs changes significantly as depth increases. Moreover, the skip connections create complex interactions among different layers and graph structures of various scales in the optimization dynamics. Theorem~\ref{thm:1} states our convergence results for multiscale linear GNNs in three cases: (i) a general form; (ii) a weaker condition for boundary cases that uses $\lambda_T^{H'}$ instead of $\lambda_T^{1:H}$; 
(iii) a faster rate if we have  monotonic expressive power as depth increases. 

\begin{theorem} \label{thm:1}
Let $f$ be a multiscale linear GNN with $H$ layers and $\ell(q,Y)=\|q-Y\|^{2}_F$ where $q,Y \in \RR^{m_y \times \bn}$. Let $\lambda_T^{(1:H)}:=  \min_{0\le l \le H}\lambda_T^{(l)}$.   For any $T > 0$, the following hold:
\begin{enumerate}
\item[(i)] \emph{(General)}. Let $G_H :=  [X\T, (XS)\T, \dots, (XS^H)\T]\T$ $\in \RR^{(H+1)m_{x}  \times n}$. Then  
\begin{align}  \label{eq:7}
&L(W_T, B_T) -L^*_{1:H}
\\ \nonumber &\le(L(W_0,B_0) -L^*_{1:H})
 e^{- 4\lambda_T^{(1:H)} \sigma^2_{\min}((G_{H})_{*\Ical})T}. 
\end{align}

\item[(ii)] \emph{(Boundary cases)}. For any $H'  \in \{0,1,\dots, H\}$, 
\begin{align} \label{eq:6}
&L(W_T, B_T) -L^*_{H'}
\\ \nonumber &\le(L(W_0,B_0) -L^*_{H'})
 e^{- 4\lambda_T^{(H')} \sigma^2_{\min}(X(S^{H'})_{*\Ical})T}.
\end{align}

\item[(iii)] \emph{(Monotonic expressive power)}.
If there exist   $l,l' \in \{0,\dots, H\}$ with  $l< l'$ such that $L^*_{l} \ge L^*_{l+1} \ge \cdots \ge L^*_{l'}$  or    $L^*_{l} \le L^*_{l+1} \le \cdots \le L^*_{l'}$, then\begin{align} 
&L(W_T, B_T) -L^*_{l''}
\\ \nonumber &\le(L(W_0,B_0) -L^*_{l''})
 e^{- 4 \sum_{k=l}^{l'} \lambda_T^{(k)} \sigma^2_{\min}(X(S^{k})_{*\Ical})T},
\end{align}
where $l''=l$ if $L^*_{l} \ge L^*_{l+1} \ge \cdots \ge L^*_{l'}$, and $l''=l'$ if $L^*_{l} \le L^*_{l+1} \le \cdots \le L^*_{l'}$. 
\end{enumerate}
\end{theorem}
\begin{proof} (Sketch)  A key observation in our proof is that the  interactions of different scales  cancel out to point towards a specific direction in the gradient dynamics induced in a space of the loss value. The complete proof is in Appendix~\ref{sec:proof_thm:1}.
\end{proof}

Similar to Theorem \ref{thm:6} for linear GNNs, the most general form (i) of Theorem \ref{thm:1} implies that convergence to the global minimum value of the \textit{entire} multiscale linear GNN $L_{1:H}^*$ at linear rate is guaranteed when $\sigma^2_{\min}((G_{H})_{*\Ical})>0$ and $\lambda_T^{(1:H)}>0$. The graph condition  $\sigma^2_{\min}((G_{H})_{*\Ical})>0$ is satisfied if $\rank((G_{H})_{*\Ical})=\min(m_{x}(H+1), \bn)$. The time-dependent condition $\lambda_T^{(1:H)}>0$ is guaranteed if the initialization $(W_0,B_0)$ has singular margin $\gamma>0$ with respect to \textit{every} layer (Proposition~\ref{prop:3} is proved in Appendix~\ref{sec:proof_prop:3}): 
\begin{proposition} \label{prop:3}
Let $f$ be a multiscale linear GNN and $\ell(q,Y)=\|q-Y\|^{2}_F$.  If the initialization $(W_0,B_0)$ has 
singular margin $\gamma>0$ with respect to every  layer $l \in [H]$ and $m_l\ge m_x$ for  $l\in[H]$, then  $\lambda_T^{(1:H)}\ge \gamma^{2}$ for all $T \in [0, \infty)$.
\end{proposition}

We demonstrate that the conditions of Theorem \ref{thm:1} (i) hold for real-world datasets, suggesting in practice multiscale linear GNNs  converge linearly to a global minimum. 

\textbf{Empirical validation of conditions.} On datasets Cora and Citeseer and for GNN models GCN and GIN, we confirm that $\sigma^2_{\min}((G_H)_{*\Ical})>0$ (Figure~\ref{fig:c21}). Moreover, we train multiscale linear GCN and GIN on Cora and Citeseer to plot an example of how the $\lambda_T^{(1:H)}$ changes with respect to time $T$ (Figure~\ref{fig:c22}), and we confirm that at convergence, $\lambda_T^{(1:H)} > 0$ across different settings (Figure~\ref{fig:c23}). Experimental details are in Appendix~\ref{sec:experiments}. 

\textbf{Boundary cases.} Because the global minimum value of  multiscale linear GNNs $L_{1:H}^*$ can be smaller than that of  linear GNNs $L_{H}^*$, 
the  conditions in Theorem \ref{thm:1}(i) may sometimes be stricter than those of Theorem~\ref{thm:6}. For example, in Theorem \ref{thm:1}(i), we require $ \lambda_T^{(1:H)}:= \min_{0\le l \le H}\lambda_T^{(l)} $ rather than $\lambda_T^{(H)}$ to be positive. If $\lambda_T^{(l)} = 0$ for some $l$, then Theorem \ref{thm:1}(i) will not guarantee convergence to $L_{1:H}^*$.

Although the boundary cases above did not occur on the tested real-world graphs (Figure~\ref{fig:condition2}), for theoretical interest, Theorem~\ref{thm:1}(ii) guarantees that in such cases, multiscale linear GNNs still converge to a value no worse than the global minimum value of  \textit{non-multiscale} linear GNNs. For any intermediate layer $H^{\prime}$, assuming $\sigma^2_{\min}(X(S^{H'})_{*\Ical}) > 0$ and $\lambda_T^{(H')} > 0$, Theorem~\ref{thm:1}(ii)  bounds the  loss  of the multiscale linear GNN $L(W_T, B_T)$ at convergence by the global minimum value  $L^*_{H'}$ of the corresponding linear GNN with $H^{\prime}$ layers.

\textbf{Faster rate under monotonic expressive power.} Theorem \ref{thm:1}(iii) considers a special case that is likely in real graphs: the global minimum value of the non-multiscale linear GNN  $L^*_{H'}$  is \textit{monotonic} as $H'$ increases. Then (iii) gives a \textit{faster rate} than (ii) and linear GNNs. For example, if the globally optimal value decreases as linear GNNs  get deeper. i.e., $L^*_{0} \ge L^*_{1} \ge \cdots \ge L^*_{H}$, or vice versa, $L^*_{0} \le L^*_{1} \le \cdots \le L^*_{H}$, then Theorem \ref{thm:1} (i) implies that 
\begin{align}
& L(W_T, B_T) -L^*_{l}
\\ \nonumber & \le(L(W_0,B_0) -L^*_{l})
 e^{- 4 \sum_{k=0}^{H} \lambda_T^{(k)} \sigma^2_{\min}(X(S^{k})_{*\Ical})T},
\end{align}
where $l=0$ if $L^*_{0} \ge L^*_{1} \ge \cdots \ge L^*_{H}$, and $l=H$ if $L^*_{0} \le L^*_{1} \le \cdots \le L^*_{H}$. 
Moreover, if the globally optimal value does not change with respect to the depth as $L^*_{1:H}=L^*_{1} = L^*_{2} = \cdots = L^*_{H}$, then we have
\begin{align}
& L(W_T, B_T) -L^*_{1:H}
\\ \nonumber &  \le(L(W_0,B_0) -L^*_{1:H})
 e^{- 4 \sum_{k=0}^{H} \lambda_T^{(k)} \sigma^2_{\min}(X(S^{k})_{*\Ical})T}.
\end{align}
We obtain a faster rate for multiscale linear GNNs than for linear GNNs, as
$ e^{-4 \sum_{k=0}^{H} \lambda_T^{(k)} \sigma^2_{\min}(X(S^{k})_{*\Ical})T} \le   e^{-4  \lambda_T^{(H)} \sigma^2_{\min}(X(S^{H})_{*\Ical})T}$. Interestingly, unlike linear GNNs, multiscale linear GNNs in this case do not require any condition on initialization to obtain a prior guarantee on global convergence since $
 e^{- 4 \sum_{k=0}^{H} \lambda_T^{(k)} \sigma^2_{\min}(X(S^{k})_{*\Ical})T} \le
 e^{- 4 \lambda_T^{(0)} \sigma^2_{\min}(X(S^{0})_{*\Ical})T}$ with $\lambda_T^{(0)} =1$ and $X(S^{0})_{*\Ical}=X_{*\Ical}$.

To summarize, we prove global convergence rates for multiscale linear GNNs (Thm.~\ref{thm:1}(i)) and experimentally validate the conditions. Part (ii) addresses boundary cases where the conditions of Part (i) do not hold.  Part (iii) gives faster rates assuming monotonic expressive power with respect to depth. So far, we have shown multiscale linear GNNs converge faster than linear GNNs in the case of (iii). Next, we compare the training speed for more general cases.

\section{Implicit Acceleration}
\label{sec:accelerate}

\begin{figure*}[!t]
    \centering
    \begin{subfigure}[b]{0.32\textwidth}
        \includegraphics[width=0.9\textwidth]{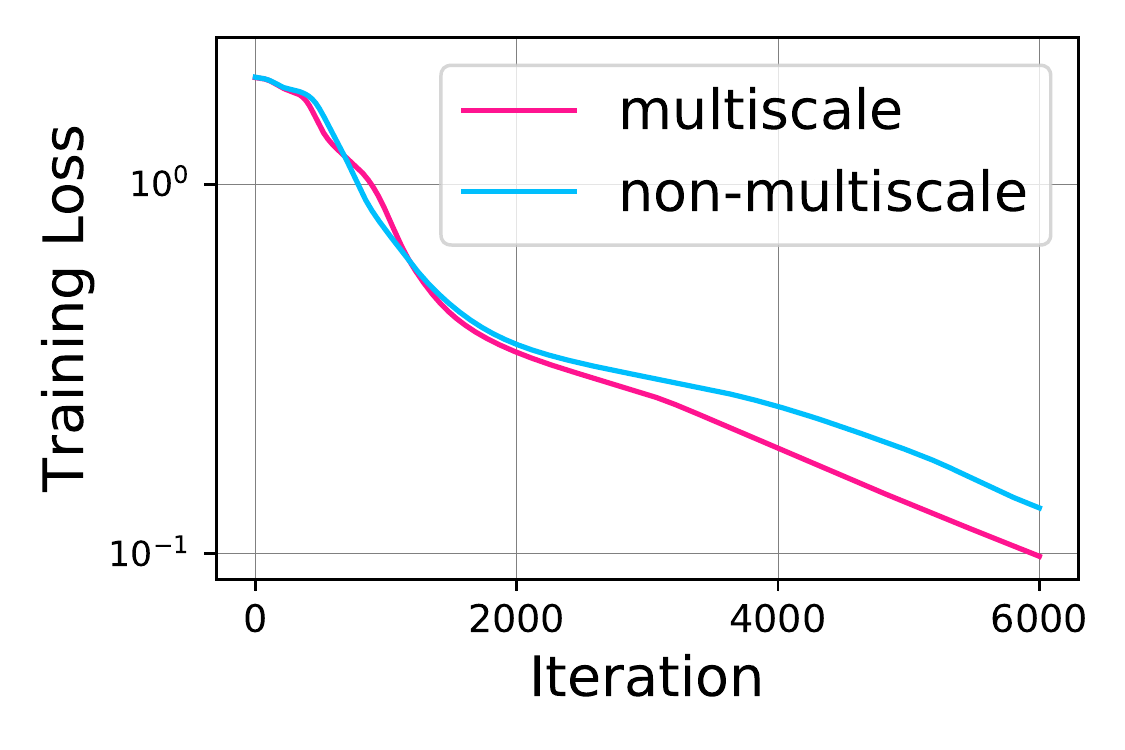} 
        \caption{Multiscale vs. non-multiscale.}
        \label{fig:c1}
    \end{subfigure}   
      \centering
    \begin{subfigure}[b]{0.32\textwidth}
        \includegraphics[width=0.9\textwidth]{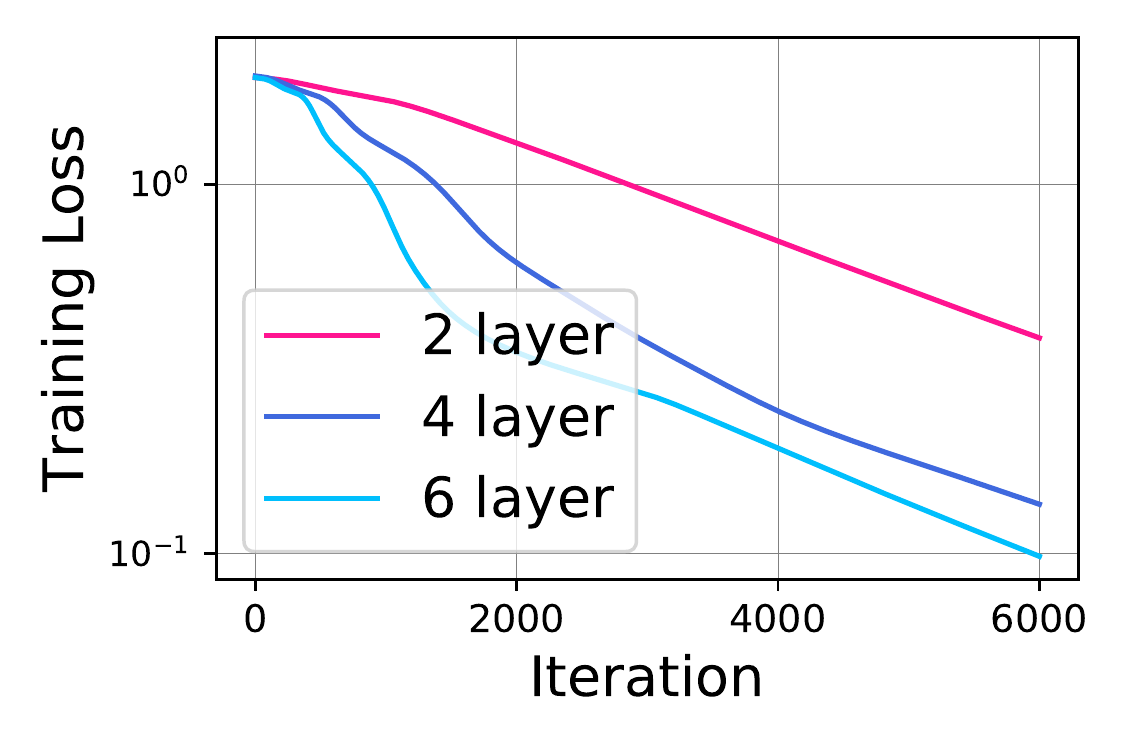} 
        \caption{Depth.}
        \label{fig:c2}
    \end{subfigure}  
      \centering
      \begin{subfigure}[b]{0.32\textwidth}
        \includegraphics[width=0.9\textwidth]{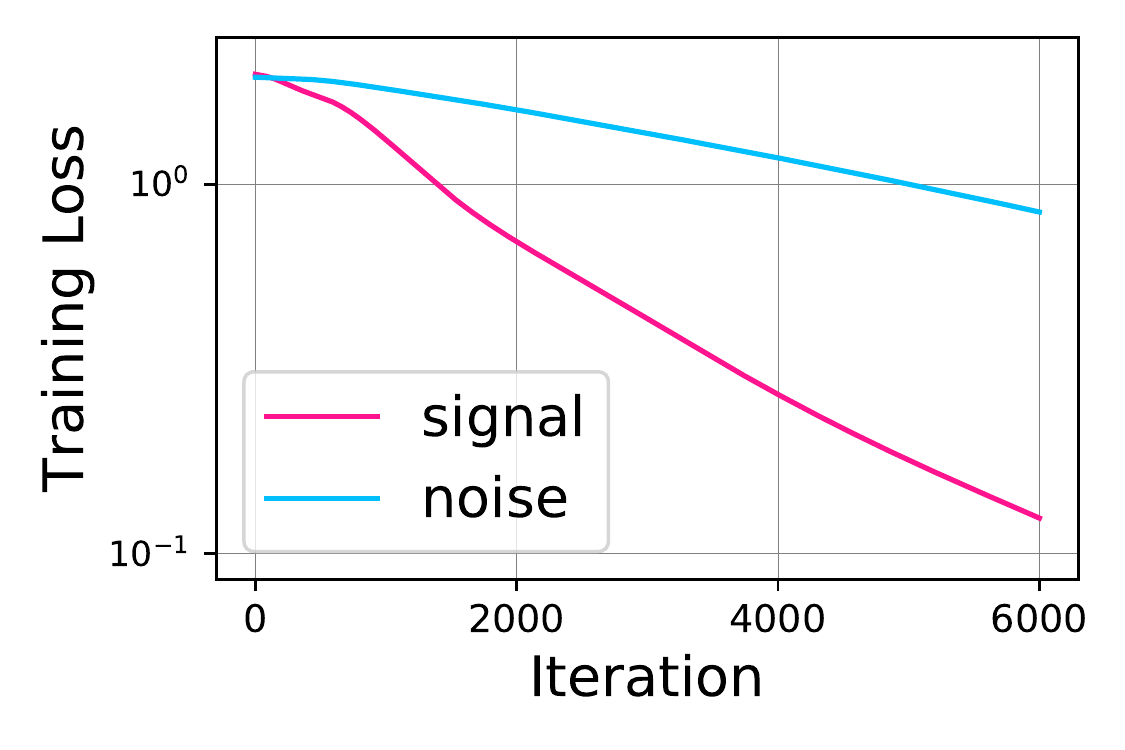}
        \caption{Signal vs. noise.}
        \label{fig:c3}
    \end{subfigure}   
   \caption{\textbf{Comparison of the training speed of GNNs.} Left: Multiscale GNNs train faster than non-multiscale GNNs. Middle: Deeper GNNs train faster. Right: GNNs train faster when the labels have signals instead of random noise. The patterns above hold for both ReLU and linear GNNs. Additional results are in Appendix~\ref{sec:results}. }
\label{fig:comparison}
\end{figure*}

In this section, we study how the  skip connections, depth of GNN, and label distribution may  affect the speed of training for GNNs. Similar to previous works~\citep{arora2018optimization}, we compare the training speed by comparing the per step loss reduction $\frac{d}{dt}  L(W_t, B_t)  $ for \textit{arbitrary} differentiable loss functions $\ell (\cdot, Y): \RR^{m_y} \rightarrow \RR$. Smaller $\frac{d}{dt}  L(W_t, B_t)   $ implies faster training. Loss reduction offers a complementary view to the  convergence rates in Section~\ref{sec:convergence}, since it is instant and not an upper bound.

We present an analytical form of the loss reduction  $\frac{d}{dt}  L(W_t, B_t)$ for linear GNNs and multiscale linear GNNs. The comparison of training speed  then follows from our formula for $\frac{d}{dt}  L(W_t, B_t)$. For better exposition, we first introduce several notations. We let $\bB^{(l':l)}=B_{(l)}B_{(l-1)} \cdots B_{(l')}$ for all $l'$ and $l$ where $\bB^{(l':l)}=I$ if $l'>l$. We also define
$$
J_{(i,l),t}:=[\bB^{(1:i-1)}_t\otimes (W_{(l),t} \bB^{(i+1:l)}_t)\T  ],
$$
$$
F_{(l),t}:=[(\bB^{(1:l)}_t)\T \bB^{(1:l)}_t \otimes I_{m_{y}}] \succeq 0,
$$ 
$$
V _{t}:=\frac{\partial L(W_t, B_t)}{\partial \hat Y_t}, 
$$
where $\hat Y_t :=f(X, W_{t}, B_{t})_{*\Ical}$.
For any vector $v \in \RR^m$ and positive semidefinite matrix $M\in\RR^{m\times m}$, we use $\|v\|_{M}^2 :=v\T Mv$.\footnote{We use this Mahalanobis norm notation for conciseness without assuming  it to be a norm, since $M$ may be low rank. }
Intuitively, $V_{t}$ represents the derivative of the loss $L(W_t, B_t)$ with respect to the model output $\hat Y=f(X, W_{t}, B_{t})_{*\Ical}$.  
$J_{(i,l),t}$ and $F_{(l),t}$ represent matrices  that describe how the errors are propagated through the weights of the networks. 

Theorem~\ref{thm:4}, proved in Appendix~\ref{sec:proof_thm:4}, gives an analytical formula of loss reduction for linear GNNs and multiscale linear GNNs.

\begin{theorem} \label{thm:4}
For any differentiable loss function  $q \mapsto \ell(q,Y)$, the following hold for any $H \ge 0$ and $t \ge 0$:
\begin{enumerate}[leftmargin=0.6cm]
\item[(i)] \emph{(Non-multiscale)} For $f$ as in Definition \ref{def:1}:
\begin{align} \label{eq:thm4:1} 
\hspace{-15pt} \frac{d}{dt}  L_1(W_t, B_t)  
 & =- \scalebox{0.95}{$\displaystyle \left\|\vect\left[V _{t}(X(S^{H})_{*\Ical})\T \right]\right\|_{F_{(H),t}}^{2}$} 
\\ \nonumber & \hspace{12pt} \scalebox{0.95}{$\displaystyle -  \sum_{i=1}^H  \left\|   J_{(i,H),t}\vect\left[V _{t}(X(S^{H})_{*\Ical})\T \right]\right\|_2^2.$}
\end{align}
\item[(ii)] \emph{(Multiscale)} For $f$ as in  Definition \ref{def:2}:
\begin{align} \label{eq:thm4:2} 
\hspace{-10pt} \frac{d}{dt} L_2(W_t, B_t) &= \scalebox{0.95}{$\displaystyle - \sum_{l=0}^H \left\|\vect\left[V _{t}(X(S^{l})_{*\Ical})\T \right]\right\|_{F_{(l),t}}^{2}$}
\\ \nonumber & \hspace{12pt} \scalebox{0.9}{$\displaystyle -  \sum_{i=1}^H  \left\| \sum_{l=i}^H   J_{(i,l),t}\vect\left[V_{t} (X(S^{l})_{*\Ical})\T \right]\right\|_2^{2}.$}
\end{align}
\end{enumerate}
\end{theorem}

In what follows, we apply Theorem~\ref{thm:4} to predict how different factors affect the training speed of GNNs.

\subsection{Acceleration with Skip Connections}

We first show that multiscale linear GNNs  tend to achieve faster loss reduction $\frac{d}{dt} L_2(W_t, B_t)$ compared to the corresponding linear GNN without skip connections, $\frac{d}{dt} L_1(W_t, B_t) $.  It follows from Theorem \ref{thm:4}  that
\begin{align}
\label{eq:multi}
&\frac{d}{dt} L_2(W_t, B_t) - \frac{d}{dt} L_1(W_t, B_t) 
\\ \nonumber & \le- \sum_{l=0}^{H-1} \left\|\vect\left[V_{t} (X(S^{l})_{*\Ical})\T \right]\right\|_{F_{(l),t}}^{2}, 
\end{align}
if $\sum_{i=1}^H (\|a_i \|_2^2 + 2 b_i\T a_i) \ge 0$,
where  $a_i=\sum_{l=i}^{H-1}   J_{(i,l),t}\vect[V_{t} (X(S^{l})_{*\Ical})\T ]$, and $b_i= J_{(i,H),t} \vect[ \allowbreak  V_{t}( X(S^{H})_{*\Ical})\T ]$. The assumption of $\sum_{i=1}^H (  \|a_i \|_2^2  + \allowbreak 2 b_i\T a_i) \allowbreak \ge \allowbreak  0$ is satisfied in various ways: for example, it is satisfied if the last layer's term $b_i$ and the other  layers' terms $a_i$ are aligned as $b_i\T a_i \ge 0$,  or if the last layer's term $b_i$ is dominated by the other  layers' terms $a_i$ as $2\| b_i\|_2 \le \|a_i\|_2$. Then equation \eqref{eq:multi} shows that the  multiscale linear GNN decreases the loss value with strictly many more negative terms, suggesting faster training. 

Empirically, we indeed observe that multiscale GNNs train faster (Figure~\ref{fig:c1}), both for (nonlinear)  ReLU and linear GNNs. We verify this by training multiscale and non-multiscale, ReLU and linear GCNs on the Cora and Citeseer datasets with   cross-entropy loss, learning rate 5e-5, and hidden dimension $32$. Results are in Appendix~\ref{sec:results}. 


\subsection{Acceleration with More Depth}

Our second finding is that deeper GNNs, with or without skip connections, train faster.  For any differentiable loss function  $q \mapsto \ell(q,Y)$, Theorem \ref{thm:4} states that the loss of the multiscale linear  GNN decreases as
\begin{align} \label{eq:thm3:1}
\frac{d}{dt} L(W_{t},B_{t}) &= \scalebox{0.95}{$\displaystyle - \underbrace{\sum_{l=0}^H \underbrace{\left\|\vect\left[V_{t} (X(S^{l})_{*\Ical})\T \right]\right\|_{F_{(l),t}}^{2}}_{\ge 0}}_{\substack{ \text{ further improvement as  depth $H$ increases} \\ }} $}
\\ \nonumber & \hspace{12pt} \scalebox{0.95}{$\displaystyle  -\underbrace{  \sum_{i=1}^H  \underbrace{\left\| \sum_{l=i}^H   J_{(i,l),t}\vect\left[V_{t}  (X(S^{l})_{*\Ical})\T \right]\right\|_2^{2}.}_{\ge 0}}_{\substack{ \text{ further improvement as  depth $H$ increases} \\ }} $}
\end{align}
In equation \eqref{eq:thm3:1}, we can see that the multiscale linear GNN achieves faster loss reduction  as depth $H$ increases. A similar argument applies to non-multiscale linear GNNs.

Empirically too, deeper GNNs train faster (Figure~\ref{fig:c2}). Again, the acceleration applies to both (nonlinear) ReLU GNNs and linear GNNs. We verify this by training multiscale and non-multiscale, ReLU and linear GCNs with 2, 4, and 6 layers on  the Cora and Citeseer datasets with learning rate 5e-5, hidden dimension $32$, and  cross-entropy loss. Results are in Appendix~\ref{sec:results}.

\subsection{Label Distribution: Signal vs. Noise}

Finally, we study how the labels affect the training speed. For the loss reduction \eqref{eq:thm4:1} and  \eqref{eq:thm4:2},  we argue that the norm of  $V _{t}(X(S^{l})_{*\Ical})\T$  tends to  be larger for labels $Y$ that are more correlated with the graph features $X(S^{l})_{*\Ical}$, e.g., labels are signals instead of ``noise''. 

Without loss of generality, we assume $Y$ is normalized, e.g., one-hot labels. Here, $V_t=\frac{\partial L(A_t,B_t)}{\partial \hat Y_t}$ is  the derivative of the loss with respect to the model output, e.g., $V_t=2(\hat Y_t - Y)$ for squared loss.  
If the rows of $Y$ are random noise vectors, then so are the rows of   $V_t$, and they are expected to  get more  orthogonal to the columns of $(X(S^{l})_{*\Ical})\T$ as $n$ increases. 
In contrast,  if the labels $Y$ are highly correlated with the graph features $(X(S^{l})_{*\Ical})\T$, i.e., the labels have signal, then the norm of $V _{t}(X(S^{l})_{*\Ical})\T$ will be larger, implying faster training. 

Our argument above focuses on the first term of the loss reduction, $\|V _{t}(X(S^{l})_{*\Ical})\T\|_{\mathrm{F}}^2$. 
We empirically demonstrate that the scale of the second term, $\left\| \sum_{l=i}^H   J_{(i,l),t}\vect\left[V_{t} (X(S^{l})_{*\Ical})\T \right]\right\|_2^2$, is dominated by that of the first term (Figure~\ref{fig:noise}). Thus, we can expect GNNs to train faster with signals than noise.

We train GNNs with the original labels of the dataset and random labels (i.e., selecting a class with uniform probability), respectively. The prediction of our theoretical analysis aligns with practice: training is much slower for random labels (Figure~\ref{fig:c3}). We verify this  for mutliscale and non-multiscale, ReLU and linear GCNs   on the Cora and Citseer datasets with learning rate 1e-4, hidden dimension $32$, and cross-entropy loss. Results are in Appendix~\ref{sec:results}.

\begin{figure}[!t]
    \centering
        \includegraphics[width=0.35\textwidth]{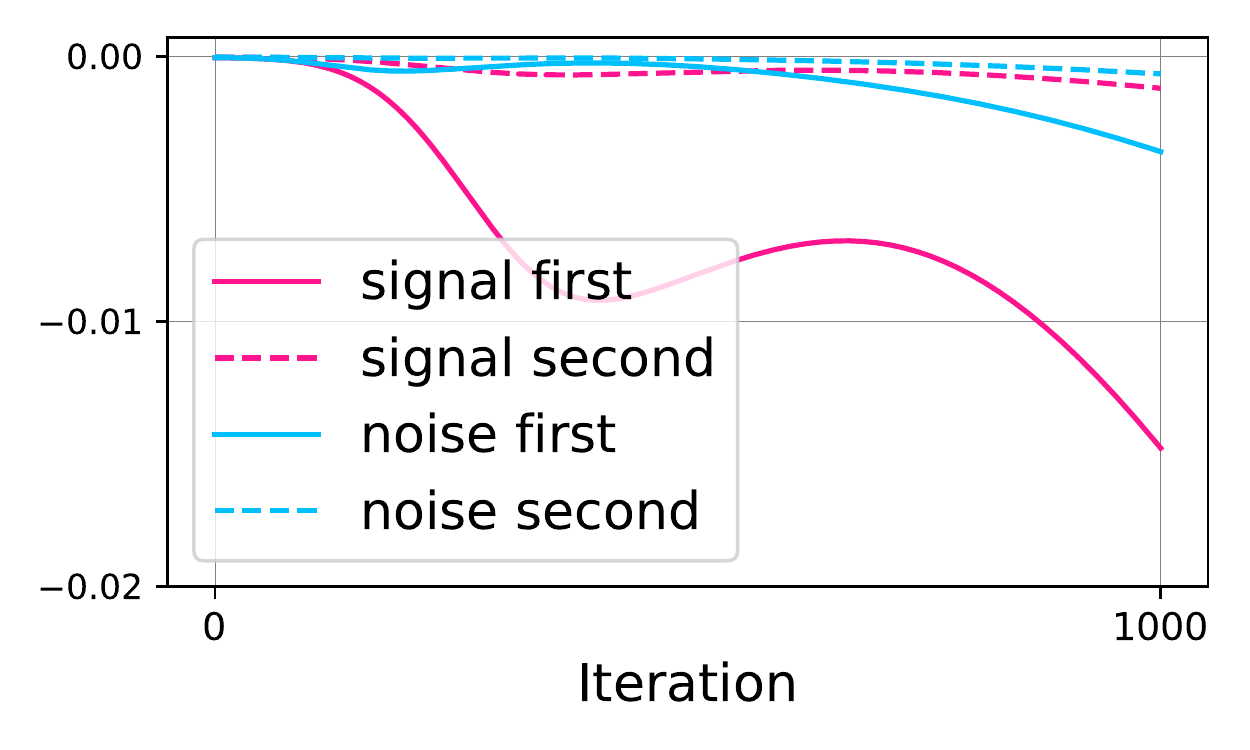}
   \caption{\textbf{The scale of the first term dominates the second term} of the loss reduction $\frac{d}{dt} L(W_{t},B_{t})$ for linear GNNs trained  with the original labels vs. random labels on Cora.  }
\label{fig:noise}
\end{figure}


\section{Related Work}
\label{sec:related}

\textbf{Theoretical analysis of linearized networks.} The theoretical study of neural networks with some linearized components has recently drawn much attention. Tremendous efforts have been made to understand linear \textit{feedforward} networks, in terms of their loss  landscape~\citep{kawaguchi2016deep, hardt2016identity, laurent2018deep} and optimization dynamics~\citep{saxe2013exact, arora2019convergence, bartlett2019gradient, du2019width, zou2020global}. Recent works prove global convergence rates for deep linear  networks under certain conditions~\citep{bartlett2019gradient, du2019width, arora2019convergence, zou2020global}. For example, \citet{arora2019convergence} assume the data to be whitened. \citet{zou2020global} fix the weights of certain layers during training. Our work is inspired by these works but differs in that our analysis applies to all learnable weights and does not require these specific assumptions, and we study the more complex GNN architecture with skip connections. GNNs consider the interaction of graph structures via the recursive message passing, but such structured, locally varying interaction is not present in feedforward networks. Furthermore, linear feedforward networks, even with skip connections, have the same expressive power as shallow linear models, a crucial condition in previous proofs \citep{bartlett2019gradient, du2019width, arora2019convergence, zou2020global}.  In contrast, the expressive power of multiscale linear GNNs can change significantly as depth increases. Accordingly, our proofs  significantly differ from previous studies. 

Another line of works studies the gradient dynamics of neural networks in the neural tangent kernel (NTK) regime~\citep{jacot2018neural, li2018learning, allen2019convergence,  arora2019fine, chizat2019lazy, du2019gradient, du2018gradient, kawaguchi2019gradient, nitanda2020optimal}. With over-parameterization, the NTK remains almost constant during training. Hence, the corresponding neural network is implicitly linearized with respect to random features of the NTK at initialization~\citep{lee2019wide, yehudai2019power, liu2020linearity}. On  the other hand, our work needs to address nonlinear dynamics and changing expressive power.

\textbf{Learning dynamics and optimization of GNNs.} Closely related to our work, \citet{du2019graph, xu2020neural} study the gradient dynamics of GNNs via the Graph NTK but focus on GNNs' generalization and extrapolation properties. We instead analyze optimization. Only~\citet{zhang2020fast} also prove global convergence for GNNs, but for the one-hidden-layer case, and they assume a specialized tensor initialization and training algorithms. In contrast, our results work for any finite depth with no assumptions on specialized training. Other works aim to accelerate and stabilize the training of GNNs through normalization techniques~\citep{cai2020graphnorm} and importance sampling~\citep{chen2018fastgcn, chen2018stochastic, huang2018adaptive, chiang2019cluster,  zou2019layer}. Our work complements these practical works with a better theoretical understanding of GNN training. 

\section{Conclusion}
\label{sec:conclusion}

This work studies the training properties of  GNNs through the lens of optimization dynamics. For  linearized GNNs with or without skip connections, despite the non-convex objective, we show that gradient descent training is guaranteed to converge to a global minimum at a linear rate. The conditions for global convergence are validated on real-world graphs. We further find out that skip connections, more depth, and/or  a good label distribution implicitly accelerate the training  of GNNs. Our results suggest deeper GNNs with skip connections may be promising in practice,  and  serve as a first foundational step for understanding the  optimization of general GNNs. 

\section*{Acknowledgements}

KX and SJ were supported by NSF CAREER award 1553284 and NSF III 1900933.  MZ was supported by ODNI, IARPA, via the BETTER Program
contract 2019-19051600005. The research of KK was partially supported by the Center of Mathematical Sciences and Applications at Harvard University. The views, opinions, and/or findings contained in this article are those of
the author and should not be interpreted as representing the official views or policies, either expressed
or implied, of the Defense Advanced Research Projects Agency, the Department of Defense, ODNI,
IARPA, or the U.S. Government. The U.S. Government is authorized to reproduce and distribute
reprints for governmental purposes notwithstanding any copyright annotation therein.

\bibliography{combined}
\bibliographystyle{icml2021}

\clearpage

\appendix


\appendix

\onecolumn


\allowdisplaybreaks

\section{Proofs}
\label{sec:proofs}

In this section, we complete the   proofs of our theoretical results. We show the proofs of Theorem~\ref{thm:6} in Appendix \ref{sec:proof_thm:6},  Proposition~\ref{prop:2} in Appendix \ref{sec:proof_prop:2}, Proposition~\ref{prop:1} in Appendix \ref{sec:complete},  Theorem~\ref{thm:1} in Appendix \ref{sec:proof_thm:1},  Proposition~\ref{prop:3} in Appendix \ref{sec:proof_prop:3}, and Theorem~\ref{thm:4} in Appendix \ref{sec:proof_thm:4}.

Before starting our proofs, we first introduce additional notation used in the proofs. We define the corner cases on the products of $B$ as: 
\begin{align}
&B_{(H)}   B_{(H-1)}\cdots B_{(l+1)}:=I_{m_{l}} \quad  \text{ if } H=l
\\ & B_{(H)} B_{(H-1)}\dots B_{(1)}:=I_{m_x}  \quad  \text{ if } H=0
\\ & B_{(l-1)}B_{(l-2)}\dots B_{(1)}:=I_{m_x}  \quad  \text{ if } l=1
\end{align}
Similarly,  for any matrices $M_{(l)}$,
we define 
$
M_{(l)}M_{(l-1)} \cdots M_{(k)} := I_{m} \text{ if } l<k,
$
and
$
M_{(l)}M_{(l-1)} \cdots M_{(k)} :=M_{(k)} =M_{(l)} \text{ if } l=k.
$
Given a scalar-valued  variable  $a \in \RR$ and a matrix $M \in \RR^{d\times d'}$,
we define
\begin{align}
\frac{\partial a}{\partial M}= \begin{bmatrix}\frac{\partial a}{\partial M_{11}} & \cdots & \frac{\partial a}{\partial M_{1d'}} \\
\vdots & \ddots & \vdots \\
\frac{\partial a}{\partial M_{d1}} & \cdots & \frac{\partial a}{\partial M_{dd'}} \\
\end{bmatrix} \in \RR^{d \times d'},
\end{align}
where $M_{ij}$ represents the $(i,j)$-th entry of the matrix $M$.
Given a vector-valued variable $a \in \RR^d$ and a column vector $b \in \RR^{d'}$,
we let
\begin{align}
\frac{\partial a}{\partial b}= \begin{bmatrix}\frac{\partial a_{1}}{\partial b_{1}} & \cdots & \frac{\partial a_{1}}{\partial  b_{d'}} \\
\vdots & \ddots & \vdots \\
\frac{\partial a_{d}}{\partial  b_{1}} & \cdots & \frac{\partial a_{d}}{\partial   b_{d'}} \\
\end{bmatrix} \in \RR^{d \times d'},
\end{align}
where $b_{i}$ represents   the $i$-th entry of the column vector  $b$. Similarly, 
given a vector-valued variable $a \in \RR^d$ and a row vector $b \in \RR^{1 \times d'}$,
we write
\begin{align}
\frac{\partial a}{\partial b}= \begin{bmatrix}\frac{\partial a_{1}}{\partial b_{11}} & \cdots & \frac{\partial a_{1}}{\partial  b_{1d'}} \\
\vdots & \ddots & \vdots \\
\frac{\partial a_{d}}{\partial  b_{11}} & \cdots & \frac{\partial a_{d}}{\partial   b_{1d'}} \\
\end{bmatrix} \in \RR^{d \times d'},
\end{align}
where $b_{1i}$ represents   the $i$-th entry of the row vector  $b$.
Finally, we  recall the  definition of the Kronecker product product of two matrices: for matrices $M \in \RR^{d_M\times d_M'}$ and $\bar M \in \RR^{d_{\bar M} \times d'_{\bar M}}$, 
\begin{align}
M \otimes \bar M= \begin{bmatrix}M_{11}\bar M  & \cdots & M_{1d_M'}\bar M \\
\vdots & \ddots\ & \vdots  \\
M_{d_M1}\bar M & \cdots & M_{d_Md_M'}\bar M \\
\end{bmatrix} \in \RR^{d_M d_{\bar M} \times d_M' d_{\bar M}'}. 
\end{align}

\subsection{Proof of Theorem~\ref{thm:6}} 
\label{sec:proof_thm:6}
We begin with  a proof overview of Theorem~\ref{thm:6}. We first relate the gradients  $\nabla_{ W_{(H)}}L$ and $\nabla_{B_{(l)}}L$ to  the gradient $\nabla_{(H)}L$, which is defined by   
$$
\nabla_{(H)}L(W,B):=\frac{\partial L(W,B)}{\partial \hat Y} (X (S^{H} )_{*\Ical})\T \in \RR^{m_y \times m_x}.
$$
Using the proven relation of  $(\nabla_{ W_{(H)}}L, \nabla_{B_{(l)}}L)$ and  $\nabla_{(H)}L$, we  first analyze the  dynamics induced in the space of $W_{(l)} B_{(l)}B_{(l-1)} \cdots B_{(1)}$ in Appendix \ref{sec:new:1}, and then    the  dynamics induced int the space of loss value $L(W,B)$ in Appendix \ref{sec:new:2}. Finally, we complete the proof by using the assumption of employing the square loss in Appendix   \ref{sec:new:3}.

Let $W_{(H)}=W$ (during the proof of Theorem~\ref{thm:6}). 
We first prove the relationship of  the gradients  $\nabla_{ W_{(H)}}L$, $\nabla_{B_{(l)}}L$ and  $\nabla_{(H)}L$  in the following lemma:

\begin{lemma} \label{lemma:1}
Let $f$ be an $H$-layer linear GNN and $\ell(q,Y)=\|q-Y\|^{2}_F$ where $q,Y \in \RR^{m_y \times \bn}$. Then, for any $(W, B)$, 
\begin{align}
\nabla_{ W_{(H)}}L(W,B) &=\nabla_{(H)}L(W,B)(B_{(H)} B_{(H-1)}\dots B_{(1)})\T  \in  \RR^{m_y \times m_l},
\end{align}
and
\begin{align}
\nabla_{B_{(l)}}L(W,B) &=(W_{(H)}B_{(H)}   B_{(H-1)}\cdots B_{(l+1)}  )\T\nabla_{(H)}L(W,B)(B_{(l-1)}B_{(l-2)}\dots B_{(1)})\T\in \RR^{m_l \times m_{l-1}}, 
\end{align}
\end{lemma}
\begin{proof}[Proof of Lemma \ref{lemma:1}]
From Definition \ref{def:1}, we have  $\hat Y =f(X, W_{}, B_{})_{*\Ical}=W_{(H)} (X_{(H)})_{*\Ical}$ where $X_{(l)} = B_{(l)} X_{(l-1)} S$. Using this definition, we can derive the formula of $\frac{\partial\vect[\hat Y]}{\partial\vect[ W_{(H)}]} \in \RR^{m_y n \times m_y m_{\bar H}}$ as: 
\begin{align} \label{eq:new:5}
\frac{\partial \vect[\hat Y]}{\partial \vect[W_{(H)}]} &=\frac{\partial}{\partial  \nonumber \vect[ W_{(H)}]}  \vect[W_{(H)} (X_{(H)})_{*\Ical}  ]
\\  & =\frac{\partial}{\partial \vect[W_{(H)}]}[( (X_{(H)})_{*\Ical})\T \otimes I_{m_{y}} ] \vect[ W_{(H)}]=[( (X_{(H)})_{*\Ical})\T \otimes I_{m_{y}} ]  \in \RR^{m_y n \times m_y m_{\bar H}}     
\end{align}

We  will now derive the formula of $\frac{\partial\vect[\hat Y]}{\partial\vect[B_{(l)} ]} \in \RR^{m_y n \times m_l m_{l-1}}$: 
\begin{align} \label{eq:new:2}
\nonumber \frac{\partial \vect[ \hat Y]}{\partial \vect[B_{(l)}]} &=\frac{\partial}{\partial \vect[B_{(l)}]} \vect[W_{(H)} (X_{(H)})_{*\Ical} ]
\\ \nonumber & =\frac{\partial}{\partial \vect[B_{(l)}]}   [I_{n}\otimes W_{(H)}  ] \vect[(X_{(H)})_{*\Ical}]    
\\ \nonumber & =  [I_{n}\otimes W_{(H)}  ] \frac{\partial\vect[(X_{(H)})_{*\Ical}]}{\partial \vect[B_{(l)}]} 
\\ \nonumber & =  [I_{n}\otimes W_{(H)}  ] \frac{\partial\vect[(X_{(H)})_{*\Ical}]}{\partial \vect[X_{(l)}]} \frac{\partial\vect[X_{(l)}]}{\partial \vect[B_{(l)}]}
\\  \nonumber& =  [I_{n}\otimes W_{(H)}  ] \frac{\partial\vect[(X_{(H)})_{*\Ical}]}{\partial \vect[X_{(l)}]} \frac{\partial\vect[B_{(l)} X_{(l-1)} S]}{\partial \vect[B_{(l)}]}
\\ \nonumber & =   [I_{n}\otimes W_{(H)}  ] \frac{\partial\vect[(X_{(H)})_{*\Ical}]}{\partial \vect[X_{(l)}]} \frac{\partial[(X_{(l-1)} S)\T \otimes I_{m_{l}} ]\vect[B_{(l)} ]}{\partial \vect[B_{(l)}]}
\\  & =   [I_{n}\otimes W_{(H)}  ] \frac{\partial\vect[(X_{(H)})_{*\Ical}]}{\partial \vect[X_{(l)}]} [(X_{(l-1)} S)\T \otimes I_{m_{l}} ]
\end{align}
Here,
we have that 
\begin{align} 
\vect[(X_{(H)})_{*\Ical}] =\vect[B_{(H)} X_{(H-1)} S_{*\Ical} ]=\vect[(S_{} \T )_{\Ical*}\otimes B_{(H)}   ] \vect[X_{(H-1)}]. 
\end{align}
and 
\begin{align} \label{eq:new:1}
\vect[X_{(H)}] =\vect[B_{(H)} X_{(H-1)} S_{*\Ical} ]=\vect[S_{} \otimes B_{(H)}   ] \vect[X_{(H-1)}]. 
\end{align}
By recursively applying \eqref{eq:new:1}, we have that \begin{align*}
\vect[(X_{(H)})_{*\Ical}] &=\vect[(S_{} \T )_{\Ical*}\otimes B_{(H)}   ]\vect[S \T \otimes B_{(H-1)}] \cdots \vect[S \T \otimes B_{(l+1)}] \vect[X_{(l)}] 
\\ & =\vect[((S^{H-l}) \T )_{\Ical*}\otimes B_{(H)}   B_{(H-1)}\cdots B_{(l+1)}]\vect[X_{(l)}],
\end{align*}
where
$$
B_{(H)}   B_{(H-1)}\cdots B_{(l+1)}:=I_{m_{l}} \quad \text{ if $H=l$.}
$$  
Therefore, 
 
\begin{align} \label{eq:new:3}
\frac{\partial\vect[(X_{(H)})_{*\Ical}]}{\partial \vect[X_{(l)}]} &=\vect[((S^{H-l}) \T )_{\Ical*}\otimes B_{(H)}   B_{(H-1)}\cdots B_{(l+1)}].
\end{align}

Combining \eqref{eq:new:2} and \eqref{eq:new:3}  yields
\begin{align} \label{eq:new:4}
\nonumber \frac{\partial \vect[\hat Y]}{\partial \vect[B_{(l)}]} & = [I_{n}\otimes W_{(H)}  ] \frac{\partial\vect[(X_{(H)})_{*\Ical}]}{\partial \vect[X_{(l)}]} [(X_{(l-1)} S)\T \otimes I_{m_{l}} ] 
\\ \nonumber &  =  [I_{n}\otimes W_{(H)}  ] \vect[((S^{H-l}) \T )_{\Ical*}\otimes B_{(H)}   B_{(H-1)}\cdots B_{(l+1)}] [(X_{(l-1)} S)\T \otimes I_{m_{l}} ] \\ &  =  [(X_{(l-1)} (S^{H-l+1})_{*\Ical})\T \otimes  W_{(H)}B_{(H)}   B_{(H-1)}\cdots B_{(l+1)}  ]\in \RR^{m_y n \times m_l m_{l-1}}.
\end{align}

Using \eqref{eq:new:5}, we will now derive the formula of $\nabla_{ W_{(H)}}L(W,B) \in \RR^{ m_y \times m_H}$:  
$$
\frac{\partial L(W,B)}{\partial\vect[ W_{(H)}]} = \frac{\partial L(W,B)}{\partial\vect[ \hat Y]}\frac{\partial\vect[ \hat Y]}{\partial\vect[ W_{(H)}]} =\frac{\partial L(W,B)}{\partial\vect[ \hat Y]}[(X_{(H)})_{*\Ical}\T \otimes I_{m_{y}} ] 
$$
Thus, with $\frac{\partial L(W,B)}{\partial \hat Y}\in \RR^{m_y \times n}$,
\begin{align*}
\nabla_{ \vect[W_{(H)}]}L(W,B) &= \left(\frac{\partial L(W,B)}{\partial\vect[ W_{(H)}]}\right)\T 
\\ &=[(X_{(H)})_{*\Ical} \otimes I_{m_{y}} ] \left(\frac{\partial L(W,B)}{\partial\vect[ \hat Y]} \right)\T   
\\ & =[(X_{(H)})_{*\Ical} \otimes I_{m_{y}} ]\vect\left[\frac{\partial L(W,B)}{\partial \hat Y} \right] 
\\ & = \vect\left[ \frac{\partial L(W,B)}{\partial \hat Y} (X_{(H)})_{*\Ical}\T \right] \in \RR^{m_ym_H}.
\end{align*}
Therefore,
\begin{align} \label{eq:new:6}
\nabla_{ W_{(H)}}L(W,B) &=\frac{\partial L(W,B)}{\partial \hat Y} (X_{(H)})_{*\Ical}\T\in \RR^{m_y \times m_H}. 
\end{align}

Using \eqref{eq:new:4}, we will now derive the formula of $\nabla_{ B_{(l)}}L(W,B) \in \RR^{ m_l \times m_{l-1}}$: 
$$
\frac{\partial L(W,B)}{\partial\vect[B_{(l)}]} = \frac{\partial L(W,B)}{\partial\vect[ \hat Y]}\frac{\partial\vect[ \hat Y]}{\partial\vect[B_{(l)}]} =\frac{\partial L(W,B)}{\partial\vect[ \hat Y]}   [(X_{(l-1)} (S^{H-l+1})_{*\Ical})\T \otimes  W_{(H)}B_{(H)}   B_{(H-1)}\cdots B_{(l+1)}  ]. 
$$
Thus, with $\frac{\partial L(W,B)}{\partial \hat Y}\in \RR^{m_y \times n}$,
\begin{align*}
\nabla_{ \vect[B_{(l)}]}L(W,B) &= \left(\frac{\partial L(W,B)}{\partial\vect[B_{(l)}]}\right)\T 
\\ &=[X_{(l-1)} (S^{H-l+1} )_{*\Ical}\otimes  (W_{(H)}B_{(H)}   B_{(H-1)}\cdots B_{(l+1)}  )\T] \left(\frac{\partial L(W,B)}{\partial\vect[ \hat Y]} \right)\T   
\\ & =   [X_{(l-1)} (S^{H-l+1} )_{*\Ical}\otimes  (W_{(H)}B_{(H)}   B_{(H-1)}\cdots B_{(l+1)}  )\T]\vect\left[\frac{\partial L(W,B)}{\partial \hat Y} \right] 
\\ & =   \vect\left[ (W_{(H)}B_{(H)}   B_{(H-1)}\cdots B_{(l+1)}  )\T\frac{\partial L(W,B)}{\partial \hat Y} (X_{(l-1)} (S^{H-l+1} )_{*\Ical})\T \right] \in \RR^{m_lm_{l-1}}.
\end{align*}
Therefore,  
\begin{align} \label{eq:new:7}
\nabla_{B_{(l)}}L(W,B) = (W_{(H)}B_{(H)}   B_{(H-1)}\cdots B_{(l+1)}  )\T\frac{\partial L(W,B)}{\partial \hat Y} (X_{(l-1)} (S^{H-l+1} )_{*\Ical})\T \in \RR^{m_l \times m_{l-1}}.
\end{align}

With \eqref{eq:new:6} and \eqref{eq:new:7}, we are now ready to prove the statement of this lemma by introducing the following notation:
$$
\nabla_{(l)}L(W,B):=\frac{\partial L(W,B)}{\partial \hat Y} (X (S^{l} )_{*\Ical})\T \in \RR^{m_y \times m_x}. 
$$
Using this notation along with  \eqref{eq:new:6}\begin{align*}
\nabla_{ W_{(H)}}L(W,B) &=\frac{\partial L(W,B)}{\partial \hat Y} (X_{(H)})_{*\Ical}\T
\\ & =\frac{\partial L(W,B)}{\partial \hat Y} (B_{(H)} X_{(H-1)} (S)_{*\Ical})\T
\\ & =\frac{\partial L(W,B)}{\partial \hat Y} (B_{(H)} B_{(H-1)}\dots B_{(1)}X(S^{H})_{*\Ical})\T
\\ & =\nabla_{(H)}L(W,B)(B_{(H)} B_{(H-1)}\dots B_{(1)})\T, 
\end{align*}
Similarly,
using   \eqref{eq:new:7}, \begin{align*}
\nabla_{B_{(l)}}L(W,B) & = (W_{(H)}B_{(H)}   B_{(H-1)}\cdots B_{(l+1)}  )\T\frac{\partial L(W,B)}{\partial \hat Y} (X_{(l-1)} (S^{H-l+1} )_{*\Ical})\T 
\\ & = (W_{(H)}B_{(H)}   B_{(H-1)}\cdots B_{(l+1)}  )\T\frac{\partial L(W,B)}{\partial \hat Y} (B_{(l-1)}B_{(l-2)}\dots B_{(1)}X(S^{l-1} S^{H-l+1} )_{*\Ical})\T  
\\ & = (W_{(H)}B_{(H)}   B_{(H-1)}\cdots B_{(l+1)}  )\T\frac{\partial L(W,B)}{\partial \hat Y} (B_{(l-1)}B_{(l-2)}\dots B_{(1)}X(S^{H})_{*\Ical} )\T
\\ & = (W_{(H)}B_{(H)}   B_{(H-1)}\cdots B_{(l+1)}  )\T\nabla_{(H)}L(W,B)(B_{(l-1)}B_{(l-2)}\dots B_{(1)})\T    
\end{align*}
where $B_{(l-1)}B_{(l-2)}\dots B_{(1)}:=I_{m_x}$ if $l=1$. 

\end{proof}

By using Lemma \ref{lemma:1}, we complete the proof of Theorem~\ref{thm:6} in the following. 

\subsubsection{Dynamics induced in the space of $W_{(l)} B_{(l)}B_{(l-1)} \cdots B_{(1)}$} \label{sec:new:1}

We now consider the dynamics induced in the space of $W_{(l)} B_{(l)}B_{(l-1)} \cdots B_{(1)}$. We first consider the following discrete version of the dynamics:
$$
W_{(H)}' = W_{(H)} -\alpha \nabla_{ W_{(H)}}L(W,B)
$$
$$
B_{(l)}' =B_{(l)} -\alpha \nabla_{ B_{(l)}}L(W,B).
$$
This dynamics induces the following dynamics:
$$
W_{(H)}' B_{(H)}'B_{(H-1)}' \cdots B_{(1)}' =( W_{(H)} -\alpha \nabla_{ W_{(H)}}L(W,B))(B_{(H)} -\alpha \nabla_{ B_{(H)}}L(W,B))  \cdots (B_{(1)} -\alpha \nabla_{ B_{(1)}}L(W,B)). 
$$
Define 
$$
Z_{(H)}:=W_{(H)} B_{(H)}B_{(H-1)} \cdots B_{(1)},
$$
and
$$
Z_{(H)}':=W_{(H)}' B_{(H)}'B_{(H-1)}' \cdots B_{(1)}'. 
$$
Then, we can rewrite
$$
Z_{(H)}' =( W_{(H)} -\alpha \nabla_{ W_{(H)}}L(W,B))(B_{(H)} -\alpha \nabla_{ B_{(H)}}L(W,B))  \cdots (B_{(1)} -\alpha \nabla_{ B_{(1)}}L(W,B)). 
$$
By expanding the multiplications, this can be written as: 
\begin{align*}
 Z_{(H)}' 
= Z_{(H)}-\alpha  \nabla_{ W_{(H)}}L(W,B)B_{(H)} \cdots B_{(1)} -  \alpha\sum_{i=1}^H  W_{(H)} B_{(H)} \cdots B_{(i+1)}\nabla_{ B_{(i)}}L(W,B)B_{(i-1)} \cdots B_{(1)} + O(\alpha^2).
\end{align*}
By vectorizing both sides,
\begin{align*}
& \vect[Z_{(H)}']- \vect[Z_{(H)} ] 
\\ & =-\alpha \vect[ \nabla_{ W_{(H)}}L(W,B)B_{(H)} \cdots B_{(1)}]
  -  \alpha\sum_{i=1}^H   \vect[W_{(H)} B_{(H)} \cdots B_{(i+1)}\nabla_{ B_{(i)}}L(W,B)B_{(i-1)} \cdots B_{(1)} ]+ O(\alpha^2).
\end{align*}
Here, using the  formula of $\nabla_{ W_{(H)}}L(W,B)$ and $\nabla_{B_{(H)}}L(W,B)$, we have that
\begin{align*}
\vect[\nabla_{ W_{(H)}}L(W,B)B_{(H)} \cdots B_{(1)} ] &=\vect[\nabla_{(H)}L(W,B)(B_{(H)} \dots B_{(1)})\T B_{(H)} \cdots B_{(1)}]
\\ & =[(B_{(H)} \dots B_{(1)})\T B_{(H)} \cdots B_{(1)} \otimes I_{m_{y}} ]\vect[\nabla_{(H)}L(W,B)], 
\end{align*}
and
\begin{align*}
&\sum_{i=1}^H  \vect[W_{(H)} B_{(H)} \cdots B_{(i+1)}\nabla_{ B_{(i)}}L(W,B)B_{(i-1)} \cdots B_{(1)}] 
\\ & =\sum_{i=1}^H  \vect\left[W_{(H)} B_{(H)} \cdots B_{(i+1)} (W_{(H)}B_{(H)}   \cdots B_{(i+1)}  )\T\nabla_{(H)}L(W,B)(B_{(i-1)}\dots B_{(1)})\T B_{(i-1)} \cdots B_{(1)} \right]
\\ & =\sum_{i=1}^H [(B_{(i-1)}\dots B_{(1)})\T B_{(i-1)} \cdots B_{(1)} \otimes W_{(H)} B_{(H)} \cdots B_{(i+1)} (W_{(H)}B_{(H)}   \cdots B_{(i+1)}  )\T ]\vect\left[\nabla_{(H)}L(W,B) \right].
\end{align*}
Summarizing above,
 \begin{align*}
& \vect[Z_{(H)}']-\vect[Z_{(H)} ] 
\\ & = - \alpha [(B_{(H)} \dots B_{(1)})\T B_{(H)} \cdots B_{(1)} \otimes I_{m_{y}} ]\vect[\nabla_{(H)}L(W,B)]
\\ & \hspace{12pt}  -  \alpha\sum_{i=1}^H  [(B_{(i-1)}\dots B_{(1)})\T B_{(i-1)} \cdots B_{(1)} \otimes W_{(H)} B_{(H)} \cdots B_{(i+1)} (W_{(H)}B_{(H)}   \cdots B_{(i+1)}  )\T ]\vect\left[\nabla_{(H)}L(W,B) \right]
  \\  & \hspace{12pt}+ O(\alpha^2)
\end{align*}
Therefore, the induced continuous dynamics of $Z_{(H)}=W_{(H)} B_{(H)}B_{(H-1)} \cdots B_{(1)}$ is 
\begin{align*}
\frac{d}{dt}\vect[Z_{(H)}] = -F_{(H)}\vect[\nabla_{(H)}L(W,B)] -\left(  \sum_{i=1}^H  J_{(i,H)}\T J_{(i,H)} \right)\vect\left[\nabla_{(H)}L(W,B) \right],
\end{align*}

where 
$$ 
F_{(H)}=[(B_{(H)} \dots B_{(1)})\T B_{(H)} \cdots B_{(1)} \otimes I_{m_{y}} ],
$$
and
$$
J_{(i,H)}=[B_{(i-1)}\dots B_{(1)}\otimes (W_{(H)} B_{(H)} \cdots B_{(i+1)} )\T  ].
$$
This is because 
\begin{align*}
J_{(i,H)}\T J_{(i,H)} &=[(B_{(i-1)}\dots B_{(1)})\T  \otimes W_{(H)} B_{(H)} \cdots B_{(i+1)} ][B_{(i-1)}\dots B_{(1)}\otimes ( W_{(H)} B_{(H)} \cdots B_{(i+1)})\T ]
\\ & =[(B_{(i-1)}\dots B_{(1)})\T  B_{(i-1)}\dots B_{(1)}\otimes W_{(H)} B_{(H)} \cdots B_{(i+1)}  ( W_{(H)} B_{(H)} \cdots B_{(i+1)})\T ]. 
\end{align*}

\subsubsection{Dynamics induced int the space of loss value $L(W,B)$} \label{sec:new:2}

We now analyze the dynamics induced int the space of loss value $L(W,B)$. Using chain rule,
\begin{align*}
\frac{d}{dt} L(W,B) &=\frac{d}{dt} L_{0}(Z_{(H)}) 
\\ & = \frac{\partial L_{0}(Z_{(H)}) }{\partial \vect[Z_{(H)}]} \frac{d\vect[Z_{(H)}]}{d t },
\end{align*}
where  
$$
L_{0}(Z_{(H)})=\ell(f_{0}(X,Z_{(H)})_{*\Ical}, Y), \ \ f_0(X,Z_{(H)})= Z_{(H)} XS^H, \  \text{ and }  Z_{(H)}=W_{(H)} B_{(H)}B_{(H-1)} \cdots B_{(1)}.
$$
Since $f_{0}(X,Z_{(H)})=f(X,W,B)=\hat Y$ and $L_{0}(Z_{(H)})=L(W,B)$,
we have  that\begin{align*}
\left(\frac{\partial L_{0}(Z_{(H)})}{\partial \vect[Z_{(H)}]}\right)\T &=\left(\frac{\partial L(W,B)}{\partial \vect[\hat Y]} \frac{\partial \vect[\hat Y]}{\partial  \vect[Z_{(H)}]} \right)\T
\\ & = \left(\frac{\partial L(W,B)}{\partial \vect[\hat Y]} \left(\frac{\partial }{\partial  \vect[Z_{(H)}]}[(X(S^H)_{*\Ical})\T \otimes I_{m_y}]\vect[Z_{(H)}]  \right) \right)\T
\\ & =    [X(S^H)_{*\Ical}\otimes I_{m_y}]\vect\left[\frac{\partial L(W,B)}{\partial \hat Y} \right]   
\\ & =    \vect\left[\frac{\partial L(W,B)}{\partial \hat Y} (X(S^H)_{*\Ical})\T \right] \\ & = \vect[\nabla_{(H)}L(W,B)]     
\end{align*}
Combining these,
\begin{align*}
& \frac{d}{dt} L(W,B)
\\  &=  \vect[\nabla_{(H)}L(W,B)]\T \frac{d\vect[Z_{(H)}]}{d t }
\\ & = -\vect[\nabla_{(H)}L(W,B)]\T F_{(H)}\vect[\nabla_{(H)}L(W,B)]-  \sum_{i=1}^H \vect[\nabla_{(H)}L(W,B)]\T  J_{(i,H)}\T J_{(i,H)}  \vect\left[\nabla_{(H)}L(W,B) \right] 
\\ & = -\vect[\nabla_{(H)}L(W,B)]\T F_{(H)}\vect[\nabla_{(H)}L(W,B)]-  \sum_{i=1}^H \| J_{(i,H)}  \vect\left[\nabla_{(H)}L(W,B) \right] \|_2^2
\end{align*}
Therefore,
\begin{align} \label{eq:4} 
\frac{d}{dt} L(W,B) &=- \vect[\nabla_{(H)}L(W,B)]\T F_{(H)}\vect[\nabla_{(H)}L(W,B)]-  \sum_{i=1}^H  \left\|   J_{(i,H)}\vect[\nabla_{(H)}L(W,B)] \right\|_2^{2}
\end{align}
Since $F_{(H)}$ is real symmetric and positive semidefinite,
\begin{align*}
\frac{d}{dt} L(W,B) &\le- \lambda_{\min}(F_{(H)}) \|\vect[\nabla_{(H)}L(W,B)]\|_2^2-  \sum_{i=1}^H  \left\|   J_{(i,H)}\vect[\nabla_{(H)}L(W,B)] \right\|_2^{2}.
\end{align*} 
With $\lambda_{W,B}=\lambda_{\min}(F_{(H)})$,
\begin{align} \label{eq:5}
\frac{d}{dt} L(W,B) &\le-\lambda_{W,B}  \|\vect[\nabla_{(H)}L(W,B)]\|_2^2-  \sum_{i=1}^H  \left\|    J_{(i,H)}\vect[\nabla_{(H)}L(W,B)] \right\|_2^{2}
\end{align}

\subsubsection{Completing the proof by using the assumption of the square loss} \label{sec:new:3}

Using the assumption that $L(W,B)=\ell(f(X,W,B)_{*\Ical}, Y)= \|f(X,W,B)_{*\Ical}-Y\|^2_F$ with $\hat Y = f(X,W,B)_{*\Ical}$,
we have$$
\frac{\partial L(W,B)}{\partial \hat Y}=\frac{\partial }{\partial \hat Y}\| \hat Y-Y\|^2_F = 2(\hat Y-Y) \in \RR^{m_y\times n},
$$
and
\begin{align*}
\vect[\nabla_{(H)}L(W,B)]=\vect\left[\frac{\partial L(W,B)}{\partial \hat Y} (X (S^{H})_{*\Ical})\T \right] &=2\vect\left[(\hat Y-Y) (X (S^{H})_{*\Ical})\T \right] 
\\ & =2 [X (S^{H})_{*\Ical}\otimes I_{m_{y}} ] \vect[\hat Y-Y]. 
\end{align*}
Therefore,
\begin{align} \label{eq:new:14}
\|\vect[\nabla_{(H)}L(W,B)]\|_2^2=4 \vect[\hat Y-Y]\T  [(X (S^{H})_{*\Ical})\T X (S^{H})_{*\Ical}\otimes I_{m_{y}} ] \vect[\hat Y-Y]   
\end{align}
Using  \eqref{eq:5} and \eqref{eq:new:14},
\begin{align*}
\frac{d}{dt} L(W,B) &\le-\lambda_{W,B}  \|\vect[\nabla_{(H)}L(W,B)]\|_2^2-  \sum_{i=1}^H  \left\|   J_{(i,H)}\vect[\nabla_{(H)}L(W,B)] \right\|_2^{2}
 \\ &\le-4\lambda_{W,B}  \vect[\hat Y-Y]\T  [(X (S^{H})_{*\Ical})\T X  (S^{H})_{*\Ical}\otimes I_{m_{y}} ] \vect[\hat Y-Y]-  \sum_{i=1}^H  \left\|  J_{(i,H)}\vect[\nabla_{(H)}L(W,B)] \right\|_2^{2}
\\ & =-4\lambda_{W,B}  \vect[\hat Y-Y]\T  \left[ \tilde G_H\T \tilde G_H\otimes I_{m_{y}} \right] \vect[\hat Y-Y]-  \sum_{i=1}^H  \left\| J_{(i,H)}\vect[\nabla_{(H)}L(W,B)] \right\|_2^2 
\end{align*}
where the last line follows from the following definition:
$$
\tilde G_H : =X (S^{H})_{*\Ical}.
$$
Decompose $\vect[\hat Y-Y]$ as $\vect[\hat Y-Y] =v+v^\perp$, where $v=\Pb_{\tilde G_H\T \otimes I_{m_{y}}  }\vect[\hat Y-Y]$, $v^\perp=(I_{m_{y}n}-\Pb_{\tilde G_H\T \otimes I_{m_{y}} })\vect[\hat Y-Y]$, and $\Pb_{\tilde G_H\T \otimes I_{m_{y}}} \in \RR^{m_yn\times m_y n}$ represents the
orthogonal projection onto the column space of $\tilde G_H\T \otimes I_{m_{y}} \in  \RR^{ m_{y}n \times m_{y}m_{x}}$. Then, 
\begin{align*}
 \vect[\hat Y-Y]\T    \left[ \tilde G_H\T \tilde G_H\otimes I_{m_{y}} \right]  \vect[\hat Y-Y] 
& =(v+v^\perp)\T\left[ \tilde G_H\T \otimes I_{m_{y}} \right]     \left[ \tilde G_H\otimes I_{m_{y}} \right]    (v + v^\perp)
 \\ & =v\T  \left[ \tilde G_H\T \otimes I_{m_{y}} \right]     \left[ \tilde G_H\otimes I_{m_{y}} \right] v
\\ & \ge \sigma^2_{\min}(\tilde G_H) \|\Pb_{\tilde G_H\T \otimes I_{m_{y}}  }\vect[\hat Y-Y]\|^2_2 
\\ & = \sigma^2_{\min}(\tilde G_H) \|\Pb_{\tilde G_H\T \otimes I_{m_{y}}  }\vect[\hat Y]-\Pb_{\tilde G_H\T \otimes I_{m_{y}}  }\vect[Y]\|^2_2
\\ & = \sigma^2_{\min}(\tilde G_H) \|\vect[\hat Y]-\Pb_{\tilde G_H\T \otimes I_{m_{y}}  }\vect[Y]\pm\vect[ Y]\|^2_2
 \\ & = \sigma^2_{\min}(\tilde G_H) \|\vect[\hat Y]-\vect[ Y]+(I_{m_yn}-\Pb_{\tilde G_H\T \otimes I_{m_{y}}  })\vect[ Y]\|^2_2
\\ & \ge \sigma^2_{\min}(\tilde G_H)(\|\vect[\hat Y-Y]\|_2 - \|(I_{m_yn}-\Pb_{\tilde G_H\T \otimes I_{m_{y}}  })\vect[ Y]\|_2 )^2 
\\ & \ge \sigma^2_{\min}(\tilde G_H)(\|\vect[\hat Y-Y]\|_2^2 - \|(I_{m_yn}-\Pb_{\tilde G_H\T \otimes I_{m_{y}}  })\vect[ Y]\|_2 ^2, 
\end{align*} 
where we used the fact that the singular values of $\left[ \tilde G_H\T \otimes I_{m_{y}} \right]$ are products of singular values of $\tilde G_H$ and $I_{m_y}$. 

By noticing that $L(W,B)=\|\vect[\hat Y-Y]\|_2^{2}$ and $L^*_{H} = \|(I_{m_yn}-\Pb_{\tilde G_H\T \otimes I_{m_{y}}  })\vect[ Y]\|_2^2$ ,
$$
\vect[\hat Y-Y]\T    \left[ \tilde G_H\T \tilde G_H\otimes I_{m_{y}} \right]  \vect[\hat Y-Y] \ge \sigma^2_{\min}(\tilde G_H)(L(W,B)-L^*_{H} ).
$$
Therefore,
\begin{align*}
\frac{d}{dt} L(W,B) &\le -4\lambda_{W,B}  \vect[\hat Y-Y]\T  \left[ \tilde G_H\T \tilde G_H\otimes I_{m_{y}} \right] \vect[\hat Y-Y]-  \sum_{i=1}^H  \left\|   J_{(i,H)}\vect[\nabla_{(H)}L(W,B)] \right\|_2^2
\\ & \le -4\lambda_{W,B}   \sigma^2_{\min}(\tilde G_H)( L(W,B)-L^*_{H})-  \sum_{i=1}^H  \left\|    J_{(i,H)}\vect[\nabla_{(H)}L(W,B)] \right\|_2^2
\end{align*}
Since $\frac{d}{dt} L^*_{H}=0$, 
$$
\frac{d}{dt} (L(W,B) -L^*_{H}) \le -4\lambda_{W,B}   \sigma^2_{\min}(\tilde G_H)(L(W,B) -L^*_{H})-  \sum_{i=1}^H  \left\|    J_{(i,H)}\vect[\nabla_{(H)}L(W,B)] \right\|_2^2 
$$
By defining  $\Lb=L(W,B) -L^*_{H}$, 
\begin{align} \label{eq:proof:2_2}
\frac{d\Lb }{dt} \le -4\lambda_{W,B}   \sigma^2_{\min}(\tilde G_H)\Lb-  \sum_{i=1}^H  \left\|   J_{(i,H)}\vect[\nabla_{(H)}L(W,B)] \right\|_2^2 
\end{align} 
Since $\frac{d}{dt} \Lb \le0$ and $\Lb\ge 0$, if $\Lb=0$ at some time $\bar t$, then $\Lb=0$ for any time $t \ge \bar t$. Therefore, if $\Lb=0$ at some time $\bar t$, then we have the desired statement of this theorem for any time $t\ge \bar t$. Thus, we can focus on the time interval $[0, \bar t]$ such that  $\Lb>0$ for  any time $t\in[0, \bar t]$ (here, it is allowed to have $\bar t = \infty$). Thus, focusing on the time interval with $\Lb>0$ , equation \eqref{eq:proof:2_2} implies that 
\begin{align*}
\frac{1}{\Lb}\frac{d\Lb}{dt}  \le -4\lambda_{W,B}   \sigma^2_{\min}(\tilde G_H)-\frac{1}{\Lb}  \sum_{i=1}^H  \left\|   J_{(i,H)}\vect[\nabla_{(H)}L(W,B)] \right\|_2^2
\end{align*}
By taking integral over time
\begin{align*}
\int_{0}^T \frac{1}{\Lb}\frac{d\Lb}{dt}  dt\le -\int_{0}^T 4 \lambda_{W,B}   \sigma^2_{\min}(\tilde G_H)dt-\int_{0}^T \frac{1}{\Lb}  \sum_{i=1}^H  \left\|  J_{(i,H)}\vect[\nabla_{(H)}L(W,B)] \right\|_2^{2}dt
\end{align*}
By using the substitution rule for integrals, $\int_{0}^T \frac{1}{\Lb}\frac{d\Lb}{dt}  dt=\int_{\Lb_0}^{\Lb_T} \frac{1}{\Lb}d\Lb=\log(\Lb_T)-\log(\Lb_0)$, where $\Lb_0=L(W_{0},B_{0}) -L^*$ and $\Lb_T=L(W_{T},B_{T}) -L^*_{H}$. Thus,
\begin{align*}
\log(\Lb_T)-\log(\Lb_0)\le - 4 \sigma^2_{\min}(\tilde G_H) \int_{0}^T  \lambda_{W,B} dt-\int_{0}^T \frac{1}{\Lb}  \sum_{i=1}^H  \left\|  J_{(i,H)}\vect[\nabla_{(H)}L(W,B)] \right\|_2^{2}dt
\end{align*} 
which implies that
\begin{align*}
\Lb_T &\le e^{\log(\Lb_0)- 4 \sigma^2_{\min}(\tilde G_H) \int_{0}^T  \lambda_{W,B} dt-\int_{0}^T \frac{1}{\Lb}  \sum_{i=1}^H  \left\|  J_{(i,H)}\vect[\nabla_{(H)}L(W,B)] \right\|_2^{2}dt}
\\ & =\Lb_0 e^{- 4 \sigma^2_{\min}(\tilde G_H) \int_{0}^T  \lambda_{W,B} dt-\int_{0}^T \frac{1}{\Lb}  \sum_{i=1}^H  \left\|   J_{(i,H)}\vect[\nabla_{(H)}L(W,B)] \right\|_2^{2}dt}
\end{align*}
By recalling the definition of $\Lb=L(W,B) -L^*_{H}$ and that  $\frac{d}{dt} \Lb \le0$, we have that if $L(W_T,B_T) -L^*_{H}> 0$, then $L(W_t,B_t) -L^*_{H}>0$ for all $t \in [0, T]$, and
\begin{align} \label{eq:new:19}
L(W_T,B_T) -L^*_{H}\le (L(W_0,B_0) -L^*_{H})
 e^{- 4 \sigma^2_{\min}(\tilde G_H) \int_{0}^T  \lambda_{W_{t},B_{t}} dt-\int_{0}^T \frac{1}{L(W_t,B_t) -L^*_{H}}  \sum_{i=1}^H  \left\|   J_{(i,H)}\vect[\nabla_{(H)}L(W_{t},B_{t})] \right\|_2^{2}dt}. 
\end{align}
Using the property of Kronecker product, 
$$
\lambda_{\min}([(B_{(H),t} \dots B_{(1),t})\T B_{(H),t} \cdots B_{(1),t} \otimes I_{m_{y}}])=\lambda_{\min}((B_{(H),t} \dots B_{(1),t})\T B_{(H),t} \cdots B_{(1),t}), 
$$ which implies that  $\lambda^{(H)}_T= \inf_{t \in[0,T]}  \lambda_{W_t,B_t}$. Thus, by noticing that $\int_{0}^T \frac{1}{L(W_t,B_t) -L^*_{H}}  \sum_{i=1}^H  \left\|    J_{(i,H)}\vect[\nabla_{(H)}L(W_{t},B_{t})] \right\|_2^{2}dt \ge 0$, equation \eqref{eq:new:19} implies that
\begin{align*}
L(W_T,B_T) -L^*_{H} &\le (L(W_0,B_0) -L^*_{H})
 e^{- 4\lambda^{(H)}_T \sigma^2_{\min}(\tilde G_H)T-\int_{0}^T \frac{1}{L(W_t,B_t) -L^*_{H}}  \sum_{i=1}^H  \left\|   J_{(i,H)}\vect[\nabla_{(H)}L(W_{t},B_{t})] \right\|_2^{2}dt} 
\\ & \le(L(W_0,B_0) -L^*_{H})
 e^{- 4\lambda^{(H)}_T \sigma^2_{\min}(\tilde G_H)T} 
 \\ & =(L(W_0,B_0) -L^*_{H})
 e^{- 4\lambda^{(H)}_T \sigma^2_{\min}(X (S^{H})_{*\Ical})T} 
\end{align*}
\qed

\subsection{Proof of Proposition~\ref{prop:2}}
\label{sec:proof_prop:2}

From Definition \ref{def:4}, we have that  $\sigma_{\min}( \bB^{(1:H)})=\sigma_{\min}(B_{(H)}B_{(H-1)} \cdots B_{(1)}) \ge \gamma$ for all $(W,B)$ such that $L(W,B)\le  L(W_0, B_0)$.
From equation \eqref{eq:5} in the proof of Theorem~\ref{thm:6}, it holds that
$
\frac{d}{dt} L(W_{t},B_{t}) \le 0
$     
for all $t$. Thus, we have that $L(W_{t},B_{t})\le  L(W_0, B_0)$ and hence  $\sigma_{\min}( \bB^{(1:H)}_t) \ge \gamma$ for all $t$.  Under this problem setting ($m_H\ge m_x$),  this implies that $\lambda_{\min}(( \bB^{(1:H)}_t)\T  \bB^{(1:H)}_t) \ge \gamma^{2}$ for all $t$ and thus $\lambda_T^{(H)}\ge\gamma^{2}$. 


\subsection{Proof of Proposition~\ref{prop:1}}
\label{sec:complete}
We first give the complete version of Proposition~\ref{prop:1}. Proposition~\ref{prop:complete} is the formal version of Proposition~\ref{prop:1} and shows that our singular margin generalizes deficiency margin proposed in \citet{arora2019convergence}. Using the deficiency margin assumption, \citet{arora2019convergence} analyzed the following optimization problem: \begin{align}\label{eq:new:20}
\mini_{\tilde W} \tilde L(\tilde W_{(1)},\dots, \tilde W_{(H+1)}): & =\frac{1}{2}\|\tilde W_{(H+1)}\tilde W_{(H)}\cdots \tilde W_{(1)}- \tilde \Phi\|^{2}_F
\\ &=\frac{1}{2}\|\tilde W_{(1)}\T \tilde W_{(2)}\T \cdots\tilde W_{(H+1)}\T\ -\tilde\Phi\T\|^{2}_F,
\end{align}
where $\tilde\Phi\in \RR^{\tilde m_y\times \tilde m_x}$ is a target matrix and the last equality follows from $\|M\|_F=\|M\T\|_F$ for any matrix $M$ by the definition of  the Frobenius  norm. Therefore, this optimization problem \eqref{eq:new:20} from the previous work is equivalent to the following optimization problem in our notation:
\begin{align}
\mini_{W,B} L(W,B):=\frac{1}{2}\|W B_{(H)}B_{(H-1)}\cdots B_{(1)}-\Phi\|^{2}_F,
\end{align}
where $W B_{(H)}B_{(H-1)}\cdots B_{(1)}=\tilde W_{(H+1)}\tilde W_{(H)}\cdots \tilde W_{(1)}$ (i.e., $W=\tilde W_{(H+1)}$ with $B_{(l)}=\tilde W_{(l)}$) and $\Phi= \tilde \Phi$ if $\tilde m_y \ge  \tilde m_x$, and  $W B_{(H)}B_{(H-1)}\cdots B_{(1)}=\tilde  W_{(1)}\T \tilde W_{(2)}\T  \cdots\tilde W_{(H+1)}\T$ (i.e., $W=\tilde  W_{(1)}\T$ with $B_{(l)}=\tilde W_{(H+2-l)}\T$) and $\Phi= \tilde \Phi\T$ if $\tilde m_y < \tilde m_x$. That is, we have $\Phi \in \RR^{m_y \times m_x}$ where $m_y = \tilde m_y$ with $m_{x} = \tilde m_x$ if  $\tilde  m_y \ge \tilde m_x$, and $m_y = \tilde m_x$ with $m_{x} = \tilde m_y$ if  $\tilde  m_y < \tilde m_{x}$. Therefore,  our general problem framework with graph structures can be reduced and applicable to  the previous optimization problem without graph structures by setting  $\frac{1}{n}XX\T =I$, $S=I$, $\Ical = [n]$,  $f(X, W, B)=WB_{(H)}B_{(H-1)} \cdots B_{(1)}$, and $\ell(q,\Phi)=\frac{1}{2}\|q-\Phi\|^{2}_F$ where $\Phi \in \RR^{m_y \times m_x}$ is a target matrix with $m_y \ge m_x$ without loss of generality. An initialization $(W_0,B_0)$ is said to have deficiency margin $c>0$ if the end-to-end matrix $W_0 \bB^{(1:H)}_0$ of the initialization $(W_0,B_0)$ has deficiency margin $c>0$ with respect to the target $\Phi$ \citep[Definition 2]{{arora2019convergence}}: i.e., \citet{arora2019convergence} assumed that the initialization $(W_0,B_0)$ has deficiency margin $c>0$ (as it is also invariant to the transpose of $\tilde W_{(H+1)}\tilde W_{(H)}\cdots \tilde W_{(1)}- \tilde \Phi$).

\begin{proposition} \label{prop:complete}
Consider  the optimization problem in \citep{{arora2019convergence}} by setting $\frac{1}{n}XX\T =I$, $S=I$, $\Ical = [n]$,  $f(X, W, B)=WB_{(H)}B_{(H-1)} \cdots B_{(1)}$, and $\ell(q,\Phi)=\frac{1}{2}\|q-\Phi\|^{2}_F$ where $\Phi \in \RR^{m_y \times m_x}$ is a target matrix with $m_y \ge m_x$ without loss of generality (since the transpose of these two dimensions leads to the equivalent optimization problem under this setting: see above). Then, if an initialization $(W_0,B_0)$ has 
deficiency margin $c>0$,  it has  singular margin $\gamma>0$. \end{proposition}

\begin{proof}[Proof of Proposition~\ref{prop:complete}]
By the definition of the deficiency margin \citep[Definition 2]{{arora2019convergence}} and its consequence \citep[Claim 1]{{arora2019convergence}}, if an initialization $(W_0,B_0)$ has 
deficiency margin $c>0$,  then any pair $(W,B)$ for which $L(W,B)\le L(W_0,B_0) $ satisfies $\sigma_{\min}(WB_{(H)}B_{(H-1)} \cdots B_{(1)}) \ge c>0$. Since the number of nonzero singular values is equal to the matrix rank, this implies that $\rank(WB_{(H)}B_{(H-1)} \cdots B_{(1)})\ge \min(m_y,m_x)$ for any pair $(W,B)$ for which $L(W,B)\le L(W_0,B_0)$.
Since $\rank(MM') \le \min (\rank(M), \rank(M'))$, this implies that 
\begin{align}\label{eq:new:21}
m_H \ge \min(m_y,m_x)=m_{x},
\end{align}  
(as well as $m_l \ge \min(m_y,m_x)$ for all $l$), and that for any pair $(W,B)$ for which $L(W,B)\le L(W_0,B_0)$, 
\begin{align}
m_{x}=\min(m_y,m_x)\le \rank(WB_{(H)}B_{(H-1)} \cdots B_{(1)}) & \le\min (\rank(W), \rank(B_{(H)}B_{(H-1)} \cdots B_{(1)}))
\\ & \le\rank(B_{(H)}B_{(H-1)} \cdots B_{(1)}) \le m_x.  
\end{align}  
This shows that $\rank(B_{(H)}B_{(H-1)} \cdots B_{(1)})=m_x$ for any pair $(W,B)$ for which $L(W,B)\le L(W_0,B_0)$.
Since $m_H \ge m_{x}$ from \eqref{eq:new:21} and the number of nonzero singular values is equal to the matrix rank, this implies that $\sigma_{\min}(B_{(H)}B_{(H-1)} \cdots B_{(1)})\ge \gamma $ for some $\gamma >0$ for any pair $(W,B)$ for which $L(W,B)\le L(W_0,B_0)$.
Thus, if an initialization $(W_0,B_0)$ has 
deficiency margin $c>0$, then it has  singular margin $\gamma>0$.

\end{proof}


\subsection{Proof of Theorem~\ref{thm:1}}
\label{sec:proof_thm:1}

This section completes the proof of Theorem~\ref{thm:1}. We compute the derivatives of the  output of multiscale linear GNN  with respect to the parameters $ W_{(l)}$ and $B_{(l)}$ in Appendix \ref{sec:new:4}. Then using these derivatives, we compute the gradient of the loss with respect to $ W_{(l)}$   in Appendix \ref{sec:new:5} and  $B_{(l)}$ in Appendix \ref{sec:new:6}. We then rearrange the formula of the gradients such that they are related to the formula of $\nabla_{(l)}L(W,B)$ in  Appendices \ref{sec:new:7}. Using the proven relation, we  first analyze the  dynamics induced in the space of $W_{(l)} B_{(l)}B_{(l-1)} \cdots B_{(1)}$ in Appendix \ref{sec:new:8}, and then    the  dynamics induced int the space of loss value $L(W,B)$ in Appendix \ref{sec:new:9}. Finally, we complete the proof by using the assumption of using the square loss in Appendices \ref{sec:new:10}--\ref{sec:new:11}. In the following, we first prove the statement for the case of $\Ical=[n]$ for the  simplicity of notation and then prove the statement for the general case afterwards.

\subsubsection{Derivation of  formula  for  $\frac{\partial\vect[ \hat Y]}{\partial\vect[ W_{(l)}]} \in \RR^{m_y n \times m_y m_l}$ and $\frac{\partial\vect[ \hat Y]}{\partial\vect[B_{(l)}]} \in \RR^{m_y n \times m_l m_{l-1}}$} \label{sec:new:4}
We can easily compute $\frac{\partial\vect[ \hat Y]}{\partial\vect[ W_{(l)}]}$ by using the property of the Kronecker product as follows: 
\begin{align} \label{eq:new:8}
\frac{\partial \vect[\hat Y]}{\partial \vect[W_{(l)}]} =\frac{\partial}{\partial \nonumber \vect[W_{(l)}]} \sum_{k=0}^H \vect[W_{(k)} X_{(k)} ]&=\frac{\partial}{\partial \vect[W_{(l)}]} \sum_{k=0}^H  [X_{(k)}\T \otimes I_{m_{y}} ] \vect[W_{(k)}]
\\ & =[X_{(l)}\T \otimes I_{m_{y}} ]\in \RR^{m_y n \times m_y m_l}    
\end{align}
We now compute $\frac{\partial\vect[ \hat Y]}{\partial\vect[B_{(l)}]}$ by using the chain rule and  the property of the Kronecker product as follows:    

\begin{align*}
\frac{\partial \vect[\hat Y]}{\partial \vect[B_{(l)}]} &=\frac{\partial}{\partial \vect[B_{(l)}]} \sum_{k=0}^H \vect[W_{(k)} X_{(k)} ]
\\ & =\frac{\partial}{\partial \vect[B_{(l)}]} \sum_{k=0}^H  [I_{n}\otimes W_{(k)}  ] \vect[X_{(k)}]    
\\ & = \sum_{k=0}^H  [I_{n}\otimes W_{(k)}  ] \frac{\partial\vect[X_{(k)}]}{\partial \vect[B_{(l)}]} 
\\ & = \sum_{k=l}^H  [I_{n}\otimes W_{(k)}  ] \frac{\partial\vect[X_{(k)}]}{\partial \vect[X_{(l)}]} \frac{\partial\vect[X_{(l)}]}{\partial \vect[B_{(l)}]}
\\ & = \sum_{k=l}^H  [I_{n}\otimes W_{(k)}  ] \frac{\partial\vect[X_{(k)}]}{\partial \vect[X_{(l)}]} \frac{\partial\vect[B_{(l)} X_{(l-1)} S]}{\partial \vect[B_{(l)}]}
\\ & = \sum_{k=l}^H  [I_{n}\otimes W_{(k)}  ] \frac{\partial\vect[X_{(k)}]}{\partial \vect[X_{(l)}]} \frac{\partial[(X_{(l-1)} S)\T \otimes I_{m_{l}} ]\vect[B_{(l)} ]}{\partial \vect[B_{(l)}]}
\\ & = \sum_{k=l}^H  [I_{n}\otimes W_{(k)}  ] \frac{\partial\vect[X_{(k)}]}{\partial \vect[X_{(l)}]} [(X_{(l-1)} S)\T \otimes I_{m_{l}} ]
\end{align*}
Here, for any $k \ge 1$, 
$$
\vect[X_{(k)}] =\vect[B_{(k)} X_{(k-1)} S ]=\vect[S \T \otimes B_{(k)}   ] \vect[X_{(k-1)}]. 
$$
By recursively applying this, we have that  for any $k \ge l$,
\begin{align*}
\vect[X_{(k)}] &=\vect[S \T \otimes B_{(k)}   ]\vect[S \T \otimes B_{(k-1)}] \cdots \vect[S \T \otimes B_{(l+1)}] \vect[X_{(l)}] 
\\ & =\vect[(S^{k-l}) \T \otimes B_{(k)}   B_{(k-1)}\cdots B_{(l+1)}]\vect[X_{(l)}],
\end{align*}
where $S^{0}:=I_n$ and 
$$
B_{(k)}   B_{(k-1)}\cdots B_{(l+1)}:=I_{m_{l}} \quad \text{ if $k=l$.}
$$  
Therefore, 
 
\begin{align*}
\frac{\partial\vect[X_{(k)}]}{\partial \vect[X_{(l)}]} &=\vect[(S^{k-l}) \T \otimes B_{(k)}   B_{(k-1)}\cdots B_{(l+1)}].
\\ &  
\end{align*}

Combining the above equations yields
\begin{align} \label{eq:new:9}
\nonumber \frac{\partial \vect[\hat Y]}{\partial \vect[B_{(l)}]} & = \sum_{k=l}^H  [I_{n}\otimes W_{(k)}  ] \frac{\partial\vect[X_{(k)}]}{\partial \vect[X_{(l)}]} [(X_{(l-1)} S)\T \otimes I_{m_{l}} ] 
\\ \nonumber &  = \sum_{k=l}^H  [I_{n}\otimes W_{(k)}  ] \vect[(S^{k-l}) \T \otimes B_{(k)}   B_{(k-1)}\cdots B_{(l+1)}] [(X_{(l-1)} S)\T \otimes I_{m_{l}} ] \\ &  = \sum_{k=l}^H    [(X_{(l-1)} S^{k-l+1})\T \otimes  W_{(k)}B_{(k)}   B_{(k-1)}\cdots B_{(l+1)}  ]\in \RR^{m_y n \times m_l m_{l-1}}.
\end{align}

\subsubsection{Derivation of a formula  of $\nabla_{ W_{(l)}}L(W,B) \in \RR^{ m_y \times m_l}$} \label{sec:new:5}
Using the chain rule and \eqref{eq:new:8}, we have that
$$
\frac{\partial L(W,B)}{\partial\vect[ W_{(l)}]} = \frac{\partial L(W,B)}{\partial\vect[ \hat Y]}\frac{\partial\vect[ \hat Y]}{\partial\vect[ W_{(l)}]} =\frac{\partial L(W,B)}{\partial\vect[ \hat Y]}[X_{(l)}\T \otimes I_{m_{y}} ]. 
$$
Thus, with $\frac{\partial L(W,B)}{\partial \hat Y}\in \RR^{m_y \times n}$,
by using \begin{align*}
\nabla_{ \vect[W_{(l)}]}L(W,B) &= \left(\frac{\partial L(W,B)}{\partial\vect[ W_{(l)}]}\right)\T 
\\ &=[X_{(l)} \otimes I_{m_{y}} ] \left(\frac{\partial L(W,B)}{\partial\vect[ \hat Y]} \right)\T   
\\ & =[X_{(l)} \otimes I_{m_{y}} ]\vect\left[\frac{\partial L(W,B)}{\partial \hat Y} \right] 
\\ & = \vect\left[ \frac{\partial L(W,B)}{\partial \hat Y} X_{(l)}\T \right] \in \RR^{m_ym_l}.
\end{align*}
Therefore,
\begin{align} \label{eq:new:10}
\nabla_{ W_{(l)}}L(W,B) &=\frac{\partial L(W,B)}{\partial \hat Y} X_{(l)}\T\in \RR^{m_y \times m_l}. 
\end{align}

\subsubsection{Derivation of a formula  of $\nabla_{ B_{(l)}}L(W,B) \in \RR^{ m_l \times m_{l-1}}$} \label{sec:new:6}
Using the chain rule and \eqref{eq:new:9}, we have that

$$
\frac{\partial L(W,B)}{\partial\vect[B_{(l)}]} = \frac{\partial L(W,B)}{\partial\vect[ \hat Y]}\frac{\partial\vect[ \hat Y]}{\partial\vect[B_{(l)}]} =\frac{\partial L(W,B)}{\partial\vect[ \hat Y]}\sum_{k=l}^H    [(X_{(l-1)} S^{k-l+1})\T \otimes  W_{(k)}B_{(k)}   B_{(k-1)}\cdots B_{(l+1)}  ]. 
$$
Thus, with $\frac{\partial L(W,B)}{\partial \hat Y}\in \RR^{m_y \times n}$,

\begin{align*}
\nabla_{ \vect[B_{(l)}]}L(W,B) &= \left(\frac{\partial L(W,B)}{\partial\vect[B_{(l)}]}\right)\T 
\\ &=\sum_{k=l}^H    [X_{(l-1)} S^{k-l+1} \otimes  (W_{(k)}B_{(k)}   B_{(k-1)}\cdots B_{(l+1)}  )\T] \left(\frac{\partial L(W,B)}{\partial\vect[ \hat Y]} \right)\T   
\\ & =\sum_{k=l}^H    [X_{(l-1)} S^{k-l+1} \otimes  (W_{(k)}B_{(k)}   B_{(k-1)}\cdots B_{(l+1)}  )\T]\vect\left[\frac{\partial L(W,B)}{\partial \hat Y} \right] 
\\ & =\sum_{k=l}^H     \vect\left[ (W_{(k)}B_{(k)}   B_{(k-1)}\cdots B_{(l+1)}  )\T\frac{\partial L(W,B)}{\partial \hat Y} (X_{(l-1)} S^{k-l+1} )\T \right] \in \RR^{m_lm_{l-1}}.
\end{align*}
Therefore,  
\begin{align} \label{eq:new:11}
\nabla_{B_{(l)}}L(W,B) =\sum_{k=l}^H (W_{(k)}B_{(k)}   B_{(k-1)}\cdots B_{(l+1)}  )\T\frac{\partial L(W,B)}{\partial \hat Y} (X_{(l-1)} S^{k-l+1} )\T \in \RR^{m_l \times m_{l-1}}.
\end{align}

\subsubsection{Relating gradients to $\nabla_{(l)}L$} \label{sec:new:7}
We now relate the gradients of the loss to $\nabla_{(l)}L$, which is defined by  
$$
\nabla_{(l)}L(W,B):=\frac{\partial L(W,B)}{\partial \hat Y} (X S^{l} )\T \in \RR^{m_y \times m_x}. 
$$
By using this  definition and \eqref{eq:new:10}, we have that
\begin{align*}
\nabla_{ W_{(l)}}L(W,B) &=\frac{\partial L(W,B)}{\partial \hat Y} X_{(l)}\T
\\ & =\frac{\partial L(W,B)}{\partial \hat Y} (B_{(l)} X_{(l-1)} S)\T
\\ & =\frac{\partial L(W,B)}{\partial \hat Y} (B_{(l)} B_{(l-1)}\dots B_{(1)}XS^{l})\T
\\ & =\nabla_{(l)}L(W,B)(B_{(l)} B_{(l-1)}\dots B_{(1)})\T, 
\end{align*}
where $B_{(l)} B_{(l-1)}\dots B_{(1)}:=I_{m_x}$ if $l=0$.
Similarly,  by using the  definition and \eqref{eq:new:11},
\begin{align*}
\nabla_{B_{(l)}}L(W,B) & =\sum_{k=l}^H (W_{(k)}B_{(k)}   B_{(k-1)}\cdots B_{(l+1)}  )\T\frac{\partial L(W,B)}{\partial \hat Y} (X_{(l-1)} S^{k-l+1} )\T 
\\ & =\sum_{k=l}^H (W_{(k)}B_{(k)}   B_{(k-1)}\cdots B_{(l+1)}  )\T\frac{\partial L(W,B)}{\partial \hat Y} (B_{(l-1)}B_{(l-2)}\dots B_{(1)}XS^{l-1} S^{k-l+1} )\T  
\\ & =\sum_{k=l}^H (W_{(k)}B_{(k)}   B_{(k-1)}\cdots B_{(l+1)}  )\T\frac{\partial L(W,B)}{\partial \hat Y} (B_{(l-1)}B_{(l-2)}\dots B_{(1)}XS^{k} )\T
\\ & =\sum_{k=l}^H (W_{(k)}B_{(k)}   B_{(k-1)}\cdots B_{(l+1)}  )\T\nabla_{(k)}L(W,B)(B_{(l-1)}B_{(l-2)}\dots B_{(1)})\T    
\end{align*}
where $B_{(l-1)}B_{(l-2)}\dots B_{(1)}:=I_{m_x}$ if $l=1$. In summary thus far,
we have that
\begin{align}\label{eq:new:12}
\nabla_{ W_{(l)}}L(W,B) &=\nabla_{(l)}L(W,B)(B_{(l)} B_{(l-1)}\dots B_{(1)})\T  \in  \RR^{m_y \times m_l}, 
\end{align}
and
\begin{align} \label{eq:new:13}
\nabla_{B_{(l)}}L(W,B) &=\sum_{k=l}^H (W_{(k)}B_{(k)}   B_{(k-1)}\cdots B_{(l+1)}  )\T\nabla_{(k)}L(W,B)(B_{(l-1)}B_{(l-2)}\dots B_{(1)})\T\in \RR^{m_l \times m_{l-1}}, 
\end{align}
where  $\nabla_{(l)}L(W,B):=\frac{\partial L(W,B)}{\partial \hat Y} (X S^{l} )\T \in \RR^{m_y \times m_x}$, $B_{(k)}   B_{(k-1)}\cdots B_{(l+1)}:=I_{m_{l}}$ if $k=l$,  $B_{(l)} B_{(l-1)}\dots B_{(1)}:=I_{m_x}$ if $l=0$, and  $B_{(l-1)}B_{(l-2)}\dots B_{(1)}:=I_{m_x}$ if $l=1$. 

\subsubsection{Dynamics induced in the space of $W_{(l)} B_{(l)}B_{(l-1)} \cdots B_{(1)}$} \label{sec:new:8}
We now consider the Dynamics induced in the space of $W_{(l)} B_{(l)}B_{(l-1)} \cdots B_{(1)}$. We first consider the following discrete version of the dynamics:
$$
W_{(l)}' = W_{(l)} -\alpha \nabla_{ W_{(l)}}L(W,B)
$$
$$
B_{(l)}' =B_{(l)} -\alpha \nabla_{ B_{(l)}}L(W,B).
$$
This dynamics induces the following dynamics:
$$
W_{(l)}' B_{(l)}'B_{(l-1)}' \cdots B_{(1)}' =( W_{(l)} -\alpha \nabla_{ W_{(l)}}L(W,B))(B_{(l)} -\alpha \nabla_{ B_{(l)}}L(W,B))  \cdots (B_{(1)} -\alpha \nabla_{ B_{(1)}}L(W,B)). 
$$
Define 
$$
Z_{(l)}:=W_{(l)} B_{(l)}B_{(l-1)} \cdots B_{(1)}
$$
and
$$
Z_{(l)}':=W_{(l)}' B_{(l)}'B_{(l-1)}' \cdots B_{(1)}'.
$$
Then, we can rewrite
$$
Z_{(l)}' =( W_{(l)} -\alpha \nabla_{ W_{(l)}}L(W,B))(B_{(l)} -\alpha \nabla_{ B_{(l)}}L(W,B))  \cdots (B_{(1)} -\alpha \nabla_{ B_{(1)}}L(W,B)). 
$$
By expanding the multiplications, this can be written as: 
\begin{align*}
 Z_{(l)}' 
= Z_{(l)}-\alpha  \nabla_{ W_{(l)}}L(W,B)B_{(l)} \cdots B_{(1)} -  \alpha\sum_{i=1}^l  W_{(l)} B_{(l)} \cdots B_{(i+1)}\nabla_{ B_{(i)}}L(W,B)B_{(i-1)} \cdots B_{(1)} + O(\alpha^2)
\end{align*}
By vectorizing both sides,
\begin{align*}
& \vect[Z_{(l)}']-\vect[Z_{(l)} ] 
\\ & = -\alpha \vect[ \nabla_{ W_{(l)}}L(W,B)B_{(l)} \cdots B_{(1)}]
  -  \alpha\sum_{i=1}^l   \vect[W_{(l)} B_{(l)} \cdots B_{(i+1)}\nabla_{ B_{(i)}}L(W,B)B_{(i-1)} \cdots B_{(1)} ]+ O(\alpha^2)
\end{align*}
Here, using the  formula of $\nabla_{ W_{(l)}}L(W,B)$ and $\nabla_{B_{(l)}}L(W,B)$, we have that
\begin{align*}
\vect[\nabla_{ W_{(l)}}L(W,B)B_{(l)} \cdots B_{(1)} ] &=\vect[\nabla_{(l)}L(W,B)(B_{(l)} \dots B_{(1)})\T B_{(l)} \cdots B_{(1)}]
\\ & =[(B_{(l)} \dots B_{(1)})\T B_{(l)} \cdots B_{(1)} \otimes I_{m_{y}} ]\vect[\nabla_{(l)}L(W,B)], 
\end{align*}
and
\begin{align*}
&\sum_{i=1}^l  \vect[W_{(l)} B_{(l)} \cdots B_{(i+1)}\nabla_{ B_{(i)}}L(W,B)B_{(i-1)} \cdots B_{(1)}] 
\\ & =\sum_{i=1}^l  \vect\left[W_{(l)} B_{(l)} \cdots B_{(i+1)}\sum_{k=i}^H (W_{(k)}B_{(k)}   \cdots B_{(i+1)}  )\T\nabla_{(k)}L(W,B)(B_{(i-1)}\dots B_{(1)})\T B_{(i-1)} \cdots B_{(1)} \right]
\\ & =\sum_{i=1}^l  \sum_{k=i}^H\vect\left[W_{(l)} B_{(l)} \cdots B_{(i+1)} (W_{(k)}B_{(k)}   \cdots B_{(i+1)}  )\T\nabla_{(k)}L(W,B)(B_{(i-1)}\dots B_{(1)})\T B_{(i-1)} \cdots B_{(1)} \right]
\\ & =\sum_{i=1}^l  \sum_{k=i}^H[(B_{(i-1)}\dots B_{(1)})\T B_{(i-1)} \cdots B_{(1)} \otimes W_{(l)} B_{(l)} \cdots B_{(i+1)} (W_{(k)}B_{(k)}   \cdots B_{(i+1)}  )\T ]\vect\left[\nabla_{(k)}L(W,B) \right].
\end{align*}
Summarizing above,
 \begin{align*}
& \vect[Z_{(l)}']-\vect[Z_{(l)} ] 
\\ & = - \alpha [(B_{(l)} \dots B_{(1)})\T B_{(l)} \cdots B_{(1)} \otimes I_{m_{y}} ]\vect[\nabla_{(l)}L(W,B)]
\\ & \hspace{12pt}  -  \alpha\sum_{i=1}^l  \sum_{k=i}^H[(B_{(i-1)}\dots B_{(1)})\T B_{(i-1)} \cdots B_{(1)} \otimes W_{(l)} B_{(l)} \cdots B_{(i+1)} (W_{(k)}B_{(k)}   \cdots B_{(i+1)}  )\T ]\vect\left[\nabla_{(k)}L(W,B) \right]
  \\  & \hspace{12pt}+ O(\alpha^2)
\end{align*}
Therefore, the induced continuous dynamics of $Z_{(l)}=W_{(l)} B_{(l)}B_{(l-1)} \cdots B_{(1)}$ is 
\begin{align*}
\frac{d}{dt}\vect[Z_{(l)}] = -F_{(l)}\vect[\nabla_{(l)}L(W,B)] -  \sum_{i=1}^l  \sum_{k=i}^H J_{(i,l)}\T J_{(i,k)}\vect\left[\nabla_{(k)}L(W,B) \right]
\end{align*}

where 
$$ 
F_{(l)}=[(B_{(l)} \dots B_{(1)})\T B_{(l)} \cdots B_{(1)} \otimes I_{m_{y}} ],
$$
and
$$
J_{(i,l)}=[B_{(i-1)}\dots B_{(1)}\otimes (W_{(l)} B_{(l)} \cdots B_{(i+1)} )\T  ].
$$
This is because 
\begin{align*}
J_{(i,k)}\T J_{(i,k)} &=[(B_{(i-1)}\dots B_{(1)})\T  \otimes W_{(l)} B_{(l)} \cdots B_{(i+1)} ][B_{(i-1)}\dots B_{(1)}\otimes ( W_{(k)} B_{(k)} \cdots B_{(i+1)})\T ]
\\ & =[(B_{(i-1)}\dots B_{(1)})\T  B_{(i-1)}\dots B_{(1)}\otimes W_{(l)} B_{(l)} \cdots B_{(i+1)}  ( W_{(k)} B_{(k)} \cdots B_{(i+1)})\T ]. 
\end{align*}

\subsubsection{Dynamics induced int the space of loss value $L(W,B)$} \label{sec:new:9}

We now analyze the dynamics induced int the space of loss value $L(W,B)$. Define
$$
L(W,B) := \ell(f(X,W,B), Y),
$$
where $\ell$ is chosen later. 
Using chain rule,
\begin{align*}
\frac{d}{dt} L(W,B) &=\frac{d}{dt} L_{0}(Z_{(H)},\dots,Z_{(0)}) 
\\ & = \sum_{l=0}^H \frac{\partial L_{0}(Z_{(l)},\dots,Z_{(0)}) }{\partial \vect[Z_{(l)}]} \frac{d\vect[Z_{(l)}]}{d t },
\end{align*}
where  
$$
L_{0}(Z_{(H)},\dots,Z_{(0)})=\ell(f_{0}(X,Z), Y), \ \ f_0(X,Z)=\sum_{l=0}^H Z_{(l)} XS^l, \  \text{ and }  Z_{(l)}=W_{(l)} B_{(l)}B_{(l-1)} \cdots B_{(1)}.
$$
Since $f_{0}(X,Z)=f(X,W,B)=\hat Y$ and $L_{0}(Z_{(H)},\dots,Z_{(0)})=L(W,B)$,
\begin{align*}
\left(\frac{\partial L_{0}(Z_{(l)},\dots,Z_{(0)})}{\partial \vect[Z_{(l)}]}\right)\T &=\left(\frac{\partial L(W,B)}{\partial \vect[\hat Y]} \frac{\partial \vect[\hat Y]}{\partial  \vect[Z_{(l)}]} \right)\T
\\ & = \left(\frac{\partial L(W,B)}{\partial \vect[\hat Y]} \left(\frac{\partial }{\partial  \vect[Z_{(l)}]} \sum_{k=0}^H [(XS^k)\T \otimes I_{m_y}]\vect[Z_{(k)}]  \right) \right)\T
\\ & =    [XS^l\otimes I_{m_y}]\vect\left[\frac{\partial L(W,B)}{\partial \hat Y} \right]   
\\ & =    \vect\left[\frac{\partial L(W,B)}{\partial \hat Y} (XS^l)\T \right] \\ & = \vect[\nabla_{(l)}L(W,B)]     
\end{align*}
Therefore,
\begin{align*}
& \frac{d}{dt} L(W,B)
\\  &= \sum_{l=0}^H \vect[\nabla_{(l)}L(W,B)]\T \frac{d\vect[Z_{(l)}]}{d t }
\\ & = - \sum_{l=0}^H\vect[\nabla_{(l)}L(W,B)]\T F_{(l)}\vect[\nabla_{(l)}L(W,B)] -\sum_{l=1}^H  \sum_{i=1}^l  \sum_{k=i}^H \vect[\nabla_{(l)}L(W,B)]\T J_{(i,l)}\T J_{(i,k)}\vect\left[\nabla_{(k)}L(W,B) \right] 
\end{align*}
To simplify the second term, define $M_{(l,i)}=\sum_{k=i}^H\vect[\nabla_{(l)}L(W,B)]\T J_{(i,l)}\T J_{(i,k)}\vect\left[\nabla_{(k)}L(W,B) \right]$ and note
that we can expand the double sums and regroup terms as follows: \begin{align*}
\sum_{l=1}^H  \sum_{i=1}^l M_{(l,i)} =\sum_{l=1}^H   M_{(l,1)} +\sum_{l=2}^H   M_{(l,2)} +  \cdots +\sum_{l=H}^H   M_{(l,H)} = \sum_{i=1}^H \sum_{l=i}^H M_{(l,i)}. 
\end{align*}
Moreover, for each $i \in \{1,\dots,H\}$,
\begin{align*}
 \sum_{l=i}^H M_{(l,i)} & =\sum_{l=i}^H \sum_{k=i}^H\vect[\nabla_{(l)}L(W,B)]\T J_{(i,l)}\T J_{(i,k)}\vect\left[\nabla_{(k)}L(W,B) \right]
\\ & = \left(\sum_{l=i}^H   J_{(i,l)}\vect[\nabla_{(l)}L(W,B)]\right)\T \left(\sum_{k=i}^H J_{(i,k)}\vect\left[\nabla_{(k)}L(W,B) \right]\right)
\\ & = \left\| \sum_{l=i}^H   J_{(i,l)}\vect[\nabla_{(l)}L(W,B)] \right\|_2^2
\end{align*}
Using these facts, the second term can be simplified as
\begin{align*}
&\sum_{l=1}^H  \sum_{i=1}^l  \sum_{k=i}^H \vect[\nabla_{(l)}L(W,B)]\T J_{(i,l)}\T J_{(i,k)}\vect\left[\nabla_{(k)}L(W,B) \right]
\\ & =\sum_{l=1}^H  \sum_{i=1}^l M_{(l,i)}
\\ & =\sum_{i=1}^H \sum_{l=i}^H M_{(l,i)} 
\\ & =  \sum_{i=1}^H  \left\| \sum_{l=i}^H   J_{(i,l)}\vect[\nabla_{(l)}L(W,B)] \right\|_2^2
\end{align*}
Combining these,
\begin{align} \label{eq:1}
\frac{d}{dt} L(W,B) &=- \sum_{l=0}^H\vect[\nabla_{(l)}L(W,B)]\T F_{(l)}\vect[\nabla_{(l)}L(W,B)]-  \sum_{i=1}^H  \left\| \sum_{l=i}^H   J_{(i,l)}\vect[\nabla_{(l)}L(W,B)] \right\|_2^{2}
\end{align}
Since $F_{(l)}$ is real symmetric and positive semidefinite,
\begin{align}  \label{eq:proof:1}
\frac{d}{dt} L(W,B) &\le- \sum_{l=0}^H \lambda_{\min}(F_{(l)}) \|\vect[\nabla_{(l)}L(W,B)]\|_2^2-  \sum_{i=1}^H  \left\| \sum_{l=i}^H   J_{(i,l)}\vect[\nabla_{(l)}L(W,B)] \right\|_2^{2}
\end{align}

\subsubsection{Completing the proof by using the assumption of the square loss} \label{sec:new:10}

Using the assumption that $L(W,B)=\ell(f(X,W,B), Y)= \|f(X,W,B)-Y\|^2_F$ with $\hat Y = f(X,W,B)$,
we have$$
\frac{\partial L(W,B)}{\partial \hat Y}=\frac{\partial }{\partial \hat Y}\| \hat Y-Y\|^2_F = 2(\hat Y-Y) \in \RR^{m_y\times n},
$$
and
hence
\begin{align*}
\vect[\nabla_{(l)}L(W,B)]=\vect\left[\frac{\partial L(W,B)}{\partial \hat Y} (X S^{l} )\T \right] =2\vect\left[(\hat Y-Y) (X S^{l} )\T \right] =2 [X S^{l} \otimes I_{m_{y}} ] \vect[\hat Y-Y]. 
\end{align*}
Therefore,
\begin{align} \label{eq:new:15}
\|\vect[\nabla_{(l)}L(W,B)]\|_2^2=4\vect[\hat Y-Y]\T  [(X S^{l})\T X S^{l} \otimes I_{m_{y}} ] \vect[\hat Y-Y].   
\end{align}
We are now ready to complete the proof of  Theorem~\ref{thm:1} for each cases  (i), (ii) and (iii).
 
\subsubsection{Case (I): Completing The Proof of Theorem~\ref{thm:1} (i)}
Using equation \eqref{eq:proof:1} and \eqref{eq:new:15} with $\lambda_{W,B}=\min_{0\le l \le H}\lambda_{\min}(F_{(l)})$,
we have that\begin{align*}
\frac{d}{dt} L(W,B) &\le-\lambda_{W,B} \sum_{l=0}^H  \|\vect[\nabla_{(l)}L(W,B)]\|_2^2-  \sum_{i=1}^H  \left\| \sum_{l=i}^H   J_{(i,l)}\vect[\nabla_{(l)}L(W,B)] \right\|_2^{2}
 \\ &\le-4\lambda_{W,B} \sum_{l=0}^H  \vect[\hat Y-Y]\T  [(X S^{l})\T X S^{l} \otimes I_{m_{y}} ] \vect[\hat Y-Y]-  \sum_{i=1}^H  \left\| \sum_{l=i}^H   J_{(i,l)}\vect[\nabla_{(l)}L(W,B)] \right\|_2^{2}
 \\ &\le-4\lambda_{W,B}  \vect[\hat Y-Y]\T  \left[\left(\sum_{l=0}^H(X S^{l})\T X S^{l}\right) \otimes I_{m_{y}} \right] \vect[\hat Y-Y]-  \sum_{i=1}^H  \left\| \sum_{l=i}^H   J_{(i,l)}\vect[\nabla_{(l)}L(W,B)] \right\|_2^{2}
\\ & =-4\lambda_{W,B}  \vect[\hat Y-Y]\T  \left[ G_H\T G_H\otimes I_{m_{y}} \right] \vect[\hat Y-Y]-  \sum_{i=1}^H  \left\| \sum_{l=i}^H   J_{(i,l)}\vect[\nabla_{(l)}L(W,B)] \right\|_2^2 
\end{align*}
where the last line follows from the following fact:
$$
G_H\T G_H = 
\begin{bmatrix}
X \\
XS \\
\vdots \\
XS^H \\
\end{bmatrix}\T  
\begin{bmatrix}
X \\
XS \\
\vdots \\
XS^H \\
\end{bmatrix}  =\sum_{l=0}^H (X S^{l})\T X S^{l}.  
$$
Decompose $\vect[\hat Y-Y]$ as $\vect[\hat Y-Y] =v+v^\perp$, where $v=\Pb_{G_H\T \otimes I_{m_{y}}  }\vect[\hat Y-Y]$, $v^\perp=(I_{m_{y}n}-\Pb_{G_H\T \otimes I_{m_{y}} })\vect[\hat Y-Y]$, and $\Pb_{G_H\T \otimes I_{m_{y}}} \in \RR^{m_yn\times m_y n}$ represents the
orthogonal projection onto the column space of $G_H\T \otimes I_{m_{y}} \in  \RR^{ m_{y}n \times (H+1)m_{y}m_{x}}$. Then, 
\begin{align*}
 \vect[\hat Y-Y]\T    \left[ G_H\T G_H\otimes I_{m_{y}} \right]  \vect[\hat Y-Y] 
& =(v+v^\perp)\T\left[ G_H\T \otimes I_{m_{y}} \right]     \left[ G_H\otimes I_{m_{y}} \right]    (v + v^\perp)
 \\ & =v\T  \left[ G_H\T \otimes I_{m_{y}} \right]     \left[ G_H\otimes I_{m_{y}} \right] v
\\ & \ge \sigma^2_{\min}(G_H) \|\Pb_{G_H\T \otimes I_{m_{y}}  }\vect[\hat Y-Y]\|^2_2 
\\ & = \sigma^2_{\min}(G_H) \|\Pb_{G_H\T \otimes I_{m_{y}}  }\vect[\hat Y]-\Pb_{G_H\T \otimes I_{m_{y}}  }\vect[Y]\|^2_2
\\ & = \sigma^2_{\min}(G_H) \|\vect[\hat Y]-\Pb_{G_H\T \otimes I_{m_{y}}  }\vect[Y]\pm\vect[ Y]\|^2_2
 \\ & = \sigma^2_{\min}(G_H) \|\vect[\hat Y]-\vect[ Y]+(I_{m_yn}-\Pb_{G_H\T \otimes I_{m_{y}}  })\vect[ Y]\|^2_2
\\ & \ge \sigma^2_{\min}(G_H)(\|\vect[\hat Y-Y]\|_2 - \|(I_{m_yn}-\Pb_{G_H\T \otimes I_{m_{y}}  })\vect[ Y]\|_2 )^2 
\\ & \ge \sigma^2_{\min}(G_H)(\|\vect[\hat Y-Y]\|_2^2 - \|(I_{m_yn}-\Pb_{G_H\T \otimes I_{m_{y}}  })\vect[ Y]\|_2 ^2, 
\end{align*} 
where we used the fact that the singular values of $\left[ G_H\T \otimes I_{m_{y}} \right]$ are products of singular values of $G_H$ and $I_{m_y}$. 

By noticing that $L(W,B)=\|\vect[\hat Y-Y]\|_2^{2}$ and $L^*_{1:H}= \|(I_{m_yn}-\Pb_{G_H\T \otimes I_{m_{y}}  })\vect[ Y]\|_2^2$ ,
$$
\vect[\hat Y-Y]\T    \left[ G_H\T G_H\otimes I_{m_{y}} \right]  \vect[\hat Y-Y] \ge \sigma^2_{\min}(G_H)(L(W,B)-L^*_{1:H}).
$$
Therefore,
\begin{align*}
\frac{d}{dt} L(W,B) &\le -4\lambda_{W,B}  \vect[\hat Y-Y]\T  \left[ G_H\T G_H\otimes I_{m_{y}} \right] \vect[\hat Y-Y]-  \sum_{i=1}^H  \left\| \sum_{l=i}^H   J_{(i,l)}\vect[\nabla_{(l)}L(W,B)] \right\|_2^2
\\ & \le -4\lambda_{W,B}   \sigma^2_{\min}(G_H)( L(W,B)-L^*_{1:H})-  \sum_{i=1}^H  \left\| \sum_{l=i}^H   J_{(i,l)}\vect[\nabla_{(l)}L(W,B)] \right\|_2^2
\end{align*}
Since $\frac{d}{dt} L^*_{1:H} =0$, 
$$
\frac{d}{dt} (L(W,B) -L^*_{1:H}) \le -4\lambda_{W,B}   \sigma^2_{\min}(G_H)(L(W,B) -L^*_{1:H})-  \sum_{i=1}^H  \left\| \sum_{l=i}^H   J_{(i,l)}\vect[\nabla_{(l)}L(W,B)] \right\|_2^2 
$$
By defining  $\Lb=L(W,B) -L^*_{1:H}$, 
\begin{align} \label{eq:proof:2}
\frac{d\Lb }{dt} \le -4\lambda_{W,B}   \sigma^2_{\min}(G_H)\Lb-  \sum_{i=1}^H  \left\| \sum_{l=i}^H   J_{(i,l)}\vect[\nabla_{(l)}L(W,B)] \right\|_2^2 
\end{align} 
Since $\frac{d}{dt} \Lb \le0$ and $\Lb\ge 0$, if $\Lb=0$ at some time $\bar t$, then $\Lb=0$ for any time $t \ge \bar t$. Therefore, if $\Lb=0$ at some time $\bar t$, then we have the desired statement of this theorem for any time $t\ge \bar t$. Thus, we can focus on the time interval $[0, \bar t]$ such that  $\Lb>0$ for  any time $t\in[0, \bar t]$ (here, it is allowed to have $\bar t = \infty$). Thus, focusing on the time interval with $\Lb>0$ , equation \eqref{eq:proof:2} implies that 
\begin{align*}
\frac{1}{\Lb}\frac{d\Lb}{dt}  \le -4\lambda_{W,B}   \sigma^2_{\min}(G_H)-\frac{1}{\Lb}  \sum_{i=1}^H  \left\| \sum_{l=i}^H   J_{(i,l)}\vect[\nabla_{(l)}L(W,B)] \right\|_2^2
\end{align*}
By taking integral over time
\begin{align*}
\int_{0}^T \frac{1}{\Lb}\frac{d\Lb}{dt}  dt\le -\int_{0}^T 4 \lambda_{W,B}   \sigma^2_{\min}(G_H)dt-\int_{0}^T \frac{1}{\Lb}  \sum_{i=1}^H  \left\| \sum_{l=i}^H   J_{(i,l)}\vect[\nabla_{(l)}L(W,B)] \right\|_2^{2}dt
\end{align*}
By using the substitution rule for integrals, $\int_{0}^T \frac{1}{\Lb}\frac{d\Lb}{dt}  dt=\int_{\Lb_0}^{\Lb_T} \frac{1}{\Lb}d\Lb=\log(\Lb_T)-\log(\Lb_0)$, where $\Lb_0=L(W_{0},B_{0}) -L^*_{1:H}$ and $\Lb_T=L(W_{T},B_{T}) -L^*_{1:H}$. Thus,
\begin{align*}
\log(\Lb_T)-\log(\Lb_0)\le - 4 \sigma^2_{\min}(G_H) \int_{0}^T  \lambda_{W,B} dt-\int_{0}^T \frac{1}{\Lb}  \sum_{i=1}^H  \left\| \sum_{l=i}^H   J_{(i,l)}\vect[\nabla_{(l)}L(W,B)] \right\|_2^{2}dt
\end{align*} 
which implies that
\begin{align*}
\Lb_T &\le e^{\log(\Lb_0)- 4 \sigma^2_{\min}(G_H) \int_{0}^T  \lambda_{W,B} dt-\int_{0}^T \frac{1}{\Lb}  \sum_{i=1}^H  \left\| \sum_{l=i}^H   J_{(i,l)}\vect[\nabla_{(l)}L(W,B)] \right\|_2^{2}dt}
\\ & =\Lb_0 e^{- 4 \sigma^2_{\min}(G_H) \int_{0}^T  \lambda_{W,B} dt-\int_{0}^T \frac{1}{\Lb}  \sum_{i=1}^H  \left\| \sum_{l=i}^H   J_{(i,l)}\vect[\nabla_{(l)}L(W,B)] \right\|_2^{2}dt}
\end{align*}
By recalling the definition of $\Lb=L(W,B) -L^*_{1:H}$ and that  $\frac{d}{dt} \Lb \le0$, we have that if $L(W_T,B_T) -L^*_{1:H}> 0$, then $L(W_t,B_t) -L^*_{1:H}>0$ for all $t \in [0, T]$, and
\begin{align*}
L(W_T,B_T) -L^*_{1:H}\le (L(W_0,B_0) -L^*_{1:H})
 e^{- 4 \sigma^2_{\min}(G_H) \int_{0}^T  \lambda_{W_{t},B_{t}} dt-\int_{0}^T \frac{1}{L(W_t,B_t) -L^*}  \sum_{i=1}^H  \left\| \sum_{l=i}^H   J_{(i,l)}\vect[\nabla_{(l)}L(W_{t},B_{t})] \right\|_2^{2}dt}. 
\end{align*}
By noticing that $\lambda_T^{(1:H)}= \inf_{t \in[0,T]}  \lambda_{W_t,B_t}$ and  that $\int_{0}^T \frac{1}{L(W_t,B_t) -L^*}  \sum_{i=1}^H  \left\| \sum_{l=i}^H   J_{(i,l)}\vect[\nabla_{(l)}L(W_{t},B_{t})] \right\|_2^{2}dt \ge 0$, this implies that
\begin{align*}
L(W_T,B_T) -L^*_{1:H} &\le (L(W_0,B_0) -L^*_{1:H})
 e^{- 4\lambda_T^{(1:H)}\sigma^2_{\min}(G_H)T-\int_{0}^T \frac{1}{L(W_t,B_t) -L^*}  \sum_{i=1}^H  \left\| \sum_{l=i}^H   J_{(i,l)}\vect[\nabla_{(l)}L(W_{t},B_{t})] \right\|_2^{2}dt} 
\\ & \le(L(W_0,B_0) -L^*_{1:H})
 e^{- 4\lambda_T^{(1:H)}\sigma^2_{\min}(G_H)T}. 
\end{align*}
This completes the proof of Theorem \ref{thm:1} (i) for the case of $\Ical=[n]$. Since every  step in this proof   is valid  when we replace $ f(X,W,B)$ by $f(X,W,B)_{*\Ical}$ and  $XS^{l}$ by $X(S^l)_{*\Ical}$ without using any assumption on $S$ or the relation between $S^{l-1}$ and $S$, our proof also yields for the general case of $\Ical$ that 
\begin{align*}
L(W_T,B_T) -L^*_{1:H} & \le(L(W_0,B_0) -L^*_{1:H})
 e^{- 4\lambda_T^{(1:H)}\sigma^2_{\min}((G_H)_{*\Ical})T}. 
\end{align*}
\qed

\subsubsection{Case (ii): Completing The Proof of Theorem~\ref{thm:1} (ii)}
Using equation \eqref{eq:proof:1} and \eqref{eq:new:15} , we have that
for any $H' \in \{0,1,\dots,H\}$,
\begin{align*}
\frac{d}{dt} L(W,B) &\le-  \lambda_{\min}(F_{(H')}) \|\vect[\nabla_{(H')}L(W,B)]\|_2^2
 \\ &\le-4  \lambda_{\min}(F_{(H')})   \vect[\hat Y-Y]\T  [(X S^{H'})\T X S^{H'} \otimes I_{m_{y}} ] \vect[\hat Y-Y]
\\ & =-4\lambda_{W,B}  \vect[\hat Y-Y]\T  \left[ \tilde G_{H'}\T \tilde G_{H'}\otimes I_{m_{y}} \right] \vect[\hat Y-Y], 
\end{align*}
where
$$
\lambda_{W,B}: =  \lambda_{\min}(F_{(H')}),   
$$
and
$$
\tilde G_{H'} : =X S^{H'}.
$$

Decompose $\vect[\hat Y-Y]$ as $\vect[\hat Y-Y] =v+v^\perp$, where $v=\Pb_{\tilde G_{H'}\T \otimes I_{m_{y}}  }\vect[\hat Y-Y]$, $v^\perp=(I_{m_{y}n}-\Pb_{\tilde G_{H'}\T \otimes I_{m_{y}} })\vect[\hat Y-Y]$, and $\Pb_{\tilde G_{H'}\T \otimes I_{m_{y}}} \in \RR^{m_yn\times m_y n}$ represents the
orthogonal projection onto the column space of $\tilde G_{H'}\T \otimes I_{m_{y}} \in  \RR^{ m_{y}n \times m_{y}m_{x}}$. Then, 
\begin{align*}
 \vect[\hat Y-Y]\T    \left[ \tilde G_{H'}\T \tilde G_{H'}\otimes I_{m_{y}} \right]  \vect[\hat Y-Y] 
& =(v+v^\perp)\T\left[ \tilde G_{H'}\T \otimes I_{m_{y}} \right]     \left[ \tilde G_{H'}\otimes I_{m_{y}} \right]    (v + v^\perp)
 \\ & =v\T  \left[ \tilde G_{H'}\T \otimes I_{m_{y}} \right]     \left[ \tilde G_{H'}\otimes I_{m_{y}} \right] v
\\ & \ge \sigma^2_{\min}(\tilde G_{H'}) \|\Pb_{\tilde G_{H'}\T \otimes I_{m_{y}}  }\vect[\hat Y-Y]\|^2_2 
\\ & = \sigma^2_{\min}(\tilde G_{H'}) \|\Pb_{\tilde G_{H'}\T \otimes I_{m_{y}}  }\vect[\hat Y]-\Pb_{\tilde G_{H'}\T \otimes I_{m_{y}}  }\vect[Y]\|^2_2
\\ & = \sigma^2_{\min}(\tilde G_{H'}) \|\vect[\hat Y]-\Pb_{\tilde G_{H'}\T \otimes I_{m_{y}}  }\vect[Y]\pm\vect[ Y]\|^2_2
 \\ & = \sigma^2_{\min}(\tilde G_{H'}) \|\vect[\hat Y]-\vect[ Y]+(I_{m_yn}-\Pb_{\tilde G_{H'}\T \otimes I_{m_{y}}  })\vect[ Y]\|^2_2
\\ & \ge \sigma^2_{\min}(\tilde G_{H'})(\|\vect[\hat Y-Y]\|_2 - \|(I_{m_yn}-\Pb_{\tilde G_{H'}\T \otimes I_{m_{y}}  })\vect[ Y]\|_2 )^2 
\\ & \ge \sigma^2_{\min}(\tilde G_{H'})(\|\vect[\hat Y-Y]\|_2^2 - \|(I_{m_yn}-\Pb_{\tilde G_{H'}\T \otimes I_{m_{y}}  })\vect[ Y]\|_2 ^2, 
\end{align*} 
where we used the fact that the singular values of $\left[ \tilde G_{H'}\T \otimes I_{m_{y}} \right]$ are products of singular values of $\tilde G_{H'}$ and $I_{m_y}$. 

By noticing that $L(W,B)=\|\vect[\hat Y-Y]\|_2^{2}$ and $L^*_{{H'}} = \|(I_{m_yn}-\Pb_{\tilde G_{H'}\T \otimes I_{m_{y}}  })\vect[ Y]\|_2^2$ ,
we have that for any $H' \in \{0,1,\dots,H \}$,
\begin{align} \label{eq:new:16}
\vect[\hat Y-Y]\T    \left[ \tilde G_{H'}\T \tilde G_{H'}\otimes I_{m_{y}} \right]  \vect[\hat Y-Y] \ge \sigma^2_{\min}(\tilde G_{H'})(L(W,B)-L^*_{{H'}} ).
\end{align}
Therefore,
\begin{align*}
\frac{d}{dt} L(W,B) &\le -4\lambda_{W,B}  \vect[\hat Y-Y]\T  \left[ \tilde G_{H'}\T \tilde G_{H'}\otimes I_{m_{y}} \right] \vect[\hat Y-Y]
\\ & \le -4\lambda_{W,B}   \sigma^2_{\min}(\tilde G_{H'})( L(W,B)-L^*_{{H'}})
\end{align*}
Since $\frac{d}{dt} L^*_{{H'}}=0$, 
$$
\frac{d}{dt} (L(W,B) -L^*_{{H'}}) \le -4\lambda_{W,B}   \sigma^2_{\min}(\tilde G_{H'})(L(W,B) -L^*_{{H'}}) 
$$
By defining  $\Lb=L(W,B) -L^*_{{H'}}$, 
\begin{align} \label{eq:proof:2_3}
\frac{d\Lb }{dt} \le -4\lambda_{W,B}   \sigma^2_{\min}(\tilde G_{H'})\Lb 
\end{align} 
Since $\frac{d}{dt} \Lb \le0$ and $\Lb\ge 0$, if $\Lb=0$ at some time $\bar t$, then $\Lb=0$ for any time $t \ge \bar t$. Therefore, if $\Lb=0$ at some time $\bar t$, then we have the desired statement of this theorem for any time $t\ge \bar t$. Thus, we can focus on the time interval $[0, \bar t]$ such that  $\Lb>0$ for  any time $t\in[0, \bar t]$ (here, it is allowed to have $\bar t = \infty$). Thus, focusing on the time interval with $\Lb>0$ , equation \eqref{eq:proof:2_3} implies that 
\begin{align*}
\frac{1}{\Lb}\frac{d\Lb}{dt}  \le -4\lambda_{W,B}   \sigma^2_{\min}(\tilde G_{H'})
\end{align*}
By taking integral over time
\begin{align*}
\int_{0}^T \frac{1}{\Lb}\frac{d\Lb}{dt}  dt\le -\int_{0}^T 4 \lambda_{W,B}   \sigma^2_{\min}(\tilde G_{H'})dt
\end{align*}
By using the substitution rule for integrals, $\int_{0}^T \frac{1}{\Lb}\frac{d\Lb}{dt}  dt=\int_{\Lb_0}^{\Lb_T} \frac{1}{\Lb}d\Lb=\log(\Lb_T)-\log(\Lb_0)$, where $\Lb_0=L(W_{0},B_{0}) -L^*$ and $\Lb_T=L(W_{T},B_{T}) -L^*_{{H'}}$. Thus,
\begin{align*}
\log(\Lb_T)-\log(\Lb_0)\le - 4 \sigma^2_{\min}(\tilde G_{H'}) \int_{0}^T  \lambda_{W,B} dt
\end{align*} 
which implies that
\begin{align*}
\Lb_T &\le e^{\log(\Lb_0)- 4 \sigma^2_{\min}(\tilde G_{H'}) \int_{0}^T  \lambda_{W,B} dt}
\\ & =\Lb_0 e^{- 4 \sigma^2_{\min}(\tilde G_{H'}) \int_{0}^T  \lambda_{W,B} dt}
\end{align*}
By recalling the definition of $\Lb=L(W,B) -L^*_{{H'}}$ and that  $\frac{d}{dt} \Lb \le0$, we have that if $L(W_T,B_T) -L^*_{{H'}}> 0$, then $L(W_t,B_t) -L^*_{{H'}}>0$ for all $t \in [0, T]$, and
\begin{align*}
L(W_T,B_T) -L^*_{{H'}}\le (L(W_0,B_0) -L^*_{{H'}})
 e^{- 4 \sigma^2_{\min}(\tilde G_{H'}) \int_{0}^T  \lambda_{W_{t},B_{t}} dt}. 
\end{align*}
By noticing that $\lambda^{({H'})}_T= \inf_{t \in[0,T]}  \lambda_{W_t,B_t}$,  this implies that
for any $H' \in \{0,1,\dots,H\}$,
\begin{align*}
L(W_T,B_T) -L^*_{{H'}} &\le (L(W_0,B_0) -L^*_{{H'}})
 e^{- 4\lambda^{({H'})}_T \sigma^2_{\min}(\tilde G_{H'})T}  
 \\ & =(L(W_0,B_0) -L^*_{H'})
 e^{- 4\lambda^{(H)}_T \sigma^2_{\min}(X S^{H'})T} 
\end{align*}
This completes the proof of Theorem \ref{thm:1} (ii) for the case of $\Ical=[n]$. Since every  step in this proof   is valid  when we replace $ f(X,W,B)$ by $f(X,W,B)_{*\Ical}$ and  $XS^{l}$ by $X(S^l)_{*\Ical}$ without using any assumption on $S$ or the relation between $S^{l-1}$ and $S$, our proof also yields for the general case of $\Ical$ that
 \begin{align*}
L(W_T,B_T) -L^*_{{H'}} &\le (L(W_0,B_0) -L^*_{H'})
 e^{- 4\lambda^{(H)}_T \sigma^2_{\min}(X (S^{H'})_{*\Ical})T} 
\end{align*}
\qed

\subsubsection{Case (iii): Completing The Proof of Theorem~\ref{thm:1} (iii)} \label{sec:new:11}
In this case, we have the following assumption: there exist   $l,l' \in \{0,\dots, H\}$ with  $l< l'$ such that $L^*_{l} \ge L^*_{l+1} \ge \cdots \ge L^*_{l'}$  or    $L^*_{l} \le L^*_{l+1} \le \cdots \le L^*_{l'}$. Using equation \eqref{eq:proof:1} and \eqref{eq:new:15} with $\tilde G_{l} =X S^{l}$,
we have that\begin{align*}
\frac{d}{dt} L(W,B) &\le- \sum_{l=0}^H \lambda_{\min}(F_{(l)}) \|\vect[\nabla_{(l)}L(W,B)]\|_2^{2}
 \\ &\le-4 \sum_{l=0}^H  \lambda_{\min}(F_{(l)})\vect[\hat Y-Y]\T  [(X S^{l})\T X S^{l} \otimes I_{m_{y}} ] \vect[\hat Y-Y]
\\ & =-4 \sum_{l=0}^H  \lambda_{\min}(F_{(l)})\vect[\hat Y-Y]\T  [\tilde G_{l}\T \tilde G_{l}  \otimes I_{m_{y}} ] \vect[\hat Y-Y] 
\end{align*}  

Using \eqref{eq:new:16}, since $\vect[\hat Y-Y]\T    \left[ \tilde G_{l}\T \tilde G_{l}\otimes I_{m_{y}} \right]  \vect[\hat Y-Y] \ge \sigma^2_{\min}(\tilde G_l)(L(W,B)-L^*_{{l}} )$ for any $l\in \{0,1,\dots,H \}$, 
\begin{align} \label{eq:new:17}
\frac{d}{dt} L(W,B) &\le-4 \sum_{l=0}^H  \lambda_{\min}(F_{(l)}) \sigma^2_{\min}(\tilde G_l)(L(W,B)-L^*_{{l}} ). 
\end{align}
Let  $l''=l$ if $L^*_{l} \ge L^*_{l+1} \ge \cdots \ge L^*_{l'}$, and $l''=l'$ if $L^*_{l} \le L^*_{l+1} \le \cdots \le L^*_{l'}$. Then, using \eqref{eq:new:17} and the assumption of $L^*_{l} \ge L^*_{l+1} \ge \cdots \ge L^*_{l'}$  or    $L^*_{l} \le L^*_{l+1} \le \cdots \le L^*_{l'}$ for some    $l,l' \in \{0,\dots, H\}$, we have that  
\begin{align} 
\frac{d}{dt} L(W,B) &\le-4(L(W,B)-L^*_{{l''}} ) \sum_{k=l}^{l'}  \lambda_{\min}(F_{(k)}) \sigma^2_{\min}(\tilde G_k). 
\end{align}

Since $\frac{d}{dt} L^*_{{l''}} =0$, 
$$
\frac{d}{dt} (L(W,B) -L^*_{{l''}}) \le-4(L(W,B)-L^*_{{l''}} ) \sum_{k=l}^{l'}  \lambda_{\min}(F_{(k)}) \sigma^2_{\min}(\tilde G_k). 
$$
By taking integral over time in the same way as that in the proof for the case of (i) and (ii), we have that 
\begin{align} \label{eq:new:18}
L(W_T,B_T) -L^*_{{l''}}  \le (L(W_0,B_0) -L^*_{{l''}}  )  e^{- 4\sum_{k=l}^{l'}  \sigma^2_{\min}(\tilde G_k) \int_{0}^T  \lambda_{\min}(F_{(k),t})dt}   
\end{align}

Using the property of Kronecker product, 
$$
\lambda_{\min}(F_{(l),t})=\lambda_{\min}([(B_{(l),t} \dots B_{(1),t})\T B_{(l),t} \cdots B_{(1),t} \otimes I_{m_{y}}])=\lambda_{\min}((B_{(l),t} \dots B_{(1),t})\T B_{(l),t} \cdots B_{(1),t}), 
$$ which implies that  $\lambda_T^{(k)}= \inf_{t \in[0,T]}  \lambda_{\min}(F_{(k),t})$. Therefore, equation \eqref{eq:new:18} with $\lambda_T^{(k)}= \inf_{t \in[0,T]}  \lambda_{\min}(F_{(k),t})$ yields that  
\begin{align}
\nonumber L(W_T,B_T) -L^*_{{l''}}  &\le (L(W_0,B_0) -L^*_{{l''}}  )  e^{- 4\sum_{k=l}^{l'}  \lambda_T^{(k)}\sigma^2_{\min}(\tilde G_k)T}   
\\& =   (L(W_0,B_0) -L^*_{{l''}}  )  e^{- 4\sum_{k=l}^{l'}  \lambda_T^{(k)}\sigma^2_{\min}(X S^{k})T}   
\end{align}
This completes the proof of Theorem \ref{thm:1} (iii) for the case of $\Ical=[n]$. Since every  step in this proof   is valid  when we replace $ f(X,W,B)$ by $f(X,W,B)_{*\Ical}$ and  $XS^{l}$ by $X(S^l)_{*\Ical}$ without using any assumption on $S$ or the relation between $S^{l-1}$ and $S$, our proof also yields for the general case of $\Ical$ that
\begin{align}
\nonumber L(W_T,B_T) -L^*_{{l''}}  &\le (L(W_0,B_0) -L^*_{{l''}}  )  e^{- 4\sum_{k=l}^{l'}  \lambda_T^{(k)}\sigma^2_{\min}(X( S^{k})_{*\Ical})T}.   
\end{align}
\qed

\subsection{Proof of Proposition~\ref{prop:3}}
\label{sec:proof_prop:3}
From Definition \ref{def:4}, for any $l\in\{1,2,\dots,H\}$, we have that    $\sigma_{\min}( \bB^{(1:l)})=\sigma_{\min}(B_{(l)}B_{(l-1)} \cdots B_{(1)}) \ge \gamma$ for all $(W,B)$ such that $L(W,B)\le  L(W_0, B_0)$.
From equation \eqref{eq:proof:1} in the proof of Theorem~\ref{thm:1}, it holds that
$
\frac{d}{dt} L(W_{t},B_{t}) \le 0
$     
for all $t$. Thus, we have that $L(W_{t},B_{t})\le  L(W_0, B_0)$ and hence  $\sigma_{\min}( \bB^{(1:l)}_t) \ge \gamma$ for all $t$.  Under this problem setting ($m_l\ge m_x$),  this implies that $\lambda_{\min}(( \bB^{(1:l)}_t)\T  \bB^{(1:l)}_t) \ge \gamma^{2}$ for all $t$ and thus $\lambda_T^{(1:H)}\ge\gamma^{2}$.

\subsection{Proof of Theorem~\ref{thm:4}}
\label{sec:proof_thm:4}
The proof of  Theorem~\ref{thm:4} follows from the intermediate results of the proofs of Theorem \ref{thm:6} and Theorem \ref{thm:1} as we show in the following. For  the non-multiscale case,  from equation \eqref{eq:4} in the proof of Theorem \ref{thm:6}, we have that
\begin{align*} 
\frac{d}{dt} L_{1}(W,B) &=- \|\vect[\nabla_{(H)}L(W,B)]\|_{F_{(H)}}^{2}-  \sum_{i=1}^H  \left\|   J_{(i,H)}\vect[\nabla_{(H)}L(W,B)] \right\|_2^{2}
\end{align*}
where 
$$
\|\vect[\nabla_{(H)}L(W,B)]\|_{F_{(H)}}^2 :=\vect[\nabla_{(H)}L(W,B)]\T F_{(H)}\vect[\nabla_{(H)}L(W,B)].
$$
Since equation \eqref{eq:4} in the proof of Theorem \ref{thm:1} is derived without the assumption on the square loss, this holds for any differentiable loss $\ell$. By noticing that $\nabla_{(H)}L(W,B)=V (X (S^{H})_{*\Ical})\T$, we have that
\begin{align*} 
\frac{d}{dt} L_{1}(W,B) &=- \|\vect[V (X (S^{H})_{*\Ical})\T]\|_{F_{(H)}}^{2}-  \sum_{i=1}^H  \left\|   J_{(i,H)}\vect[V (X (S^{H})_{*\Ical})\T] \right\|_2^{2}.
\end{align*}
This proves the statement of Theorem~\ref{thm:4} (i).

For  the multiscale case,  from equation \eqref{eq:1} in the proof of Theorem \ref{thm:1}, we have that
\begin{align} \label{eq:2}
\frac{d}{dt} L_{2}(W,B) &=- \sum_{l=0}^H \|\vect[\nabla_{(l)}L(W,B)]\|_{F_{(l)}}^{2}-  \sum_{i=1}^H  \left\| \sum_{l=i}^H   J_{(i,l)}\vect[\nabla_{(l)}L(W,B)] \right\|_2^{2}
\end{align}
where 
$$
\|\vect[\nabla_{(l)}L(W,B)]\|_{F_{(l)}}^2 :=\vect[\nabla_{(l)}L(W,B)]\T F_{(l)}\vect[\nabla_{(l)}L(W,B)].
$$
Since equation \eqref{eq:1} in the proof of Theorem \ref{thm:1} is derived without the assumption on the square loss, this holds for any differentiable loss $\ell$. Since every  step to derive  equation \eqref{eq:1}   is valid  when we replace $ f(X,W,B)$ by $f(X,W,B)_{*\Ical}$ and  $XS^{l}$ by $X(S^l)_{*\Ical}$ without using any assumption on $S$ or the relation between $S^{l-1}$ and $S$, the  steps to derive  equation \eqref{eq:1} also yields this for the general case of $\Ical$: i.e., $\nabla_{(l)}L(W,B)=V (X (S^{l})_{*\Ical})\T$. Thus, we have that
\begin{align*} 
\frac{d}{dt} L_{1}(W,B) &=- \sum_{l=0}^H \|\vect[V (X (S^{l})_{*\Ical})\T]\|_{F_{(l)}}^{2}-  \sum_{i=1}^H  \left\| \sum_{l=i}^H   J_{(i,l)}\vect[V (X (S^{l})_{*\Ical})\T] \right\|_2^2
\end{align*}
This completes the proof of Theorem~\ref{thm:4} (ii).

\section{Additional Experimental  Results}
\label{sec:results}
In this section, we present additional experimental results.

\begin{figure*}[p]
\centering
    \begin{subfigure}[b]{0.46\textwidth}
     \centering
        \includegraphics[width=0.9\textwidth]{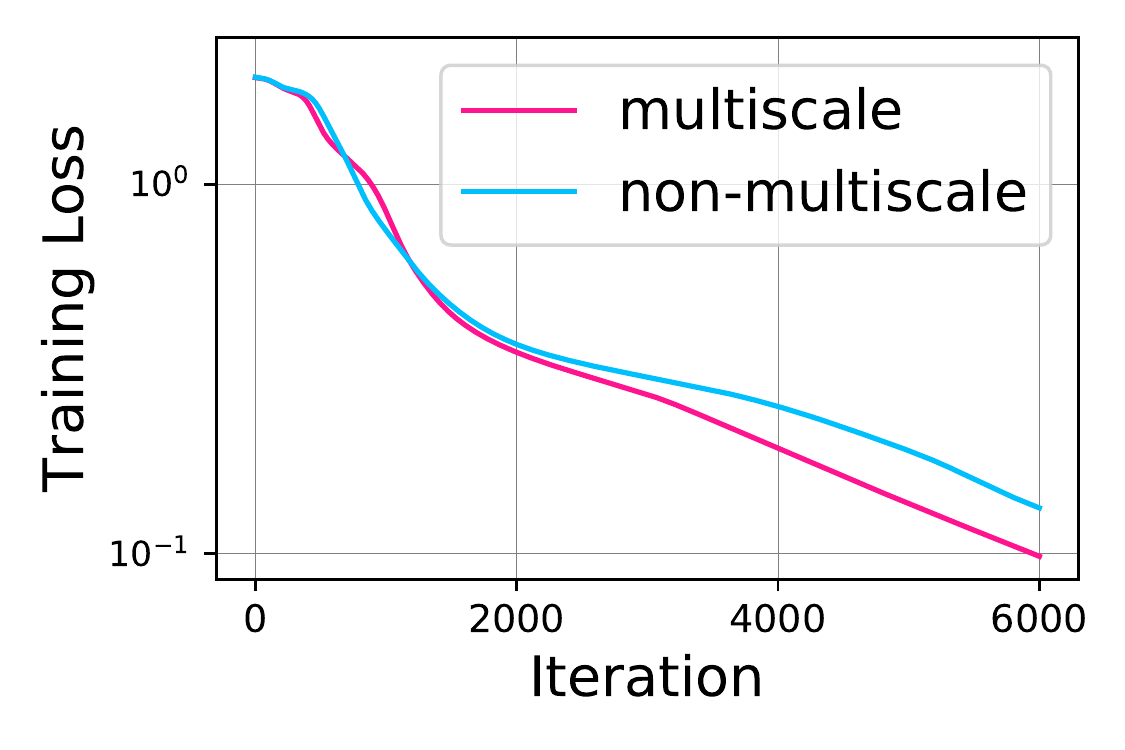} 
        \centering
        \caption{Linear and Cora.}
    \end{subfigure}  
    \begin{subfigure}[b]{0.46\textwidth}
     \centering
        \includegraphics[width=0.9\textwidth]{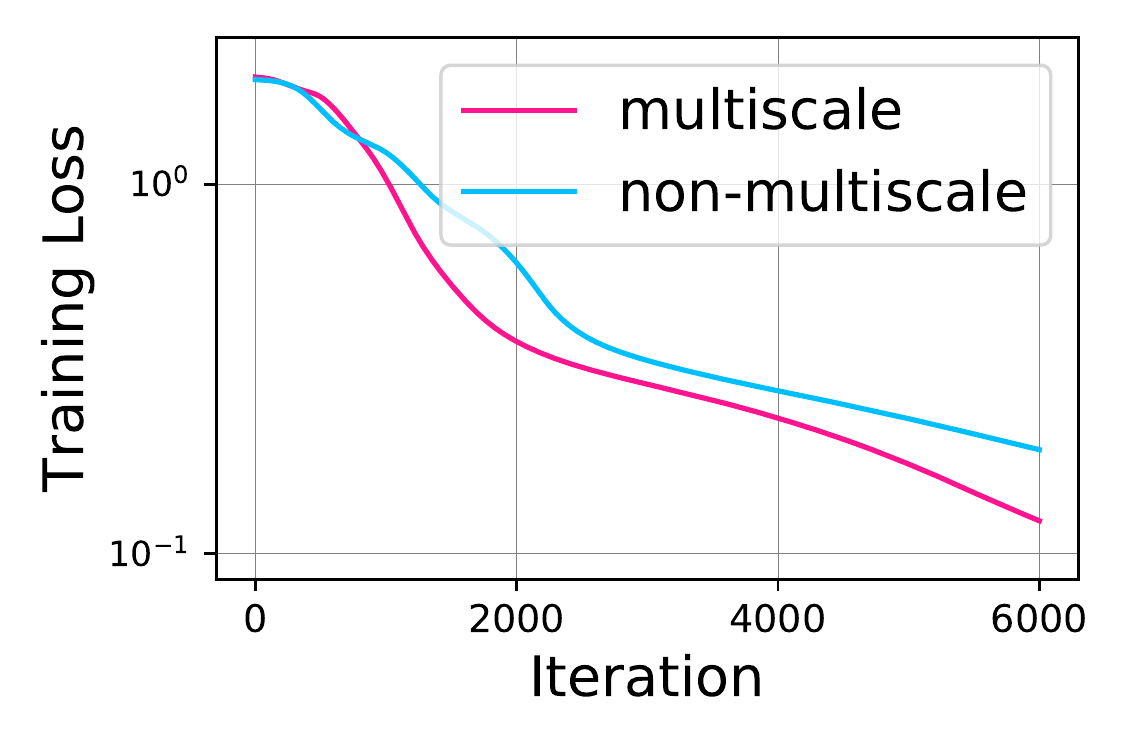} 
        \centering
        \caption{ReLU and Cora.}
    \end{subfigure}   
     \begin{subfigure}[b]{0.46\textwidth}
     \centering
        \includegraphics[width=0.9\textwidth]{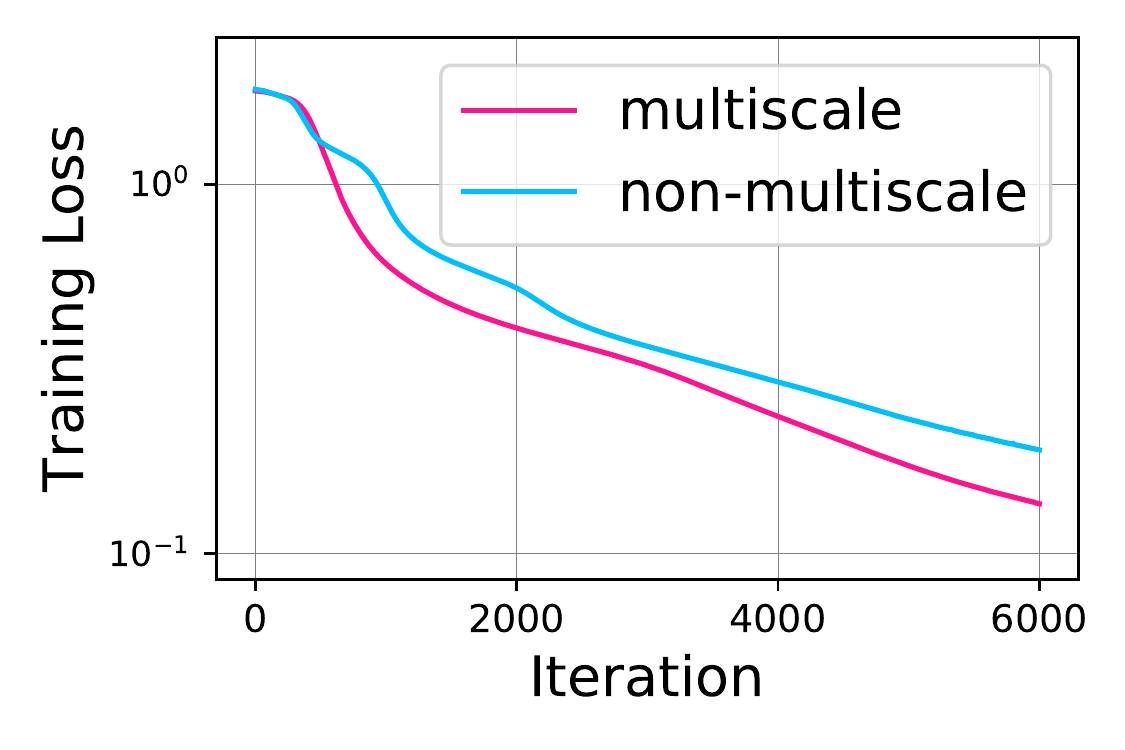} 
        \centering
        \caption{Linear and Citeseer.}
    \end{subfigure}  
    \begin{subfigure}[b]{0.46\textwidth}
     \centering
        \includegraphics[width=0.9\textwidth]{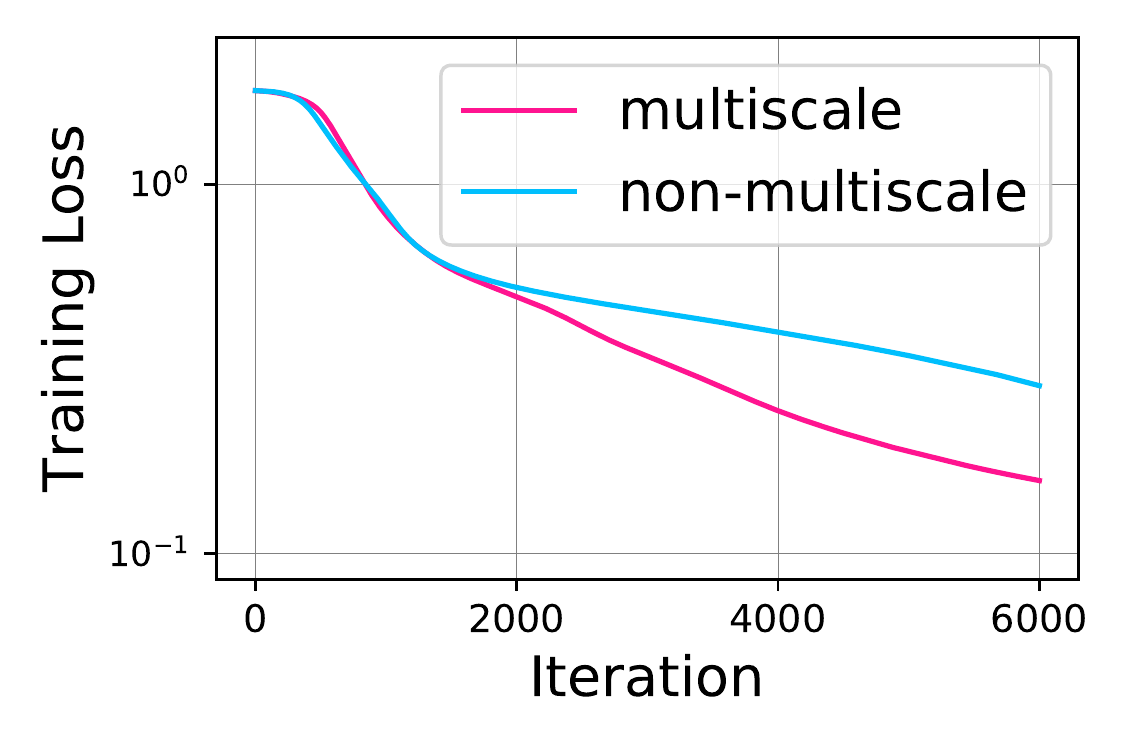} 
        \centering
        \caption{ReLU and Citeseer.}
    \end{subfigure}   
        \vspace{0.2in}
    \caption{\textbf{Multiscale skip connection accelerates GNN training}. We plot the training curves of GNNs with ReLU and linear activation on the \textit{Cora} and \textit{Citeseer} dataset. We use the GCN model with learning rate $5e-5$, six layers, and hidden dimension  $32$. }
\end{figure*}
\clearpage

\begin{figure*}[p]
\centering
    \begin{subfigure}[b]{0.46\textwidth}
     \centering
        \includegraphics[width=0.9\textwidth]{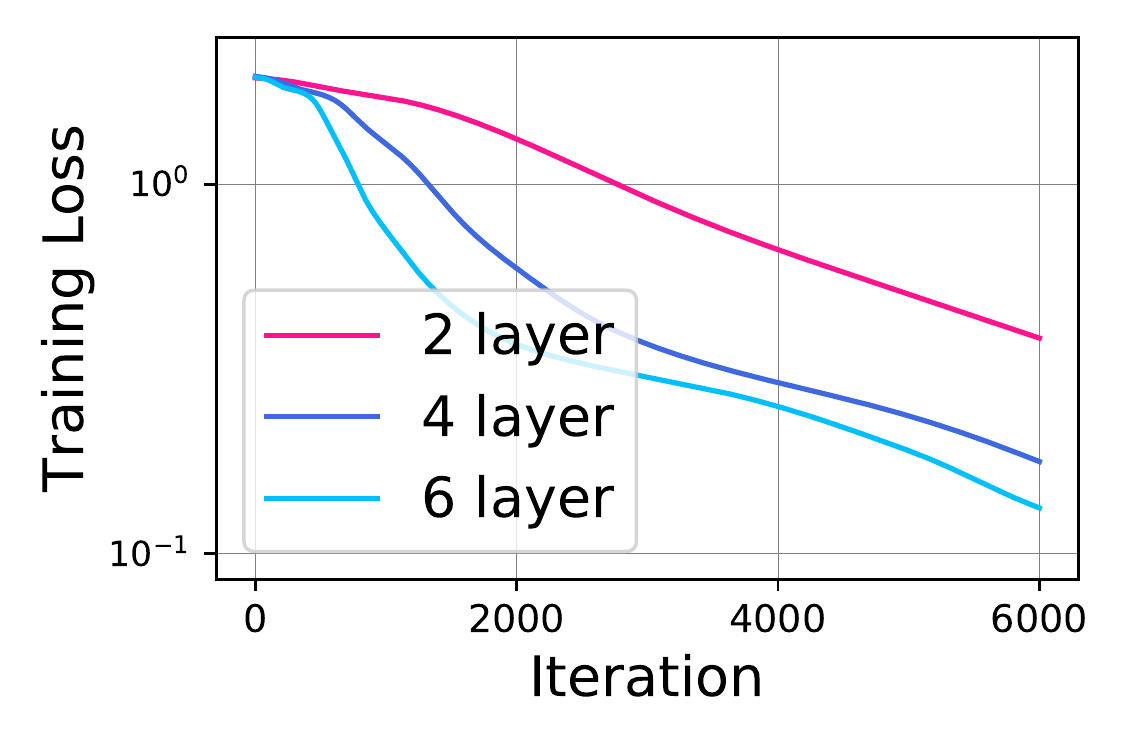} 
        \centering
        \caption{Linear and non-multiscale.}
    \end{subfigure}  
    \begin{subfigure}[b]{0.46\textwidth}
     \centering
        \includegraphics[width=0.9\textwidth]{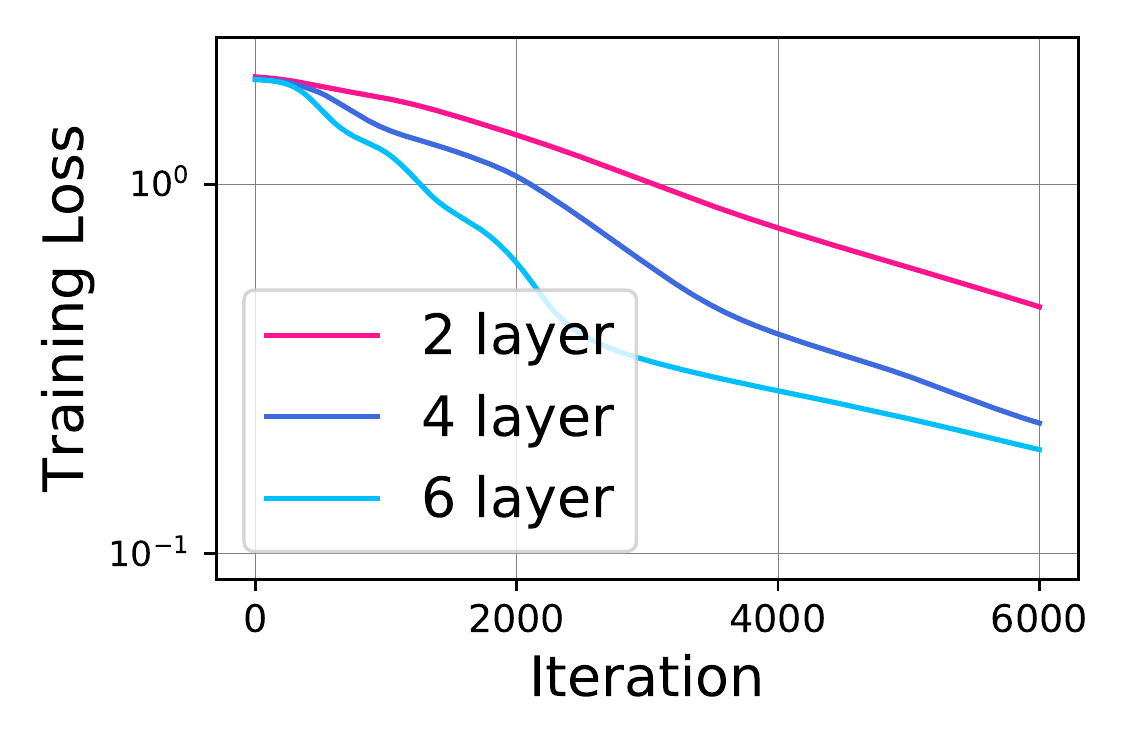} 
        \centering
        \caption{ReLU and non-multiscale.}
    \end{subfigure}   
     \begin{subfigure}[b]{0.46\textwidth}
     \centering
        \includegraphics[width=0.9\textwidth]{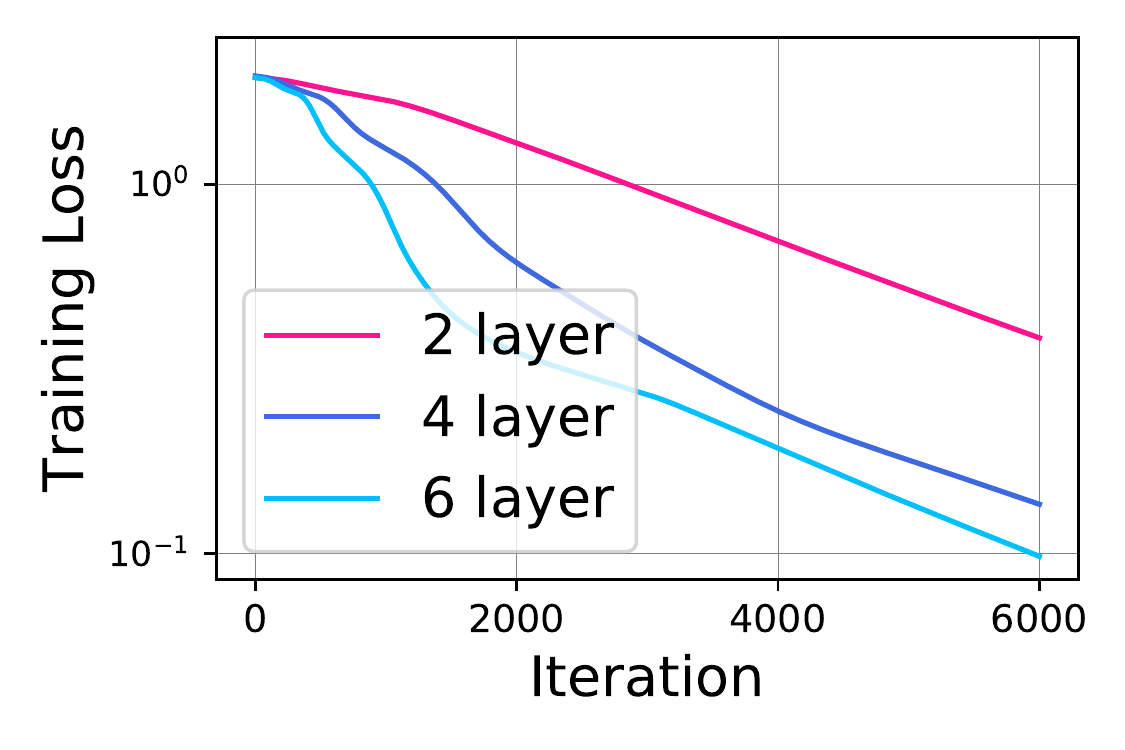} 
        \centering
        \caption{Linear and multiscale.}
    \end{subfigure}  
    \begin{subfigure}[b]{0.46\textwidth}
     \centering
        \includegraphics[width=0.9\textwidth]{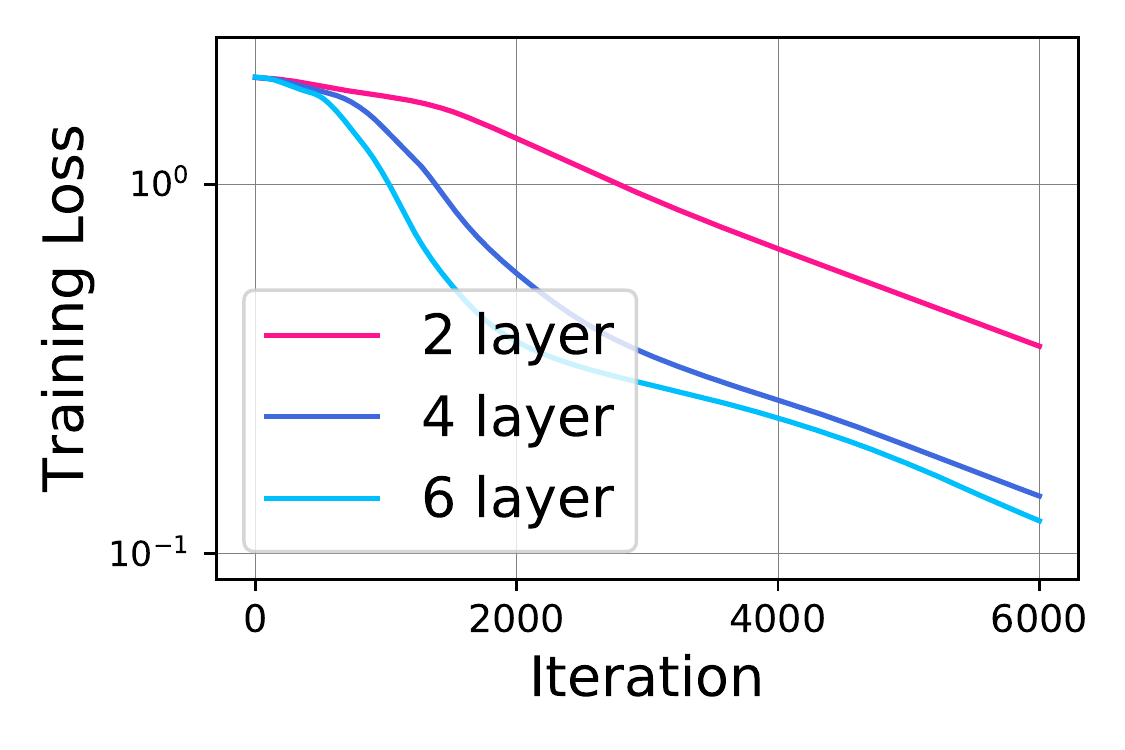} 
        \centering
        \caption{ReLU and multiscale.}
    \end{subfigure}   
        \vspace{0.2in}
    \caption{\textbf{Depth accelerates GNN training}. We plot the training curves of GNNs with ReLU and linear activation, multiscale and non-multiscale on the \textbf{Cora}  dataset. We use the GCN model with learning rate $5e-5$ and hidden dimension  $32$. }
    \label{fig:2a}
\end{figure*}
\clearpage

\begin{figure*}[p]
\centering
    \begin{subfigure}[b]{0.46\textwidth}
     \centering
        \includegraphics[width=0.9\textwidth]{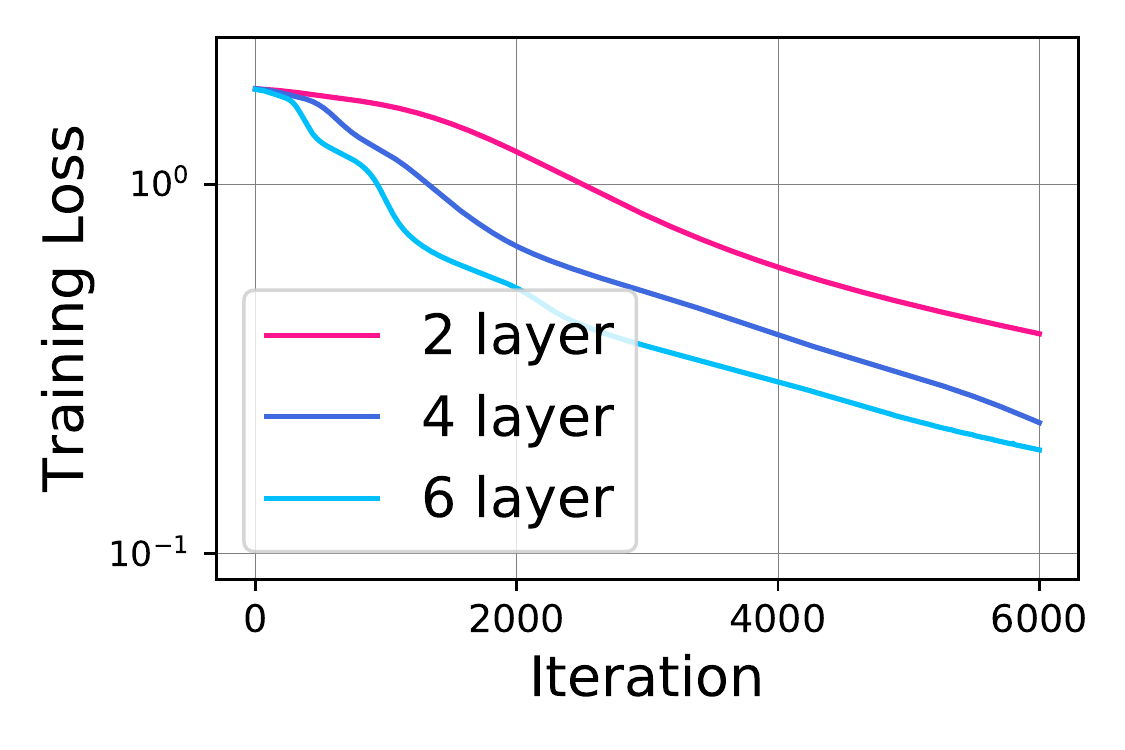} 
        \centering
        \caption{Linear and non-multiscale.}
    \end{subfigure}  
    \begin{subfigure}[b]{0.46\textwidth}
     \centering
        \includegraphics[width=0.9\textwidth]{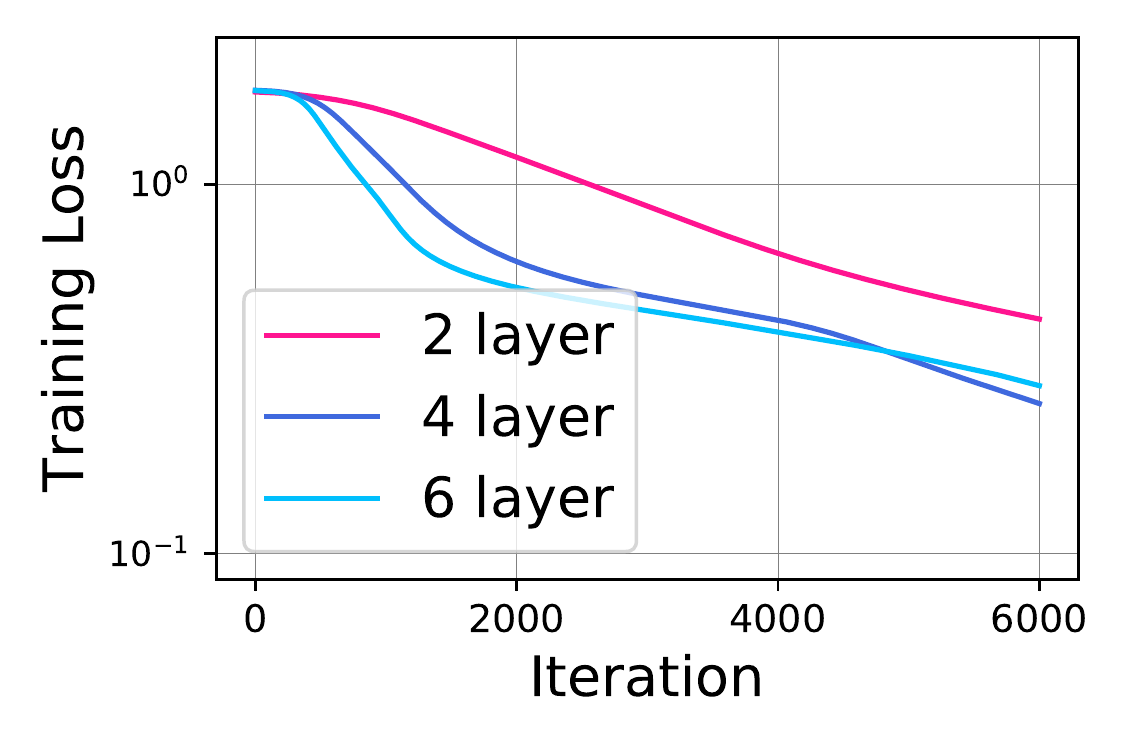} 
        \centering
        \caption{ReLU and non-multiscale.}
    \end{subfigure}   
     \begin{subfigure}[b]{0.46\textwidth}
     \centering
        \includegraphics[width=0.9\textwidth]{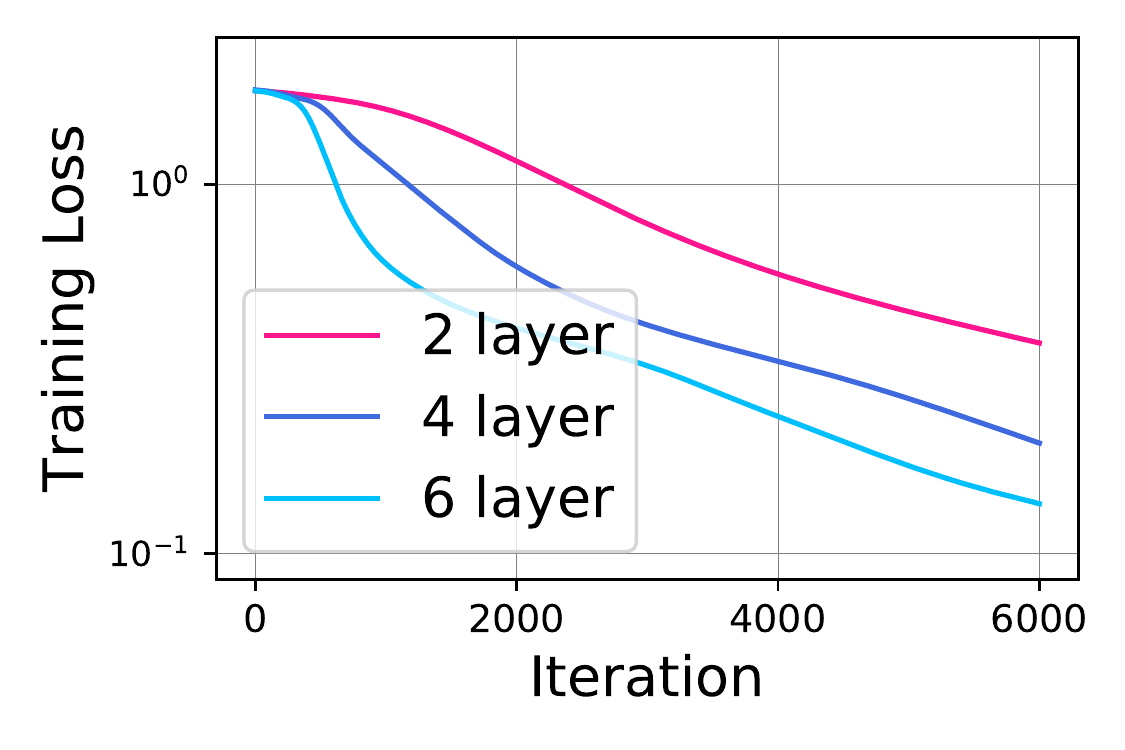} 
        \centering
        \caption{Linear and multiscale.}
    \end{subfigure}  
    \begin{subfigure}[b]{0.46\textwidth}
     \centering
        \includegraphics[width=0.9\textwidth]{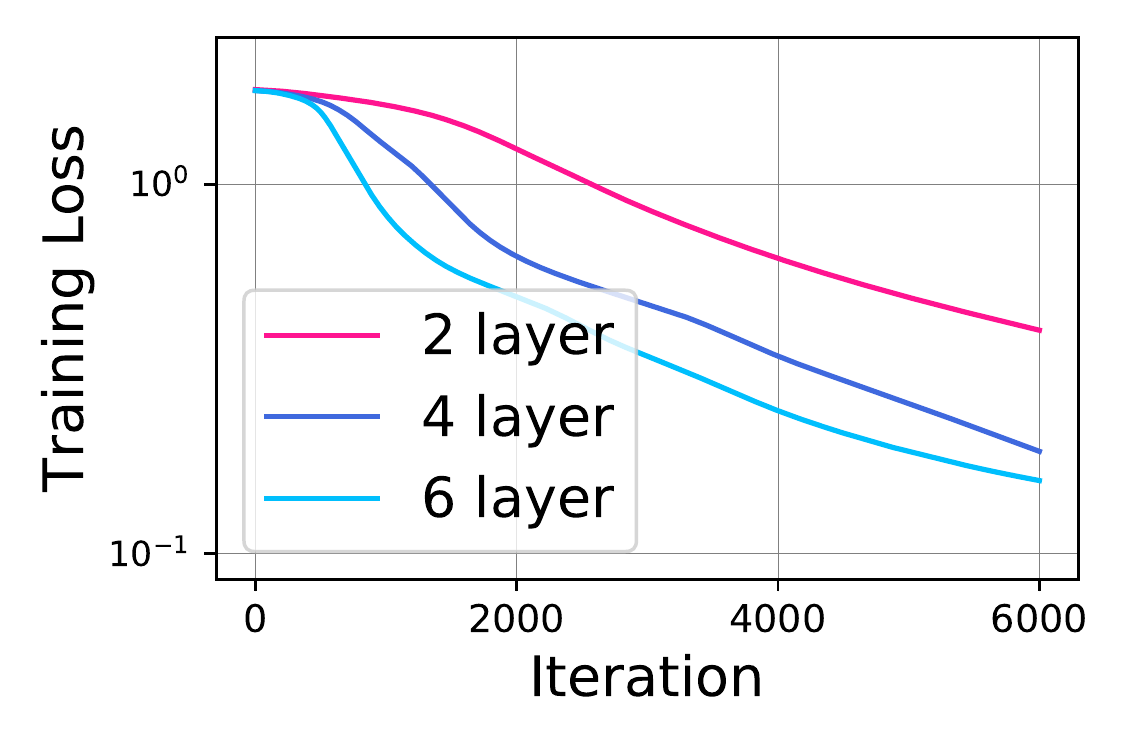} 
        \centering
        \caption{ReLU and multiscale.}
    \end{subfigure}   
    \vspace{0.2in}
    \caption{\textbf{Depth accelerates GNN training}. We plot the training curves of GNNs with ReLU and linear activation, multiscale and non-multiscale on the \textbf{Citeseer}  dataset. We use the GCN model with learning rate $5e-5$ and hidden dimension  $32$.}
    \label{fig:2b}
\end{figure*}
\clearpage

\begin{figure*}[p]
\centering
    \begin{subfigure}[b]{0.46\textwidth}
     \centering
        \includegraphics[width=0.9\textwidth]{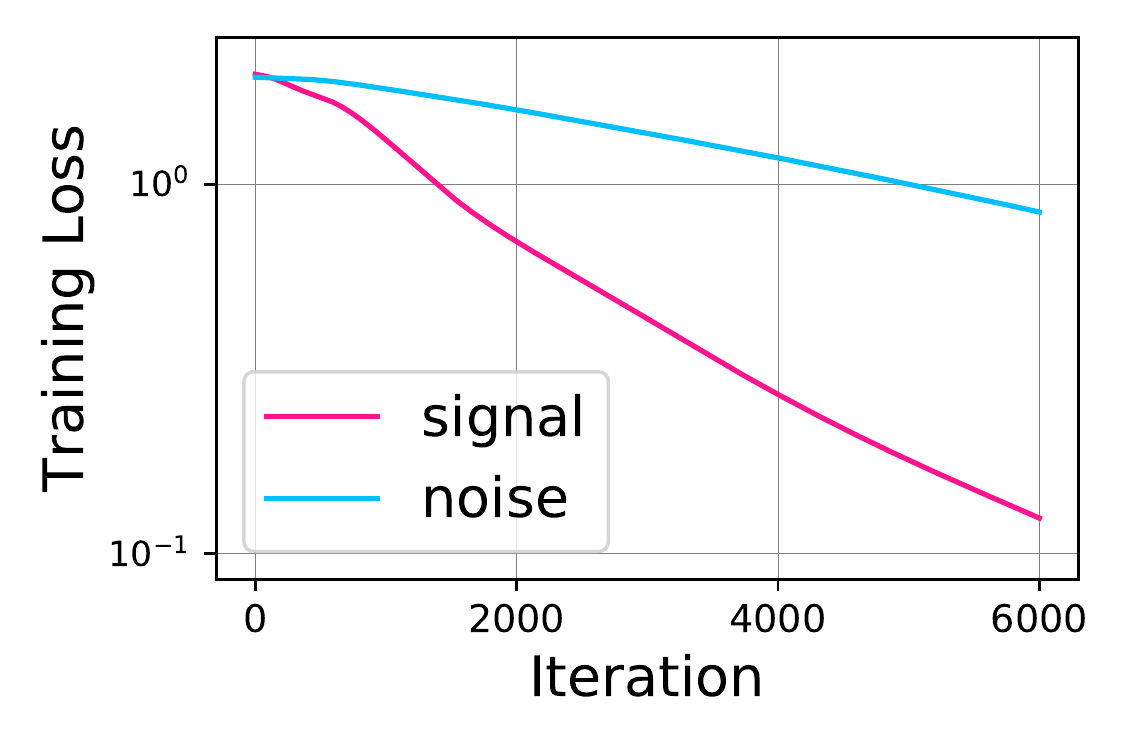} 
        \centering
        \caption{Linear and non-multiscale.}
    \end{subfigure}  
    \begin{subfigure}[b]{0.46\textwidth}
     \centering
        \includegraphics[width=0.9\textwidth]{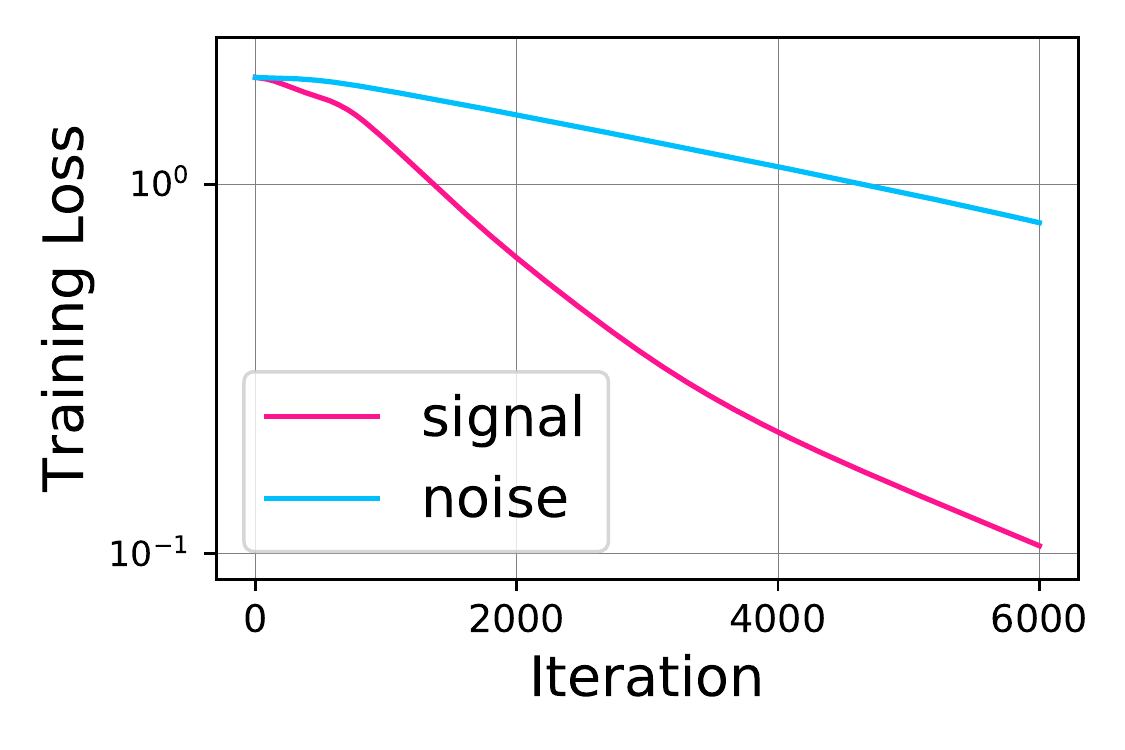} 
        \centering
        \caption{ReLU and non-multiscale.}
    \end{subfigure}   
     \begin{subfigure}[b]{0.46\textwidth}
     \centering
        \includegraphics[width=0.9\textwidth]{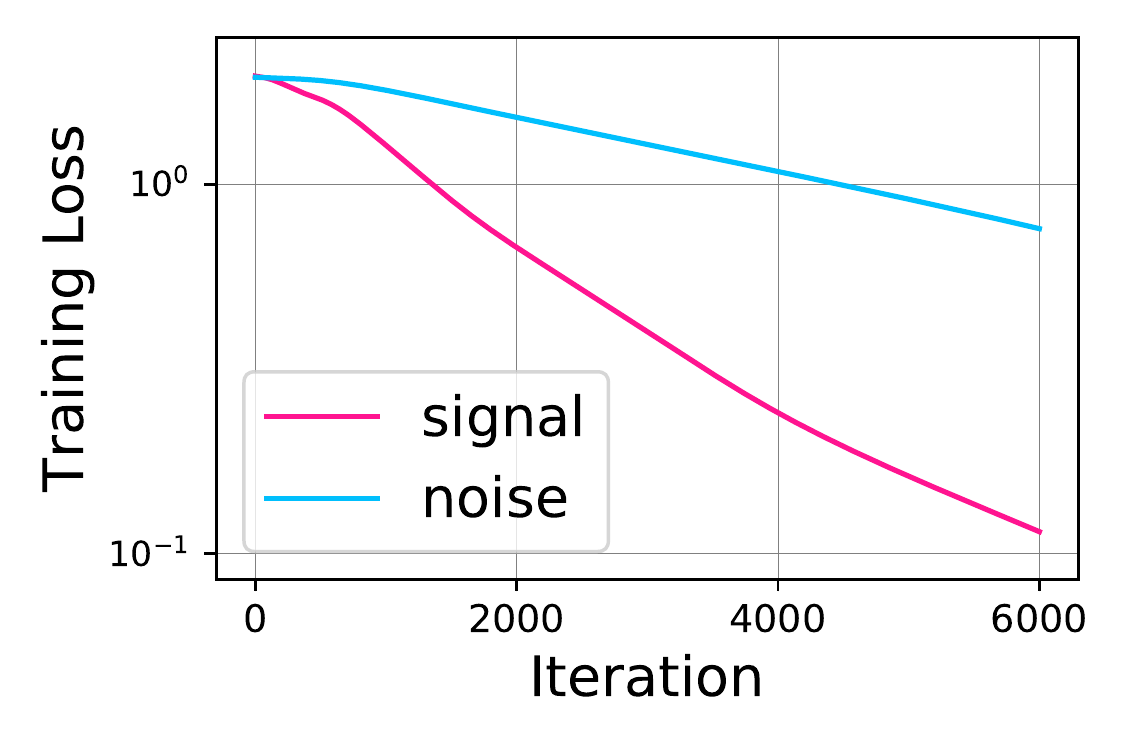} 
        \centering
        \caption{Linear and multiscale.}
    \end{subfigure}  
    \begin{subfigure}[b]{0.46\textwidth}
     \centering
        \includegraphics[width=0.9\textwidth]{figure/noise/Cora_gcn_jk_relu_2layer_lr1e-4.pdf} 
        \centering
        \caption{ReLU and multiscale.}
    \end{subfigure}   
        \vspace{0.2in}
    \caption{\textbf{GNNs train faster when the labels have signal instead of random noise}. We plot the training curves of multiscale and non-multiscale GNNs with ReLU and linear activation,  on the \textbf{Cora}  dataset. We use the two-layer GCN model with learning rate $1e-4$ and hidden dimension  $32$. }
    \label{fig:3a}
\end{figure*}
\clearpage

\begin{figure*}[p]
\centering
    \begin{subfigure}[b]{0.46\textwidth}
     \centering
        \includegraphics[width=0.9\textwidth]{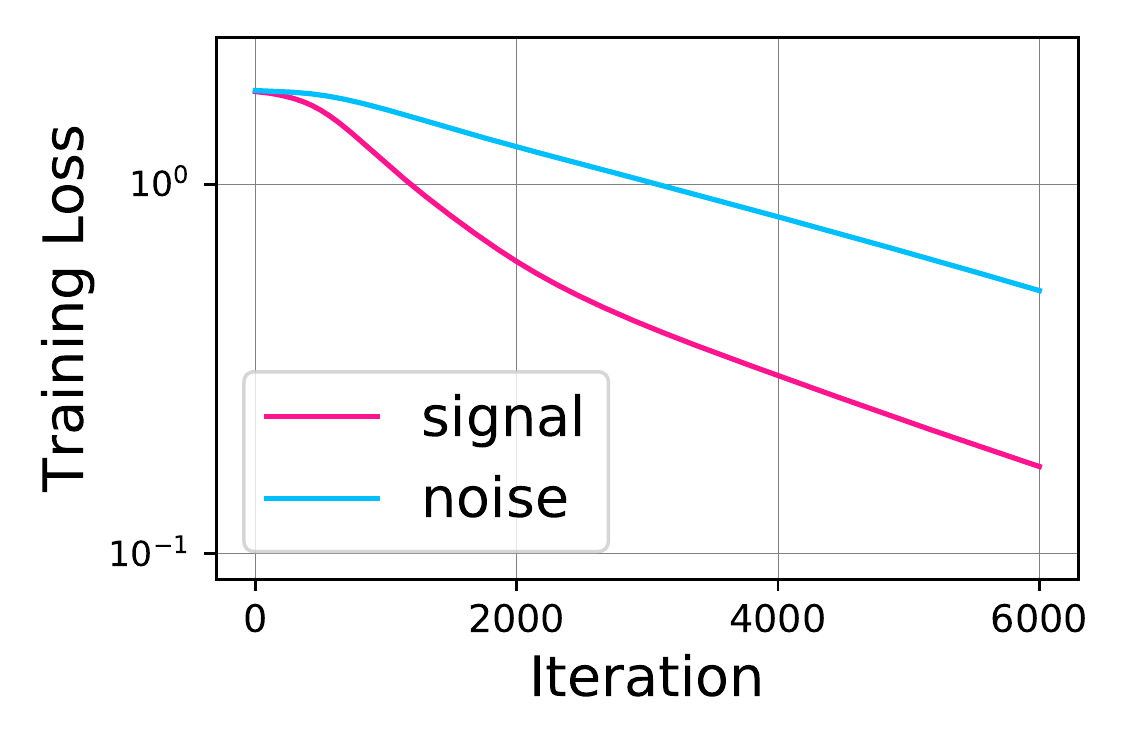} 
        \centering
        \caption{Linear and non-multiscale.}
    \end{subfigure}  
    \begin{subfigure}[b]{0.46\textwidth}
     \centering
        \includegraphics[width=0.9\textwidth]{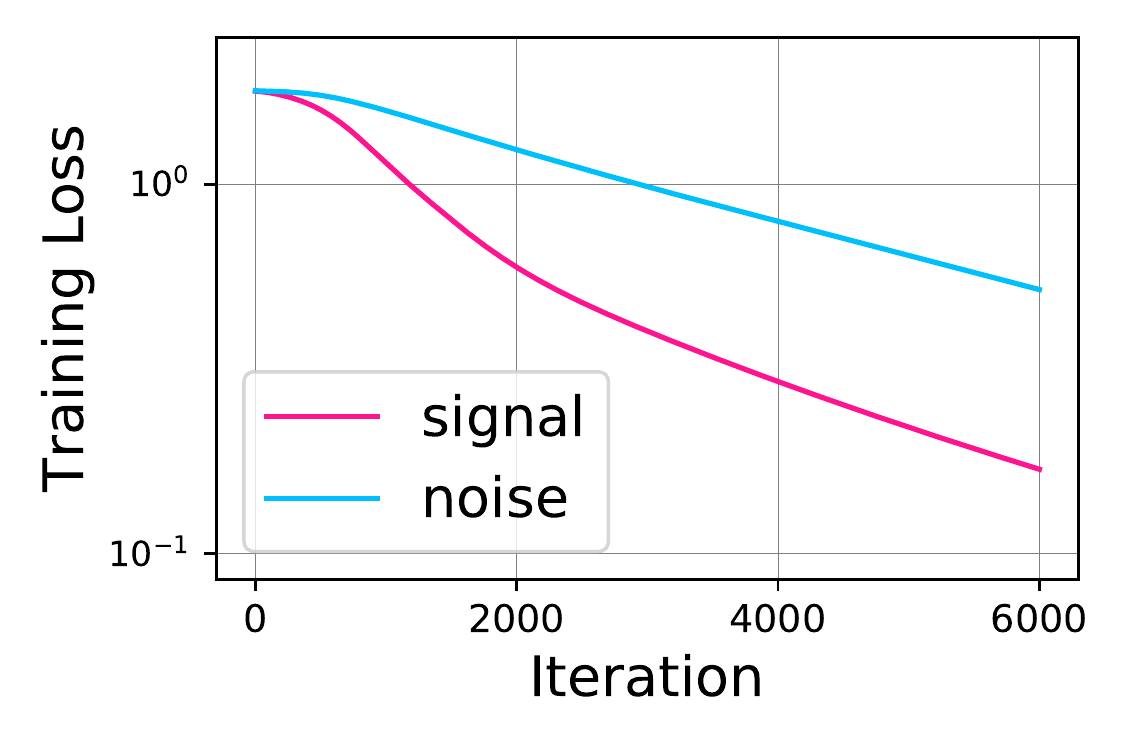} 
        \centering
        \caption{ReLU and non-multiscale.}
    \end{subfigure}   
     \begin{subfigure}[b]{0.46\textwidth}
     \centering
        \includegraphics[width=0.9\textwidth]{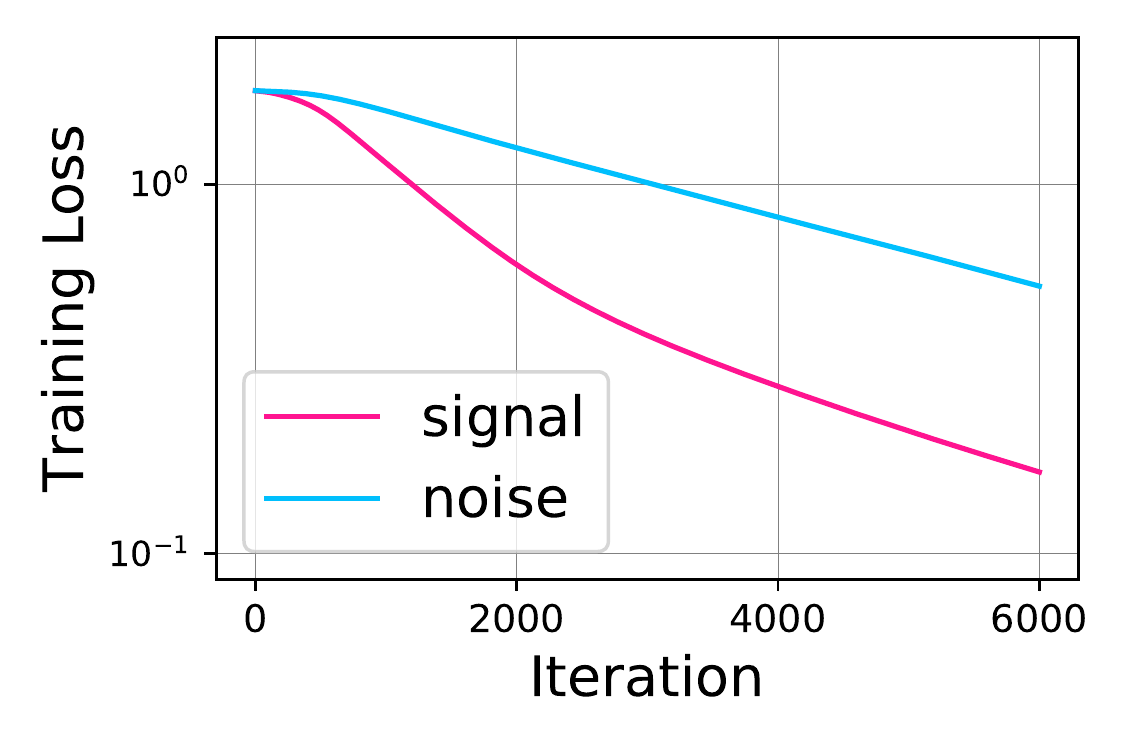} 
        \centering
        \caption{Linear and multiscale.}
    \end{subfigure}  
    \begin{subfigure}[b]{0.46\textwidth}
     \centering
        \includegraphics[width=0.9\textwidth]{figure/noise/CiteSeer_gcn_jk_relu_2layer_lr1e-4.pdf} 
        \centering
        \caption{ReLU and multiscale.}
    \end{subfigure}   
        \vspace{0.2in}
    \caption{\textbf{GNNs train faster when the labels have signal instead of random noise}. We plot the training curves of multiscale and non-multiscale GNNs with ReLU and linear activation,  on the \textbf{Citeseer}  dataset. We use the two-layer GCN model with learning rate $1e-4$ and hidden dimension  $32$. }
    \label{fig:3b}
\end{figure*}
\clearpage

\begin{figure*}[p]
\centering
    \begin{subfigure}[b]{0.55\textwidth}
     \centering
        \includegraphics[width=0.9\textwidth]{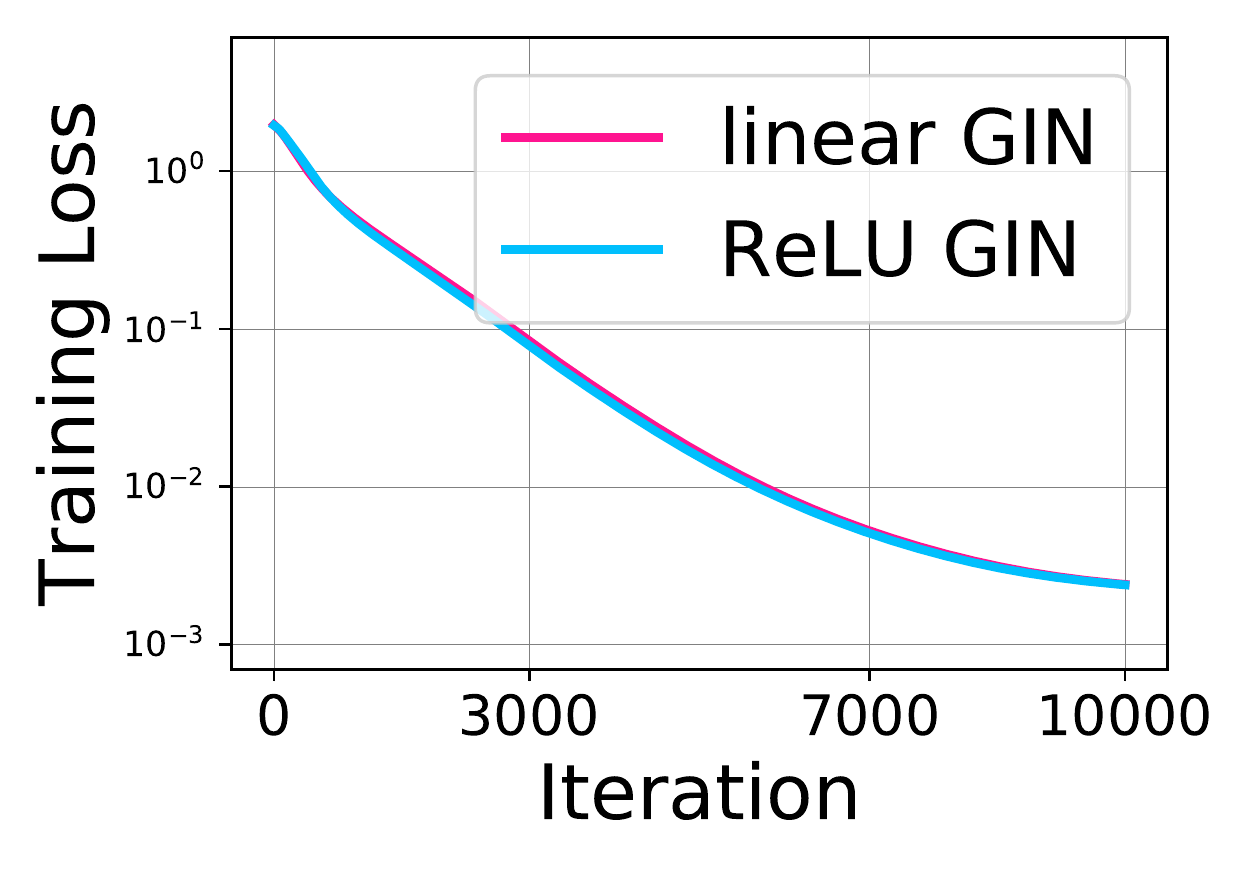} 
        \centering
        \caption{Linear GIN vs. ReLU GIN.}
    \end{subfigure}  
    \begin{subfigure}[b]{0.55\textwidth}
     \centering
        \includegraphics[width=0.9\textwidth]{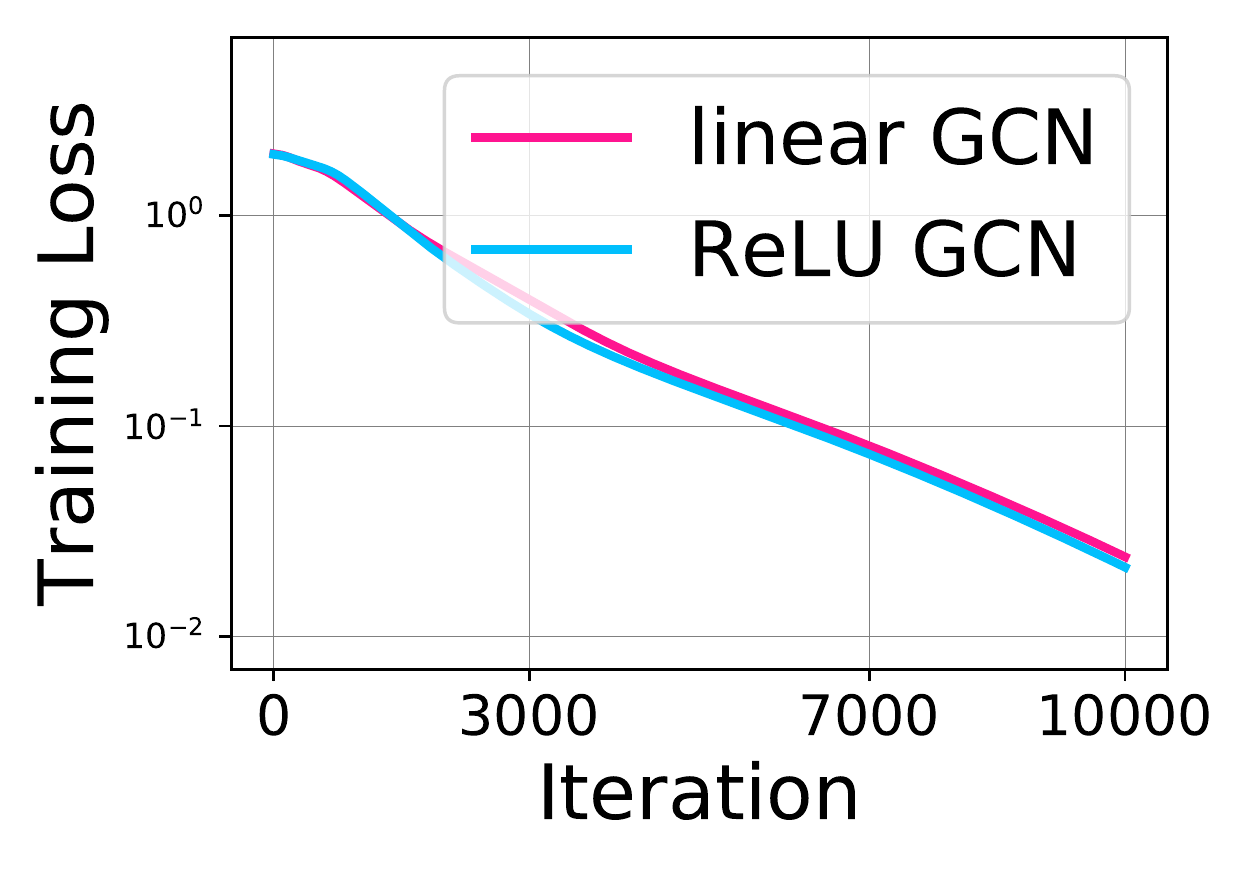} 
        \centering
        \caption{Linear  GCN vs. ReLU GCN.}
    \end{subfigure}   
    \caption{\textbf{Linear GNNs vs. ReLU GNNs}. We plot the training curves of GCN and GIN with ReLU and linear activation on the Cora dataset. The training curves of linear GNNs and ReLU  GNNs are similar, both converging to nearly zero training loss with the same linear rate. Moreover, GIN trains faster than GCN, which agrees with our bound in Theorem~\ref{thm:6}. We use the  learning rate $1e-4$, two layers, and hidden dimension  $32$.  }
    \label{fig:1a}
\end{figure*}
\clearpage

\section{Experimental Setup}
\label{sec:experiments}

In this section, we describe the experimental setup for reproducing our experiments. 

\paragraph{Dataset.} We perform all experiments on the Cora and Citeseer datasets~\cite{sen2008collective}. Cora and Citeer are citation networks and the goal is to classify academic documents into different subjects. The dataset contains bag-of-words features for each document
(node) and citation links (edges) between documents.  The tasks are semi-supervised node classification. Only a subset of nodes  have training labels. In our experiments, we use the default dataset split, i.e., which nodes have training labels, and minimize the training loss accordingly. Tabel~\ref{tab:stat_dataset} shows an overview of the dataset statistics.

\begin{table}[h!]
\begin{center}
\begin{tabular}{ccccc}
\multicolumn{1}{c}{Dataset} & \multicolumn{1}{c}{Nodes} & \multicolumn{1}{c}{Edges} & \multicolumn{1}{c}{Classes} & \multicolumn{1}{c}{Features} \\
\hline
\multicolumn{1}{c}{Citeseer} & \multicolumn{1}{c}{3,327} & \multicolumn{1}{c}{4,732} & \multicolumn{1}{c}{6} & \multicolumn{1}{c}{3,703} \\
\multicolumn{1}{c}{Cora} & \multicolumn{1}{c}{2,708} & \multicolumn{1}{c}{5,429} & \multicolumn{1}{c}{7} & \multicolumn{1}{c}{1,433} \\
\end{tabular}
\end{center}
\caption{Dataset statistics}
\label{tab:stat_dataset}
\end{table}

\paragraph{Training details.}  We describe  the training settings for our experiments. Let us first describe some common hyperparameters and settings, and then for each experiment or figure we describe the other hyperparameters. For our experiments, to more closely align with the common practice in GNN training, we use the Adam optimizer and keep  optimizer-specific hyperparameters except initial learning rate default. We set weight decay to zero. Next, we describe the settings for each experiment respectively.

For the experiment in Figure~\ref{fig:intro}, i.e., the training curves of linear vs. ReLU GNNs, we train the GCN and GIN with two layers on Cora with cross-entropy loss and learning rate 1e-4. We set  the hidden dimension to $32$.

For the experiment in Figure~\ref{fig:c11}, i.e., computing the graph condition for linear GNNs,  we use the linear GCN and GIN model with three layers on Cora and Citeseer. For linear GIN, we set $\epsilon$ to zero and MLP  layer to one.

For the experiment in Figure~\ref{fig:c12}, i.e., computing and plotting the time-dependent condition for linear GNNs, we train a  linear GCN with two layers on Cora with squared loss and learning rate 1e-4. We set the hidden dimension the input dimension for both Cora and for CiteSeer, because the global convergence theorem requires the hidden  dimension to be at least the same as input dimension. Note that  this requirement is standard in previous works as well, such as~\citet{arora2019convergence}. We use the default random initialization of PyTorch. The formula for computing the time-dependent $\lambda_T$ is given in the main paper.

For the experiment in Figure~\ref{fig:c13}, i.e., computing and plotting the time-dependent condition for linear GNNs across multiple training settings, we consider the following settings: 
\begin{enumerate}
    \item Dataset: Cora and Citeseer.
    \item Model: GCN and  GIN.
    \item Depth: Two and four layers.
    \item Activation: Linear and ReLU.
\end{enumerate}
We train the GNN with the settings above with squared loss and learning rate 1e-4. We set the hidden dimension to input dimension for Cora and CiteSeer. We use the default random initialization of PyTorch. The formula for  computing the time-dependent $\lambda_T$ is given in the main paper. For each point, we report the $\lambda_T$ at last epoch.

For the experiment in Figure~\ref{fig:c21}, i.e., computing the graph condition for multiscale linear GNNs,  we use the linear GCN and GIN model with three layers on Cora and Citeseer. For linear GIN, we set $\epsilon$ to zero and MLP  layer to one. 

For the experiment in Figure~\ref{fig:c22}, i.e., computing and plotting the time-dependent condition for multiscale linear GNNs, we train a  linear GCN with two layers on Cora with squared loss and learning rate 1e-4.
We set the hidden dimension to 2000 for Cora and 4000 for CiteSeer. We use the default random initialization of PyTorch.
The formula for  computing the time-dependent $\lambda_T$ is given in the main paper.

For the experiment in Figure~\ref{fig:c23}, i.e., computing and plotting the time-dependent condition for multiscale linear GNNs across multiple training settings, we consider the following settings: 
\begin{enumerate}
    \item Dataset: Cora and Citeseer.
    \item Model: Multiscale GCN and GIN.
    \item Depth: Two and four layers.
    \item Activation: Linear and ReLU.
\end{enumerate}
We train the multiscale GNN with the settings above with squared loss and learning rate 1e-4. We set the hidden dimension to 2000 for Cora and 4000 for CiteSeer. We use the default random initialization of PyTorch. The formula for  computing the time-dependent $\lambda_T$ is given in the main paper. For each point, we report the $\lambda_T$ at last epoch.

For the experiment in Figure~\ref{fig:c1}, i.e., multiscale vs. non-multiscale, we train the GCN with six layers and ReLU activation on Cora with cross-entropy loss and learning rate 5e-5. We set  the hidden dimension to $32$.

We perform  more extensive experiments to verify the conclusion for multiscale vs. non-multiscale in Figure~\ref{fig:1a}. There, we train the GCN with six layers with both ReLU and linear activation on both Cora and Citeseer with cross-entropy loss and learning rate 5e-5. We set  the hidden dimension to $32$.

For the experiment in Figure~\ref{fig:c2}, i.e., acceleration with  depth, we train the non-multiscale GCN with two, four, six layers and ReLU activation on Cora with cross-entropy loss and learning rate 5e-5. We set  the hidden dimension to $32$.

We perform  more extensive experiments to verify the conclusion for acceleration with depth in Figure~\ref{fig:2a} and Figure~\ref{fig:2b}. There, we train both multiscale and non-multiscale  GCN with $2$, $4$, $6$ layers with both ReLU and linear activation on both Cora and Citeseer with cross-entropy loss and learning rate 5e-5. We set  the hidden dimension to $32$.

For the experiment in Figure~\ref{fig:c3}, i.e., signal vs. noise, we train the non-multiscale GCN with two layers and ReLU activation on Cora with cross-entropy loss and learning rate 1e-4. We set  the hidden dimension to $32$. For signal, we use the default labels of Cora. For noise, we randomly choose a class as the label. 

We perform  more extensive experiments to verify the conclusion for signal vs. noise in Figure~\ref{fig:3a} and Figure~\ref{fig:3b}. There, we train both multiscale and non-multiscale  GCN with two layers with both ReLU and linear activation on both Cora and Citeseer with cross-entropy loss and learning rate 1e-4. We set  the hidden dimension to $32$.

For the experiment in Figure~\ref{fig:noise}, i.e., first term vs. second term, we use the same setting as in Figure~\ref{fig:c3}. We use the formula of our Theorem in the main paper.

\paragraph{Computing resources.} The computing hardware is based on the CPU and the
NVIDIA GeForce RTX 1080 Ti GPU. The software implementation is based on PyTorch and PyTorch Geometric~\citep{fey2019fast}. For all experiments, we train the GNNs with CPU and compute the eigenvalues with GPU.

\end{document}